\documentclass{article} %
\usepackage{scott}
\usepackage{iclr2025_conference,times}

\title{Towards General-Purpose Model-Free\\ Reinforcement Learning} %

\author{Scott Fujimoto, Pierluca D'Oro, Amy Zhang, Yuandong Tian, Michael Rabbat \\
Meta FAIR
}

\iclrfinalcopy %
\begin{document}

\maketitle

\begin{abstract}
Reinforcement learning (RL) promises a framework for near-universal problem-solving. In practice however, RL algorithms are often tailored to specific benchmarks, relying on carefully tuned hyperparameters and algorithmic choices. Recently, powerful model-based RL methods have shown impressive general results across benchmarks but come at the cost of increased complexity and slow run times, limiting their broader applicability. In this paper, we attempt to find a unifying model-free deep RL algorithm that can address a diverse class of domains and problem settings. To achieve this, we leverage model-based representations that approximately linearize the value function, taking advantage of the denser task objectives used by model-based RL while avoiding the costs associated with planning or simulated trajectories. We evaluate our algorithm, MR.Q, on a variety of common RL benchmarks with a single set of hyperparameters and show a competitive performance against domain-specific and general baselines, providing a concrete step towards building general-purpose model-free deep RL algorithms. 
{\def\thefootnote{}\footnotetext{Correspondence: \texttt{sfujimoto@meta.com}. Code: \url{https://github.com/facebookresearch/MRQ}.}}
\end{abstract}

\vspace{-4pt}
\input{figures/summary}
\vspace{-4pt}

\section{Introduction}

The conceptual premise of RL is inherently general-purpose---an RL agent can learn optimal behavior with only two basic elements: a well-defined objective and data describing its interactions with the environment. In reality, however, most RL algorithms are anything but general-purpose. Instead, RL algorithms are highly specialized and typically characterized by specific problem classes, such 
\begin{wraptable}[16]{r}{0.5\textwidth}
\vspace{-4pt}
\small
\centering
\caption{Hyperparameter differences between Rainbow~\citep{hessel2018rainbow} and TD3~\citep{fujimoto2018addressing}. TD3 uses an expected moving average~(EMA) update with an effective frequency of $\frac{1}{1 - 0.995} = 200$. } \label{table:rainbow_td3}
\vspace{-4pt}
\setlength{\tabcolsep}{4pt}
\begin{tabular}{lll}
\toprule
Hyperparameter & Rainbow & TD3 \\
\midrule
\rowcolor{sb_green!10}
Discount factor & $0.99$ & $0.99$ \\
\rowcolor{sb_green!10}
Optimizer & Adam & Adam \\ 
\rowcolor{sb_red!20}
Learning Rate & $6.25 \cdot 10^{-5}$ & $10^{-3}$ \\
\rowcolor{sb_red!20}
Adam $\e$ & $1.5 \cdot 10^{-4}$ & $10^{-8}$ \\
\rowcolor{sb_green!10}
Replay buffer size & $1$M & $1$M \\
\rowcolor{sb_red!20}
Minibatch size & $32$ & $100$ \\
\rowcolor{sb_red!20}
Target network update & Iterative & EMA \\
\rowcolor{sb_red!20}
Effective target update freq.\ & $8$k & $200$ \\
\rowcolor{sb_red!20}
Initial random steps & $20$k & $1$k \\
\bottomrule
\end{tabular}
\end{wraptable}
as discrete versus continuous actions or vector versus pixel observations, with each category requiring its own set of algorithmic choices and hyperparameters. For example, Rainbow and TD3~\citep{hessel2018rainbow, fujimoto2018addressing}, common methods for Atari and MuJoCo respectively~\citep{bellemare2013arcade, mujoco}, have more differences than similarities in their shared hyperparameters~(\ref{table:rainbow_td3})---without accounting for further algorithmic differences.

To some extent, general-purpose algorithms do exist---policy gradient methods~\citep{williams1992reinforce, trpo, ppo} and many evolutionary approaches~\citep{rechenberg1978evolutionsstrategien, back1996evolutionary, rubinstein1997optimization, salimans2017evolution} require few assumptions on the underlying problem. Unfortunately, these methods often offer poor sample efficiency and asymptotic performance compared more domain-specific approaches, and in some instances, can require extensive re-tuning over numerous implementation-level details~\citep{Engstrom2020Implementation, shengyi2022the37implementation}. 

Recently, DreamerV3~\citep{hafner2023mastering} and TD-MPC2~\citep{hansen2024td}, have showcased the potential of general-purpose model-based approaches, achieving impressive single-task performance on a diverse set of benchmarks without re-tuning hyperparameters. However, despite their success, model-based methods also introduce substantial algorithmic and computational complexity, making them less practical than lightweight domain-specific model-free algorithms.

This paper presents a general model-free RL algorithm that leverages model-based representations to achieve the sample efficiency and performance of model-based methods, without the computational overhead. A recent surge of high-performing model-free RL algorithms with dynamics-based representations~\citep{guo2020bootstrap, guo2022byol, schwarzer2020data, schwarzer2023bigger, zhao2023simplified, fujimoto2024sale, zheng2024texttt, scannell2024iqrl} has showcased the potential of this family of algorithms when tailored for a single benchmark. Recognizing the similarity between these model-based and model-free approaches, our hypothesis is that the true benefit of model-based objectives is in the implicitly learned representation, rather than the model itself, and thus prompting the question: 
\begin{center}
    \textit{
    Can model-based representations alone enable sample-efficient general-purpose learning?
    }
\end{center}
Our proposed approach is based on learning features that approximately capture a linear relationship between state-action pairs and value. To do so, we draw heavily from modern dynamics-based representation learning methods (see \hyperref[sec:related_work]{Related Work}) as well as the work of \cite{parr2008analysis}, who show that both model-based and model-free objectives converge to the same solution in linear space. By mapping states and actions into a single, unified embedding, we eliminate any environment-specific characteristics of the input space and allow for a standardized set of hyperparameters.

We evaluate our method, MR.Q, on four widely used RL benchmarks and 118 environments, and achieve competitive performance against state-of-the-art domain-specific and general baselines without algorithmic or hyperparameter changes between environments or benchmarks. %

\vspace{-4pt}
\section{Related Work} \label{sec:related_work}
\vspace{-4pt}

\textbf{General-purpose RL.} 
Although many traditional RL methods are general-purpose in principle, practical constraints often force assumptions about the task domain. For example, algorithms like Q-learning and SARSA~\citep{watkins1989qlearning, rummery1994line} can be conceptually extended to continuous spaces, but are typically implemented using discrete lookup tables. In practice, early examples of general decision-making approaches can be found in on-policy methods with function approximation. For instance, both evolutionary algorithms~\citep{rechenberg1978evolutionsstrategien, back1996evolutionary, rubinstein1997optimization, salimans2017evolution} and policy gradient methods~\citep{williams1992reinforce, sutton1999policy, trpo, ppo} offer update rules with convergence guarantees and independence to the input space. However, despite their generality, these methods are also hindered by poor sample efficiency and are prone to local minima, limiting their suitability for many practical applications. 

In contrast, the design of deep RL algorithms tends to favor more specialized approaches that align closely with a single benchmark---e.g., DQN$\leftrightarrow$Atari~\citep{bellemare2013arcade,DQN}, DDPG$\leftrightarrow$MuJoCo~\citep{mujoco,DDPG}, or AlphaGo$\leftrightarrow$Go~\citep{alphago}. Generalizing beyond these initial benchmarks can often require significant engineering, tuning, or algorithmic discovery~\citep{luong2019applications, schrittwieser2020mastering, haydari2020deep,ibarz2021train}. In imitation learning, GATO achieved generalist behavior, but relied on large expert datasets~\citep{reed2022generalist}. 
Recently, DreamerV3~\citep{hafner2023mastering} demonstrated a strong capability over many benchmarks without re-tuning, but used costly large models and simulated rollouts. Our objective is to discover a lightweight model-free approach to general-purpose learning. 

\textbf{Dynamics-based representation learning.} Building representations from system dynamics is a long-standing approach for adaptation, partial observability, and feature selection~\citep{dayan1993improving, littman2001predictive, parr2008analysis}. Numerous model-free methods have been developed to learn representations by predicting future latent states~\citep{munk2016learning, van2016stable, zhang2018decoupling, gelada2019deepmdp, lee2020stochastic,  guo2020bootstrap, guo2022byol, schwarzer2020data, schwarzer2023bigger, zintgraf2021varibad, yu2021playvirtual, yu2022mask, fujimoto2021srdice, fujimoto2024sale, mcinroe2021learning, seo2022reinforcement, kim2022self, tang2023understanding, zhao2023simplified, zheng2024texttt, ni2024bridging, scannell2024iqrl}. Unsurprisingly, these model-free approaches closely relate to model-based counterparts which learn a latent dynamics model for planning or value estimation~\citep{watter2015embed, finn2016deep, karl2017deep, ha2018world, schrittwieser2020mastering, schrittwieser2021online, ye2021mastering, hansen2022temporal, hansen2024td, hafner2019learning, hafner2023mastering, wang2024efficientzero}. Our approach, MR.Q, is most closely related to the state-action representation learning in TD7~\citep{fujimoto2024sale}. At a high level, MR.Q differs from TD7 by discarding the original input and including losses over the reward and termination. MR.Q also differs significantly in implementation, drawing inspiration from prior work to determine a set of design choices that performs well across benchmarks, including multi-step returns, unrolled dynamics, and categorical losses. 

Our motivation also relates to linear MDPs~\citep{jin2020provably, agarwal2020flambe} and linear spectral representation~\citep{ren2022free, ren2023latent, zhang2022making, shribak2024diffusion}. The latter aims to learn a low-rank decomposition of the transition dynamics of the MDP and recover a linear relationship between an embedding and the value function. Similarly, our work connects to two-stage linear RL, where a non-linear embedding is learned for linear~RL~\citep{levine2017shallow,chung2018twotimescale}. 

\textbf{State abstraction.} Our work is closely related to bisimulation metrics~\citep{ferns2004metrics, ferns2011bisimulation, castro2020scalable} and MDP homomorphisms~\citep{ravindran2004algebraic, van2020plannable, van2020mdp, rezaei2022continuous} which rely on measures of similarity in reward and dynamics for state or action abstraction. These concepts have inspired practical approximations to bisimulation metrics as a means of shaping representations in deep RL agents, particularly those using image-based observations~\citep{zhang2020learning, castro2021mico, zang2022simsr}.

\vspace{-4pt}
\section{Background}
\vspace{-4pt}

Reinforcement learning (RL) problems are described by a Markov Decision Process~(MDP)~\citep{bellman1957markovian}, which we define by a tuple~$(S, A, p, R, \y)$ of state space~$S$, action space~$A$, dynamics function~$p$, reward function~$R$ and discount factor~$\y$. Value-based RL methods learn a value function~$Q^\pi(s,a) := \E_\pi[ \sum_{t=0}^\infty \y^t r_t | s_0 = s, a_0 = a]$ that models the expected discounted sum of rewards~$r_t \sim R(s_t,a_t)$ by following a policy~$\pi$ which maps states~$s$ to actions~$a$. 

The true value function~$Q^\pi$ is estimated by an approximate value function~$Q_\ta$. We use subscripts to indicate the network parameters~$\ta$. Target networks, which are used to introduce stationarity in prediction targets, have parameters denoted by an apostrophe, e.g., $Q_{\ta'}$. These parameters are periodically synchronized with the current network parameters ($\ta' \leftarrow \ta$).

\vspace{-4pt}
\section{Model-based Representations for Q-learning}
\vspace{-4pt}

This section presents the MR.Q algorithm (Model-based Representations for Q-learning), a model-free RL algorithm that learns an approximately linear representation of the value function through model-based objectives. Value-based RL algorithms learn a value function~$Q$ that maps state-action pairs~$(s,a)$ to values in $\mathbb{R}$ and a policy~$\pi$ that maps states~$s$ to actions~$a$. Like many representation learning methods for RL, MR.Q adds an initial step that transforms states and state-action pairs into embeddings~$\mathbf{z}_{s}$ and $\mathbf{z}_{sa}$, which serves as inputs to the downstream policy and value function. 
\begin{alignat}{2} \label{eqn:intermediate_representation}
    &f_\omega: s \rightarrow \mathbf{z}_s, 
    \qquad \qquad \qquad \qquad
    &&g_\omega: (s,a) \rightarrow \mathbf{z}_{sa}, \\
    &\pi_\phi: \mathbf{z}_s \rightarrow a,
    \qquad \qquad \qquad \qquad
    &&Q_\ta: \mathbf{z}_{sa} \rightarrow \mathbb{R}. 
\end{alignat}
While neither the value function nor policy require explicit representation learning, using intermediate embeddings has two main benefits: 
\begin{enumerate}[nosep]
    \item Introducing an explicit representation learning stage can enable richer alternative learning signals that are grounded in the dynamics and rewards of the MDP, as opposed to relying exclusively on non-stationary value targets used in both value and policy learning.
    \item Representation learning can transform the input into a unified, abstract space that is decoupled from the original input characteristics, e.g., images or action spaces. This abstraction allows us to filter irrelevant or spurious details and use unified downstream architectures, improving robustness to environment variations.  
\end{enumerate}
To learn these embeddings, we draw inspiration from linear feature selection, revisiting the work of \cite{parr2008analysis}, as well as MDP homomorphisms~\citep{ravindran2002model}. In \cref{sec:theory} we highlight how model-based objectives can be used to learn features that share an approximately linear relationship with the true value function. Then in \cref{sec:algorithm}, we relax our theoretical motivation for a practical algorithm based on recent advances in dynamics-based representation learning.

\subsection{Theoretical Motivation} \label{sec:theory}

Consider a linear decomposition of the value function, where the value function $Q(s,a)$ is represented by features~$\mathbf{z}_{sa}$ and linear weights~$\mathbf{w}$: 
\begin{equation} \label{eqn:linear_Q}
    Q(s,a) = \mathbf{z}_{sa}^\top \mathbf{w}.
\end{equation}
Our primary objective is to learn features~$\mathbf{z}_{sa}$ that share an approximately linear relationship with the true value function~$Q^\pi$. However, since this relationship is only approximate, we use these features as input to a non-linear function~$\hat Q(\mathbf{z}_{sa})$, rather than relying solely on linear function approximation.

We start by exploring how to find features that can linearly represent the true value function. Given a dataset~$D$ of tuples~$(s,a,r,s',a')$, we consider two possible approaches for learning a value function~$Q$: 
A model-free update based on semi-gradient TD~\citep{sutton1988tdlearning, suttonbarto}: 
\begin{equation} \label{eqn:semi_gradient_TD}
    \mathbf{w} \leftarrow \mathbf{w} - \al \E_D 
    \lb \nabla_{\mathbf{w}} \lp \mathbf{z}_{sa}^\top \mathbf{w} - |r + \y \mathbf{z}_{s'a'}^\top \mathbf{w}|_\text{sg} \rp^2 \rb.
\end{equation}
A model-based approach to learn $\mathbf{w}_\text{mb}$, based on rolling out estimates of the dynamics and reward: 
\begin{align} 
    &\mathbf{w}_\text{mb} := \sum_{t=0}^\infty \y^t W_p^t \mathbf{w}_r, %
    \label{eqn:linear_model_based} \\
    &\mathbf{w}_r := \argmin_{\mathbf{w}} \E_D \lb \lp \mathbf{z}_{sa}^\top \mathbf{w} - r \rp^2 \rb,
    \qquad
    W_p := \argmin_{W} \E_D \lb \lp \mathbf{z}_{sa}^\top W - \mathbf{z}_{s'a'} \rp^2 \rb.
    \label{eqn:wp_wr}
\end{align}
Closely following \cite{parr2008analysis} and \cite{song2016linear}, we can show that these approaches converge to the same solution (proofs for this section can be found in \cref{appendix:proofs}). 
\begin{theorem} \label{thm:equal}
    The fixed point of the model-free approach~(\ref{eqn:semi_gradient_TD}) and the solution of the model-based approach~(\ref{eqn:linear_model_based}) are the same.
\end{theorem}
From the insight of \ref{thm:equal}, we can connect the value error~VE, the difference between an approximate value function~$Q$ and the true value function~$Q^\pi$,
\begin{equation}
    \text{VE}(s,a) := Q(s,a) - Q^\pi(s,a) 
\end{equation}
to the accuracy of reward and dynamics components of the estimated model (\ref{thm:bound}).
\begin{theorem} \label{thm:bound} The value error of the solution described by \ref{thm:equal} is bounded by the accuracy of the estimated dynamics and reward:
\begin{equation}
    |\textnormal{VE}(s,a)| \leq \frac{1}{1 - \y} \max_{(s,a) \in S \times A} \lp | \mathbf{z}_{sa}^\top \mathbf{w}_r - \E_{r|s,a}[r] | + \max_i | \mathbf{w}_i | \sum | \mathbf{z}_{sa}^\top W_p - \E_{s',a'|s,a} [ \mathbf{z}_{s'a'} ] | \rp.
\end{equation}  
\end{theorem}
\cite{parr2008analysis} and \cite{song2016linear} use a related insight regarding the Bellman error to infer an approach for feature selection. However, with the advent of deep learning, we can instead directly learn the features~$\mathbf{z}_{sa}$ by jointly optimizing them alongside the linear weights~$\mathbf{w}_r$ and $W_p$. This is accomplished by treating the features and linear weights as a unified end-to-end model and balancing the losses in \ref{eqn:wp_wr} with a hyperparameter~$\lambda$:
\begin{equation} \label{eqn:linear_features_loss}
    \Loss(\mathbf{z}_{sa}, \mathbf{w}_r, W_p) = \underbracket[\fontdimen8\textfont3]{\E_D \lb \lp \mathbf{z}_{sa}^\top \mathbf{w}_r - r \rp^2 \rb}_\text{Reward learning} + 
    \underbracket[\fontdimen8\textfont3]{\lambda \E_D \lb \lp \mathbf{z}_{sa}^\top W_p - \mathbf{z}_{s'a'} \rp^2 \rb}_\text{Dynamics learning}. 
\end{equation}
However, the resulting \ref{eqn:linear_features_loss} has some notable drawbacks. 

\textbf{Dependency on $\pi$.} The dynamics target $\mathbf{z}_{s'a'}$ depends on an action~$a'$ determined by the policy~$\pi$. In policy optimization problems, this introduces non-stationarity, where the target embedding must be continually updated to reflect changes in the policy. This creates an undesirable interdependence between the policy and encoder. 

\textbf{Undesirable local minima.} Jointly optimizing both the features~$\mathbf{z}_{sa}$ and the dynamics target can lead to undesirable local minima, similar to the issues encountered with Bellman residual minimization~\citep{baird1995residual, fujimoto2022should}. This can result in collapsed or trivial solutions when the dataset does not fully cover the state and action space or when the reward is sparse. 

To address these issues, we suggest relaxations on our proposed, theoretically grounded approach: 
\begin{equation} \label{eqn:new_feature_loss}
    \Loss(\mathbf{z}_{sa}, \mathbf{w}_r, W_p) = \E_D \lb \lp \mathbf{z}_{sa}^\top \mathbf{w}_r - r \rp^2 \rb + \lambda \E_D \Big[ \Big( \mathbf{z}_{sa}^\top W_p - \underbracket[\fontdimen8\textfont3]{\bar{\mathbf{z}}_{s'}}_{\makebox[0pt]{\scriptsize Adjustment}} \Big)^2 \Big]. 
\end{equation}
We propose two key modifications to alleviate the aforementioned issues. Firstly, we use a state-dependent embedding~$\mathbf{z}_{s'}$ as the dynamics target, rather than the state-action embedding~$\mathbf{z}_{s'a'}$. This eliminates any dependency on the current policy while still capturing the environment's dynamics. 

Secondly, to mitigate the issue of local minima, we use a target network~$f_{\omega'}(s')$ to generate the dynamics target~$\bar{\mathbf{z}}_{s'}$, where the parameters~$\omega'$ are periodically updated to track the current network parameters~$\omega$. Empirical evidence from prior work suggests that this approach can yield significant performance gains~(\cite{grill2020bootstrap, assran2023self}, see \hyperref[sec:related_work]{Related Work}), although it no longer guarantees convergence to a fixed point.

Due to these two changes, even if the modified objective defined by \ref{eqn:new_feature_loss} is minimized, we can no longer assume there is a \textit{linear} relationship between the embedding~$\mathbf{z}_{sa}$ and the value function. However, we can instead allow for a \textit{non-linear} relationship, replacing linear weights $\mathbf{w}$ with a non-linear function $\hat Q(\mathbf{z}_{sa})$. We can show that this relationship exists as long as the features are sufficiently rich~(i.e., such that a MDP homomorphism is satisfied~\citep{ravindran2002model}).
\begin{theorem} \label{thm:homomorphism}
Given functions $f(s)=\mathbf{z}_s$ and $g(\mathbf{z}_s, a)=\mathbf{z}_{sa}$, then if there exists functions $\hat p$ and $\hat R$ such that for all $(s,a) \in S \times A$: 
\begin{align} \label{eqn:optimal_pi}
    &\E_{\hat R} [ \hat R(\mathbf{z}_{sa}) ] = \E_R \lb R(s,a) \rb, 
    &\hat p(\mathbf{z}_{s'}| \mathbf{z}_{sa}) = \sum_{\hat s : \mathbf{z}_{\hat s}=\mathbf{z}_{s'}} p(\hat s|s, a),
\end{align}
then for any policy~$\pi$ where there exists a corresponding policy~$\hat \pi(a | \mathbf{z}_s) = \pi(a | s)$, there exists a function~$\hat Q$ equal to the true value function $Q^\pi$ over all possible state-action pairs~$(s,a) \in S \times A$:
    \begin{equation}
    \hat Q(\mathbf{z}_{sa}) = Q^\pi(s,a).    
    \end{equation} 
Furthermore, \ref{eqn:optimal_pi} guarantees the existence of an optimal policy~$\hat \pi^*(a | \mathbf{z}_s) = \pi^*(a|s)$. 
\end{theorem}
Consequently, even if the features~$\mathbf{z}_{sa}$ do not linearly represent the true value function, i.e., the loss in \ref{eqn:linear_features_loss} cannot be not exactly minimized, $\mathbf{z}_{sa}$ can still be used in a non-linear relationship to represent the value function. Furthermore, \ref{thm:homomorphism} outlines a similar objective as the original linear objective defined in \ref{eqn:linear_features_loss}, in learning the reward and dynamics of the MDP. 

These results motivates the practical algorithm discussed in the following section. Using the adjusted loss defined in \ref{eqn:new_feature_loss}, we will aim to learn features with an approximately linear relationship to the true value function, but use a non-linear value function with those features to account for the error induced by our approximations.

\subsection{Algorithm} \label{sec:algorithm}

We now present the details of MR.Q (Model-based Representations for Q-learning). Building on the insights from the previous section, our key idea is to learn a state-action embedding~$\mathbf{z}_{sa}$ that is approximately linear with the true value function~$Q^\pi$. To account for approximation errors, these features are used with \textit{non-linear} function approximation to determine the value. 

The state embedding vector~$\mathbf{z}_s$ is obtained as an intermediate component by training end-to-end with the state-action encoder. MR.Q handles different input modalities by swapping the architecture of the state encoder. Since $\mathbf{z}_s$ is a vector, the remaining networks are independent of the observation space and use feedforward networks. 

Given the transition $(s,a,r,d,s')$ from the replay buffer:  %

\begin{minipage}{0.49\textwidth}
\centering
\addtolength{\parskip}{-8pt}
\setlength\extrarowheight{2pt}
\begin{tabular}{ll}
\toprule
\textbf{Output} MR.Q \\ 
\midrule
\rowcolor{sb_gray!20}
\multicolumn{2}{c}{\textcolor{gray}{Trained end-to-end}} \\
State Encoder & $\mathbf{z}_s = f_\omega(s)$ \\
State-Action Encoder & $\mathbf{z}_{sa} = g_\omega(\mathbf{z}_s,a)$ \\
MDP predictor & $\tilde{\mathbf{z}}_{s'}, \tilde r, \tilde d = \mathbf{z}_{sa}^\top \mathbf{m}$ \\
\rowcolor{sb_gray!10}
\multicolumn{2}{c}{\textcolor{gray}{Decoupled RL}} \\ 
Value & $\tilde Q_i = Q_\ta(\mathbf{z}_{sa})$ \\
Policy & $a_\pi = \pi_\phi(\mathbf{z}_s)$
\\ \bottomrule
\end{tabular}
\end{minipage}
\begin{minipage}{0.49\textwidth}
\centering
\setlength\extrarowheight{2pt}
\begin{tabular}{l}
\toprule
\textbf{Update} MR.Q \\
\midrule
\textbf{if} $t~\%~T_\text{target} = 0$ \textbf{then} \\
\qquad Target networks: $\ta', \phi', \omega' \leftarrow \ta, \phi, \omega$. \\
\qquad Reward scaling: $\bar{r}' \leftarrow \bar r$, $\bar r \leftarrow \mathrm{mean}_D r$. \\
\rowcolor{sb_gray!20}
\qquad \textbf{for} $T_\text{target}$ time steps \textbf{do} \\
\rowcolor{sb_gray!20}
\qquad\qquad Encoder update: \ref{eqn:encoder_full_loss}. \\
\rowcolor{sb_gray!10}
Value update: \ref{eqn:value_loss}. \\
\rowcolor{sb_gray!10}
Policy update: \ref{eqn:policy_loss}. 
\\ \bottomrule
\end{tabular}
\end{minipage}

The encoder loss is composed of three terms based on the reward, dynamics and terminal signal that are unrolled over a short horizon. The value function and policy are trained independently, using standard losses~\citep{DPG, fujimoto2018addressing}. We use LAP~\citep{fujimoto2020equivalence} to sample transitions with priority according to their TD errors~\citep{PrioritizedExpReplay}, the absolute difference between the predicted value and the target value in \ref{eqn:value_loss}. 

The target network, reward scaling (defined in \ref{eqn:value_loss}), and the encoder are updated periodically every $T_\text{target}$ time steps. This synchronized update schedule keeps the input and target output fixed for the downstream value function and policy within each iteration, thus reducing non-stationarity in the optimization~\citep{fujimoto2024sale}. 

\subsubsection{Encoder}

The encoder loss is based on unrolling the dynamics of the learned model over a short horizon. Given a subsequence of an episode $(s_0, a_0, r_1, d_1, s_1,..., r_{H_\text{Enc}}, d_{H_\text{Enc}}, s_{H_\text{Enc}})$, the model is unrolled by encoding the initial state~$s_0$, then by repeatedly applying the state-action encoder~$g_\omega$ and linear MDP predictor~$\mathbf{m}$:
\begin{equation}
     \tilde{\mathbf{z}}^t, \tilde{r}^t, \tilde{d}^t := g_\omega(\tilde{\mathbf{z}}^{t-1}, a^{t-1})^\top \mathbf{m}, \quad \text{ where } \tilde{\mathbf{z}}^0 := f_\omega(s_0). 
\end{equation}
The final loss is summed over the unrolled model and balanced by corresponding hyperparameters: 
\begin{equation} \label{eqn:encoder_full_loss}
    \Loss_\text{Encoder}(f,g,\mathbf{m}) := \sum_{t=1}^{H_\text{Enc}} \lambda_\text{Reward} \Loss_\text{Reward}(\tilde r^t) 
    + \lambda_\text{Dynamics} \Loss_\text{Dynamics}(\tilde{\mathbf{z}}^t_{s'}) 
    + \lambda_\text{Terminal} \Loss_\text{Terminal}(\tilde d^t).
\end{equation}
$\lambda_\text{Terminal}$ is set to $0$ until the first terminal transition (i.e., $d=0$) is viewed. This approach is commonly used in model-based RL~\citep{oh2015action, hafner2023mastering, hansen2024td}, as well as dynamics-based representation learning~\citep{schwarzer2020data, schwarzer2023bigger, scannell2024iqrl}. 

\textbf{Reward loss.} While our theoretical analysis suggests using the mean-squared error to train the predicted reward, we find that a categorical representation of the reward is more effective in practice for predicting sparse rewards and is robust to reward magnitude. This empirical benefit is consistent with prior work~\citep{schrittwieser2020mastering, hafner2023mastering, hansen2024td, wang2024efficientzero}. Our reward loss function uses the cross entropy~CE between the predicted reward~$\tilde r$ and a two-hot encoding of the reward~$r$: 
\begin{equation} \label{eqn:reward_loss}
    \Loss_\text{Reward}(\tilde r) := \text{CE} \lp \tilde r, \text{Two-Hot}(r) \rp. 
\end{equation}
To handle a wide range of reward magnitudes without prior knowledge, the locations of the two-hot encoding are spaced at increasing non-uniform intervals, according to $\mathrm{symexp}(x) = \sign(x) (\exp{(x)} - 1)$~\citep{hafner2023mastering}. %

\textbf{Dynamics loss.} The dynamics loss minimizes the mean-squared error between the predicted next state embedding~$\tilde{\mathbf{z}}_{s'}$ and the next state embedding~$\bar{\mathbf{z}}_{s'}$ from the target encoder~$f_{\omega'}$:
\begin{equation} \label{eqn:dynamics_loss}
    \Loss_\text{Dynamics}(\tilde{\mathbf{z}}_{s'}) := \lp \tilde{\mathbf{z}}_{s'} - \bar{\mathbf{z}}_{s'} \rp^2.
\end{equation}
As discussed in the previous section, using the next state embedding~$\mathbf{z}_{s'}$ eliminates the dependency on the policy that would occur when using a state-action embedding target. 

\textbf{Terminal loss.} 
The predicted scalar terminal signal~$\tilde d$ is trained simply using a MSE loss with the binary terminal signal~$d$: 
\begin{equation} \label{eqn:terminal_loss}
    \Loss_\text{Terminal}(\tilde d) := ( \tilde d - d)^2.
\end{equation}

\subsubsection{Value Function}

Value learning is primarily based on TD3~\citep{fujimoto2018addressing}. Specifically, we train two value functions and take the minimum output between their respective target networks to determine the value target. Similar to TD3, the target action is determined by the target policy~$\pi_{\phi'}$, perturbed by small amount of clipped Gaussian noise: 
\begin{equation} \label{eqn:target_action_noise}
a_{\pi} = \begin{cases}
    \argmax a' & \text{for discrete } A, \\
    \clip(a', -1,1) & \text{for continuous } A,
\end{cases}
\quad \text{where } a' = \pi_{\phi'}(s') + \clip(\e, -c, c), \quad \e \sim \mathcal{N}(0, \sigma^2). 
\end{equation}
Discrete actions are represented by a one-hot encoding, where the Gaussian noise is added to each dimension. Action noise and the clipping is scaled according the range of the action space. 

We modify the TD3 loss in a few ways. Firstly, following numerous prior work across benchmarks~\citep{hessel2018rainbow, barth-maron2018distributional, yarats2022mastering, schwarzer2023bigger}, we predict multi-step returns over a horizon~$H_Q$. Secondly, we use the Huber loss instead of mean-squared error to eliminate bias from prioritized sampling~\citep{fujimoto2020equivalence}. Finally, the target value is normalized according to the average absolute reward~$\bar r$ in the replay buffer: 
\begin{align} \label{eqn:value_loss}
    &\Loss_\text{Value}(\tilde Q_i) := \text{Huber} \lp \tilde Q_i, \frac{1}{\bar r} \lp \sum_{t=0}^{H_Q-1} \y^t r_t + \y^{H_Q} \tilde{Q}'_j\rp \rp, 
    &\tilde{Q}'_j := {\bar r}' \min_{j=1,2} Q_{\ta'_j}(\mathbf{z}_{s_{H_Q} a_{H_Q, \pi}}).
\end{align}
The value ${\bar r}'$ captures the \textit{target} average absolute reward, which is the scaling factor used to the most recently copied value functions~$Q_{\ta'_j}$. This value is updated simultaneously with the target networks \mbox{${\bar r}' \leftarrow \bar r$}. Maintaining a consistent reward scale keeps the loss magnitude constant across different benchmarks, thus improving the robustness of a single set of hyperparameters.

\subsubsection{Policy}

For both continuous and discrete action spaces, the policy is updated using the deterministic policy gradient~\citep{DPG}:
\begin{equation} \label{eqn:policy_loss}
\Loss_\text{Policy}(a_\pi) := - 0.5 \sum_{i=\{1,2\}} \tilde Q_i(\mathbf{z}_{sa_\pi}) + \lambda_\text{pre-activ}\mathbf{z}_\pi^2, \quad\text{where } a_\pi = \text{activ}(\mathbf{z}_\pi). 
\end{equation}
To make the loss universal between action spaces, we use Gumbel-Softmax~\citep{jang2017categorical, lowe2017multi, cianflone2019discrete} for discrete actions, and Tanh for continuous actions. A small regularization penalty is added to the square of the pre-activations~$\mathbf{z}_\pi$ before the policy's final activation to help avoid local minima when the reward, and value, is sparse~\citep{bjorck2021high}. 

For exploration, Gaussian noise is added to each dimension of the action (or one-hot encoding of the action). Similar to \ref{eqn:target_action_noise}, the resulting action vector is clipped to the range of the action space for continuous actions. For discrete actions, the final action is determined by the $\argmax$ operation. 

\section{Experiments}

\begin{figure}[t]
    \centering
    \begin{tikzpicture}[trim axis right]
    \begin{axis}[
        height=0.1615\textwidth,
        width=0.19\textwidth,
        title={
        \shortstack{
        Gym - Locomotion\vphantom{p}\\
        \textcolor{gray}{5 tasks}\vphantom{p}
        }
        },
        ylabel={TD3-Normalized},
        xlabel={Time Steps (1M)},
        xtick={0.0, 40, 80, 120, 160, 200},
        ytick={0.0, 0.3, 0.6, 0.9, 1.2, 1.5},
        xticklabels={0.0, 0.2, 0.4, 0.6, 0.8, 1.0},
        yticklabels={0.0, 0.3, 0.6, 0.9, 1.2, 1.5},
    ]
    \addplot graphics [
    ymin=-0.06798028118092571, ymax=1.6825077195022538,
    xmin=-10.0, xmax=210.0,
    ]{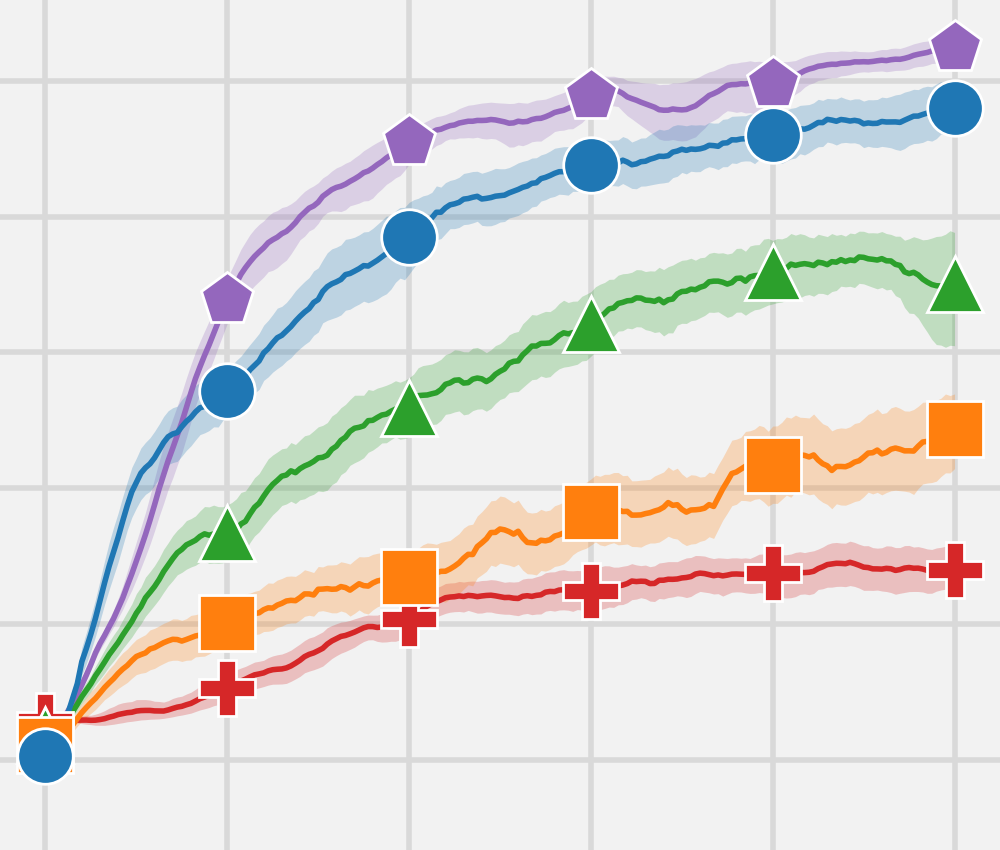};
    \end{axis}
    \end{tikzpicture}
    \begin{tikzpicture}[trim axis right]
    \begin{axis}[
        height=0.1615\textwidth,
        width=0.19\textwidth,
        title={
        \shortstack{
        DMC - Proprioceptive\vphantom{p}\\
        \textcolor{gray}{28 tasks}\vphantom{p}
        }
        },
        ylabel={Total Reward (1k)},
        xlabel={Time Steps (1M)},
        xtick={0, 20, 40, 60, 80, 100},
        xticklabels={0.0, 0.1, 0.2, 0.3, 0.4, 0.5},
        ytick={0, 200, 400, 600, 800},
        yticklabels={0.0, 0.2, 0.4, 0.6, 0.8},
    ]
    \addplot graphics [
    ymin=-20.0, ymax=900.0,
    xmin=-5.0, xmax=105.0,
    ]{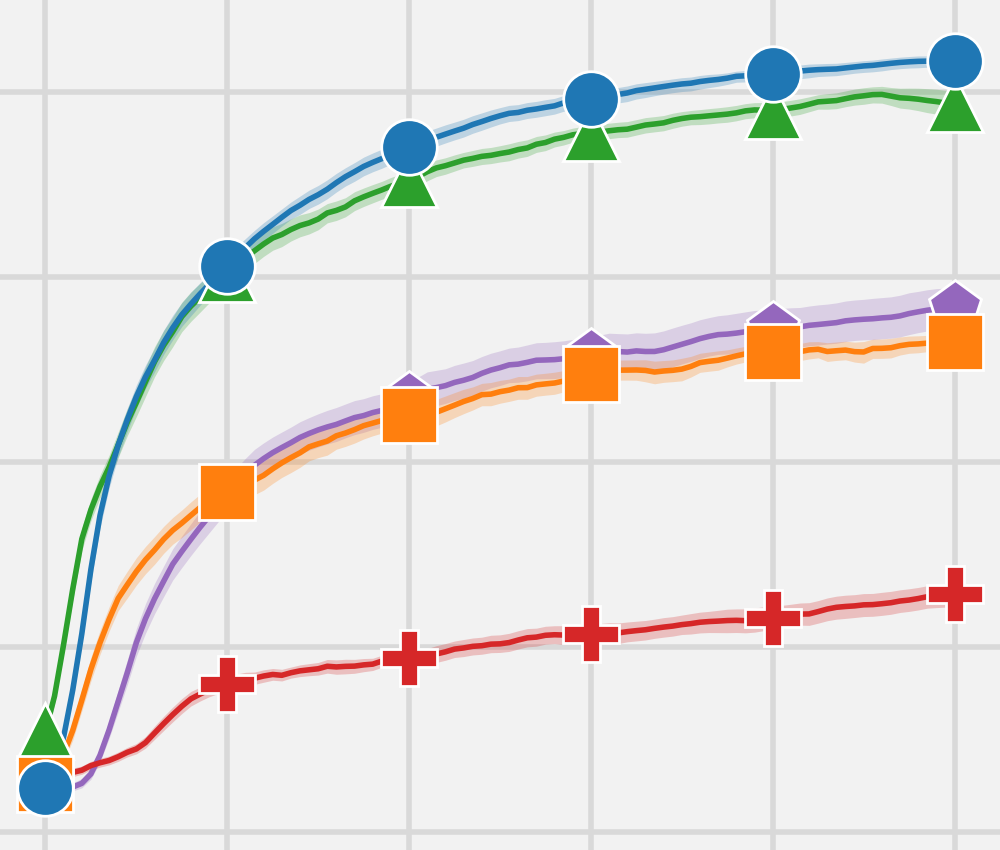};
    \end{axis}
    \end{tikzpicture}
    \begin{tikzpicture}[trim axis right]
    \begin{axis}[
        height=0.1615\textwidth,
        width=0.19\textwidth,
        title={
        \shortstack{
        DMC - Visual\vphantom{p}\\
        \textcolor{gray}{28 tasks}\vphantom{p}
        }
        },
        ylabel={Total Reward (1k)},
        xlabel={Time Steps (1M)},
        xtick={0, 20, 40, 60, 80, 100},
        xticklabels={0.0, 0.1, 0.2, 0.3, 0.4, 0.5},
        ytick={0, 100, 200, 300, 400, 500, 600, 700},
        yticklabels={0.0, 0.1, 0.2, 0.3, 0.4, 0.5, 0.6, 0.7},
    ]
    \addplot graphics [
    ymin=-20.0, ymax=650.0,
    xmin=-5.0, xmax=105.0,
    ]{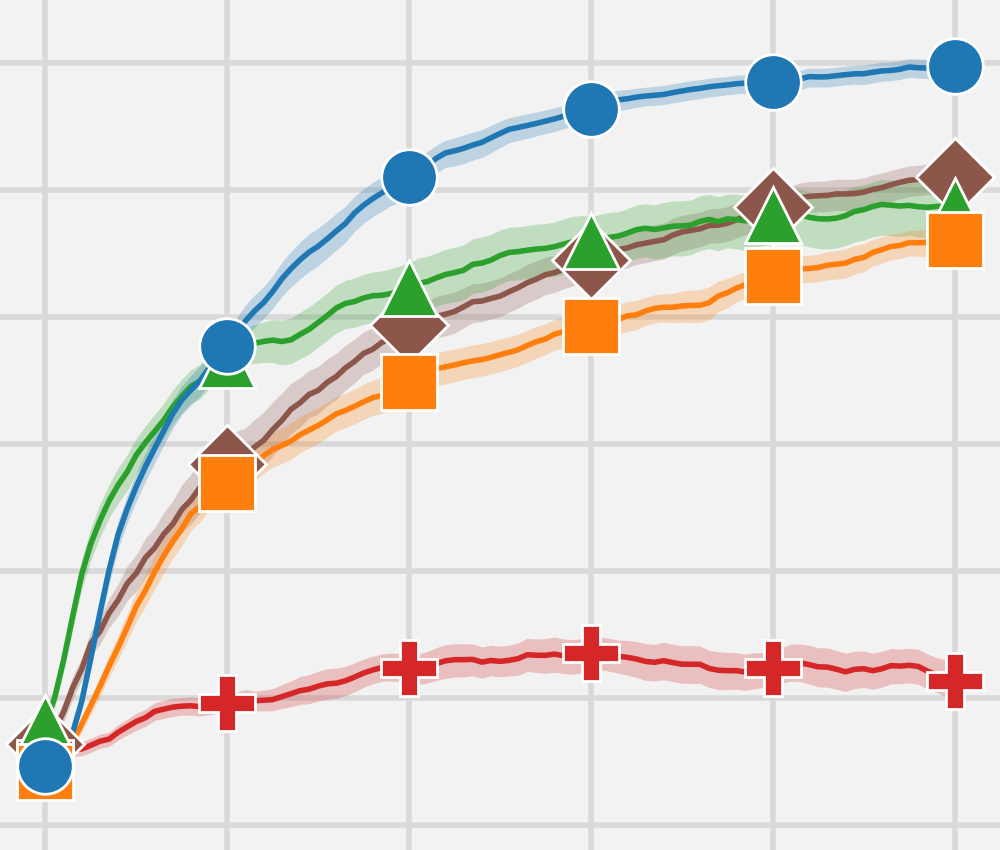};
    \end{axis}
    \end{tikzpicture}    
    \begin{tikzpicture}[trim axis right]
    \begin{axis}[
        height=0.1615\textwidth,
        width=0.19\textwidth,
        title={
        \shortstack{
        Atari\vphantom{p}\\
        \textcolor{gray}{57 tasks}\vphantom{p}
        }
        }, 
        ylabel={Human-Normalized},
        xlabel={Time Steps (1M)},
        xtick={-5.0, 0.0, 5.0, 10.0, 15.0, 20.0, 25.0, 30.0},
        xticklabels={-0.2, 0.0, 0.5, 1.0, 1.5, 2.0, 2.5},
        ytick={-1.0, 0.0, 1.0, 2.0, 3.0, 4.0, 5.0},
        yticklabels={-1.0, 0.0, 1.0, 2.0, 3.0, 4.0, 5.0},
    ]
    \addplot graphics [
    ymin=-0.5, ymax=4.5,
    xmin=-1.25, xmax=26.25,
    ]{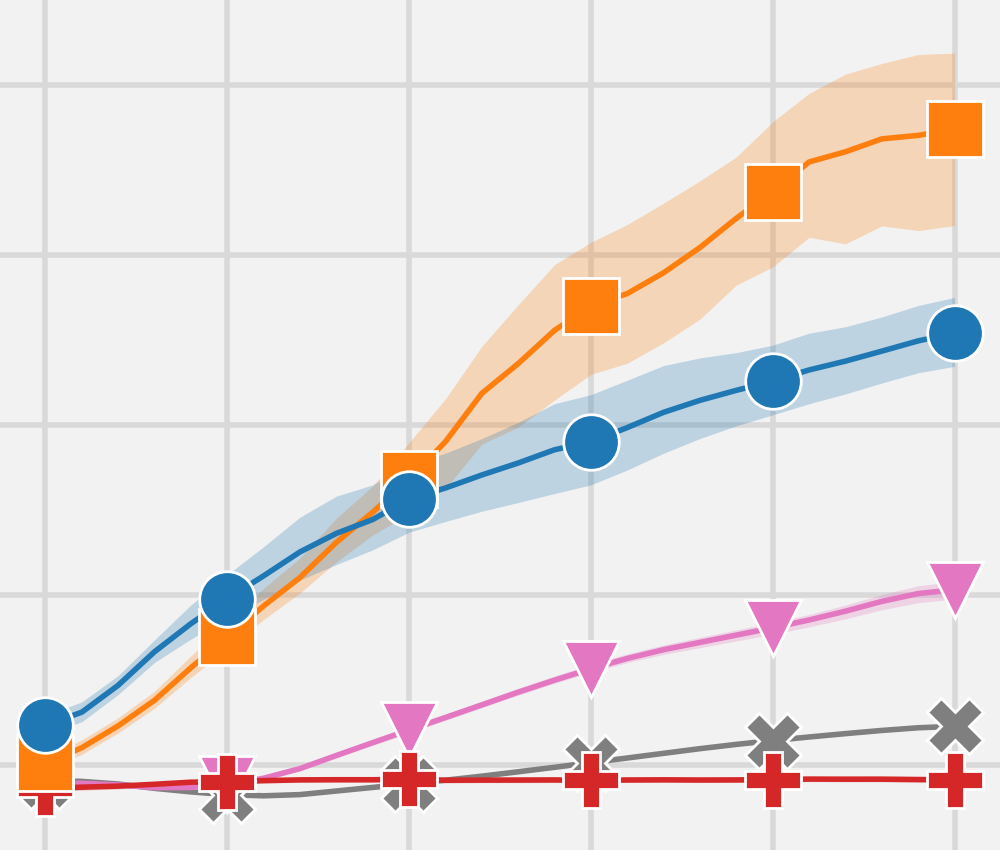};
    \end{axis}
    \end{tikzpicture}

    \fcolorbox{gray!10}{gray!10}{
    \small
    \coloredcircle{sb_blue}{sb_blue!25!light_gray} MR.Q \quad
    \coloredsquare{sb_orange}{sb_orange!25!light_gray} DreamerV3 \quad
    \coloredtriangle{sb_green}{sb_green!25!light_gray} TD-MPC2 \quad
    \coloredplus{sb_red}{sb_red!25!light_gray} PPO \quad
    \coloredpoly{sb_purple}{sb_purple!25!light_gray} TD7 \quad
    \coloreddiamond{sb_brown}{sb_brown!25!light_gray} DrQ-v2 \quad 
    \coloredinvertedtriangle{sb_pink}{sb_pink!25!light_gray} Rainbow \quad
    \coloredx{sb_gray}{sb_gray!25!light_gray} DQN
    }
    
    \captionof{figure}{\textbf{Aggregate learning curves.} Average performance over each benchmark. Results are over 10 seeds. The shaded area captures a 95\% stratified bootstrap confidence interval. Due to action repeat, 500k time steps in DMC correspond to 1M frames in the original environment and 2.5M time steps in Atari corresponds to 10M frames in the original environment. 
    } \label{fig:aggregate_curves}            
\end{figure}

We evaluate MR.Q on four popular RL benchmarks and 118 environments, and compare its performance against strong domain-specific baselines, general model-based approaches, DreamerV3~\citep{hafner2023mastering} and TD-MPC2~\citep{hansen2024td}, and a general model-free algorithm, PPO~\citep{ppo}. Rather than establish MR.Q as the state-of-the-art approach in any particular benchmark, our objective is to demonstrate its broad applicability and effectiveness across a diverse set of tasks with a single set of hyperparameters. 
The baselines use author-suggested default hyperparameters and are fixed across environments. Additional details can be found in \cref{appendix:exp_details}.

\subsection{Main Results}

Aggregate learning curves are displayed in \ref{fig:aggregate_curves}, with full results displayed in \cref{appendix:additional_results}. 

\textbf{Gym - Locomotion.} This subset of the Gym benchmark~\citep{OpenAIGym, towers2024gymnasium} considers 5 locomotion tasks in the MuJoCo simulator~\citep{mujoco} with continuous actions and low level states. Agents are trained for 1M time steps without any environment preprocessing. We evaluate against three baselines: TD7~\citep{fujimoto2024sale}, a state-of-the-art (or near) approach for this benchmark, as well as TD-MPC2, DreamerV3, and PPO. To aggregate results, we normalize using the performance of TD3~\citep{fujimoto2018addressing}.

\textbf{DMC - Proprioceptive.} The DeepMind Control suite (DMC)~\citep{tassa2018deepmind} is a collection of continuous control robotics tasks built on the MuJoCo simulator. These tasks use the proprioceptive states as the observation space, meaning that the input is a vector, and limit the total reward for each episode at 1000, making it easy to aggregate results. 
We report results on all 28 default tasks that were used by either TD-MPC2 or DreamerV3. Agents are trained for 500k time steps, equivalent to 1M frames in the original environment due to action repeat. %
For comparison, we evaluate against the same three algorithms as in the Gym benchmark, with TD-MPC2 considered state-of-the-art (or near) for this benchmark. We also include TD7 due to its strong performance in the Gym benchmark.

\textbf{DMC - Visual.} The visual DMC benchmark includes the same 28 tasks as the proprioceptive benchmark, but uses image-based observations instead. 
Agents are trained for 500k time steps. For baselines, we include \mbox{DrQ-v2}~\citep{yarats2022mastering}, given its state-of-the-art (or near) performance in model-free RL, alongside TD-MPC2, DreamerV3, and PPO. 

\textbf{Atari.} The Atari benchmark is built on the Arcade Learning Environment~\citep{bellemare2013arcade}. This benchmark uses pixel observations and discrete actions and includes the 57 games used by DreamerV3. We follow standard preprocessing steps, including sticky actions~\citep{machado2018revisiting} (full details in \cref{appendix:envs}). Agents are trained for 2.5M time steps (equivalent to 10M frames), a setting which has been considered by prior work~\citep{sokar2023dormant}. For comparison, we evaluate against three baselines: the model-based approach DreamerV3, as well as model-free approaches, DQN~\citep{DQN}, Rainbow~\citep{hessel2018rainbow}, and PPO.
Results are aggregated by normalizing scores against human performance. 

\textbf{Discussion.} Throughout our experiments, we find the presence of ``no free lunch'', where the top-performing baseline in one benchmark fails to replicate its success in another. Regardless, MR.Q achieves the highest performance in both DMC benchmarks, showcasing its ability to handle different observation spaces. Although it falls slightly behind TD7 in the Gym benchmark, MR.Q is the strongest method overall across all continuous control benchmarks. In Atari, while DreamerV3 outperforms MR.Q, it relies on a model with 40 times more parameters and struggles comparatively in the remaining benchmarks. When compared to the model-free baselines, MR.Q surpasses PPO, DQN, and Rainbow, demonstrating its effectiveness with discrete action spaces.

\subsection{Design Study} \label{sec:design}

To better understand the impact of certain design choices and hyperparameters, we attempt variations of MR.Q, and report the aggregate results in \ref{table:design}.

\begin{table}[ht]
\footnotesize
\centering
\setlength{\tabcolsep}{0pt}
\newcolumntype{Y}{>{\centering\arraybackslash}X} 
\caption{\textbf{Design study.} Average difference in normalized performance from varying design choices across each benchmark over 5 seeds. %
Negative changes are highlighted lightly \hlfancy{sb_red!10}{$[-0.01,-0.2)$}. Damaging changes are highlighted moderately \hlfancy{sb_red!25}{$[-0.2,-0.5)$}. Catastrophic changes are highlighted boldly \hlfancy{sb_red!40}{$(\leq -0.5)$}. Positive changes are similarly highlighted \hlfancy{sb_green!10}{$(> 0.01)$}.} \label{table:design}
\vspace{-4pt}
\begin{tabularx}{\textwidth}{l@{\hspace{1pt}} r@{}X r@{}X r@{}X r@{}X}
\toprule
\multirow{2}{*}{Design} & \multicolumn{2}{c}{Gym - Locomotion} & \multicolumn{2}{c}{DMC - Proprioceptive} & \multicolumn{2}{c}{DMC - Visual} & \multicolumn{2}{c}{Atari - 1M} \\
 & \multicolumn{2}{c}{\textcolor{gray}{TD3-Normalized}} & \multicolumn{2}{c}{\textcolor{gray}{Reward (1k)}} & \multicolumn{2}{c}{\textcolor{gray}{Reward (1k)}} & \multicolumn{2}{c}{\textcolor{gray}{Human-Normalized}} \\ 
\midrule
& \multicolumn{8}{c}{Relaxations} \\
\midrule
Linear value function 
& \cellcolor{sb_red!40} -1.17 & \cellcolor{sb_red!40}~\textcolor{gray}{[-1.19, -1.15]}
& \cellcolor{sb_red!40} \z-0.58 & \cellcolor{sb_red!40}~\textcolor{gray}{[-0.59, -0.56]}
& \cellcolor{sb_red!25} \z-0.41 & \cellcolor{sb_red!25}~\textcolor{gray}{[-0.42, -0.39]}
& \cellcolor{sb_red!40} -1.35 & \cellcolor{sb_red!40}~\textcolor{gray}{[-1.41, -1.29]} \\
Dynamics target 
& \cellcolor{sb_red!10} -0.10 & \cellcolor{sb_red!10}~\textcolor{gray}{[-0.17, -0.04]}
& \cellcolor{sb_red!10} -0.15 & \cellcolor{sb_red!10}~\textcolor{gray}{[-0.15, -0.15]}
& \cellcolor{sb_red!10} -0.05 & \cellcolor{sb_red!10}~\textcolor{gray}{[-0.05, -0.04]} 
& \cellcolor{sb_red!25} -0.38 & \cellcolor{sb_red!25}~\textcolor{gray}{[-0.81, 0.05]} \\
No target encoder 
& \cellcolor{sb_red!40} -0.53 & \cellcolor{sb_red!40}~\textcolor{gray}{[-0.60, -0.46]}
& \cellcolor{sb_red!25} -0.35 & \cellcolor{sb_red!25}~\textcolor{gray}{[-0.35, -0.34]}
& \cellcolor{sb_red!10} -0.15 & \cellcolor{sb_red!10}~\textcolor{gray}{[-0.15, -0.15]} 
& \cellcolor{sb_red!40} -0.86 & \cellcolor{sb_red!40}~\textcolor{gray}{[-0.89, -0.83]} \\
Revert 
& \cellcolor{sb_red!40} -1.47 & \cellcolor{sb_red!40}~\textcolor{gray}{[-1.54, -1.39]}
& \cellcolor{sb_red!40} -0.72 & \cellcolor{sb_red!40}~\textcolor{gray}{[-0.73, -0.72]}
& \cellcolor{sb_red!40} -0.52 & \cellcolor{sb_red!40}~\textcolor{gray}{[-0.52, -0.51]}
& \cellcolor{sb_red!40} -1.69 & \cellcolor{sb_red!40}~\textcolor{gray}{[-1.70, -1.67]} \\
Non-linear model 
& \cellcolor{sb_red!10} -0.01 & \cellcolor{sb_red!10}~\textcolor{gray}{[-0.07, 0.03]}
& \cellcolor{white} -0.00 & \cellcolor{white}~\textcolor{gray}{[-0.02, 0.01]} 
& \cellcolor{white} -0.01 & \cellcolor{white}~\textcolor{gray}{[-0.02, -0.00]} 
& \cellcolor{sb_red!10} -0.07 & \cellcolor{sb_red!10}~\textcolor{gray}{[-0.32, 0.18]} \\
\midrule
& \multicolumn{8}{c}{Loss functions} \\
\midrule
MSE reward loss
& \cellcolor{sb_green!10} 0.10 & \cellcolor{sb_green!10}~\textcolor{gray}{[-0.02, 0.19]}
& \cellcolor{sb_red!10} -0.06 & \cellcolor{sb_red!10}~\textcolor{gray}{[-0.08, -0.05]} 
& \cellcolor{sb_red!10} -0.05 & \cellcolor{sb_red!10}~\textcolor{gray}{[-0.07, -0.04]} 
& \cellcolor{sb_red!40} -0.79 & \cellcolor{sb_red!40}~\textcolor{gray}{[-0.86, -0.73]} \\
No reward scaling
& \cellcolor{sb_red!10} -0.04 & \cellcolor{sb_red!10}~\textcolor{gray}{[-0.09, 0.02]}
& \cellcolor{white} -0.01 & \cellcolor{white}~\textcolor{gray}{[-0.02, 0.00]} 
& \cellcolor{white} -0.00 & \cellcolor{white}~\textcolor{gray}{[-0.01, 0.01]} 
& \cellcolor{sb_green!10} 0.18 & \cellcolor{sb_green!10}~\textcolor{gray}{[-0.25, 0.56]} \\
No min
& \cellcolor{sb_red!10} -0.09 & \cellcolor{sb_red!10}~\textcolor{gray}{[-0.16, -0.01]}
& \cellcolor{white} -0.01 & \cellcolor{white}~\textcolor{gray}{[-0.02, 0.01]} 
& \cellcolor{white} 0.00 & \cellcolor{white}~\textcolor{gray}{[-0.01, 0.01]} 
& \cellcolor{sb_green!10} 0.13 & \cellcolor{sb_green!10}~\textcolor{gray}{[-0.10, 0.58]} \\
No LAP
& \cellcolor{sb_red!10} -0.10 & \cellcolor{sb_red!10}~\textcolor{gray}{[-0.24, -0.00]}
& \cellcolor{white} 0.00 & \cellcolor{white}~\textcolor{gray}{[-0.00, 0.01]} 
& \cellcolor{sb_red!10} -0.01 & \cellcolor{sb_red!10}~\textcolor{gray}{[-0.02, -0.01]}
& \cellcolor{sb_red!10} -0.13 & \cellcolor{sb_red!10}~\textcolor{gray}{[-0.38, 0.14]} \\
No MR
& \cellcolor{sb_red!40} -0.56 & \cellcolor{sb_red!40}~\textcolor{gray}{[-0.69, -0.43]}
& \cellcolor{sb_red!10} -0.19 & \cellcolor{sb_red!10}~\textcolor{gray}{[-0.19, -0.18]} 
& \cellcolor{sb_red!10} -0.07 & \cellcolor{sb_red!10}~\textcolor{gray}{[-0.09, -0.03]} 
& \cellcolor{sb_red!40} -0.78 & \cellcolor{sb_red!40}~\textcolor{gray}{[-0.88, -0.69]}  \\
\midrule
& \multicolumn{8}{c}{Horizons} \\
\midrule
1-step return
& \cellcolor{sb_red!25} -0.33 & \cellcolor{sb_red!25}~\textcolor{gray}{[-0.46, -0.21]}
& \cellcolor{sb_red!10} -0.04 & \cellcolor{sb_red!10}~\textcolor{gray}{[-0.05, -0.02]} 
& \cellcolor{sb_red!10} -0.03 & \cellcolor{sb_red!10}~\textcolor{gray}{[-0.03, -0.02]} 
& \cellcolor{sb_red!40} -0.70 & \cellcolor{sb_red!40}~\textcolor{gray}{[-0.81, -0.59]} \\
No unroll
& \cellcolor{sb_green!10} 0.07 & \cellcolor{sb_green!10}~\textcolor{gray}{[0.01, 0.14]}
& \cellcolor{white} -0.01 & \cellcolor{white}~\textcolor{gray}{[-0.01, -0.00]}
& \cellcolor{sb_red!10} -0.04 & \cellcolor{sb_red!10}~\textcolor{gray}{[-0.06, -0.01]}
& \cellcolor{sb_red!25} -0.33 & \cellcolor{sb_red!25}~\textcolor{gray}{[-0.41, -0.28]} \\
\bottomrule
\end{tabularx}
\end{table}

\textbf{Relaxations.} In \cref{sec:theory}, we outlined a loss~(\ref{eqn:linear_features_loss}) that, if globally minimized, would provide features that are linear with the true value function. MR.Q in practice relaxes this theoretical result by modifying the loss and using a non-linear value function. In \textbf{Linear value function}, we replace the non-linear value function with a linear function. In \textbf{Dynamics target}, we replace the state embedding dynamics target with a state-action embedding~$\bar{\mathbf{z}}_{s'a'}$ determined from the target state-action encoder~$g_\omega$. In \textbf{No target encoder}, we use the current encoder to generate the dynamics target~$\mathbf{z}_{s'a'}$, and jointly optimize it within the encoder loss. In \textbf{Revert}, we consider all of the aforementioned changes simultaneously, using linear value functions and setting the dynamics target as a state-action embedding determined by the current encoder. In \textbf{Non-linear model}, we replace the linear MDP predictor with individual networks that predict each component separately from $\mathbf{z}_{sa}$. 

\textbf{Loss functions.} MR.Q's loss functions use several unconventional choices. In \textbf{MSE reward loss}, we replace the categorical loss function on the predicted reward in \ref{eqn:reward_loss} with the mean-squared error (MSE). In \textbf{No reward scaling}, we remove the reward scaling in \ref{eqn:value_loss}, setting $\bar r = \bar{r}'=1$.  
In \textbf{No min}, we take the mean over the target value functions instead of the minimum in \ref{eqn:value_loss}. In \textbf{No LAP}, we remove prioritized sampling~\citep{fujimoto2020equivalence} and use the MSE instead of the Huber loss in the value update. Lastly, in \textbf{No MR}, we remove model-based representation learning and train the encoder end-to-end with the value function. 

\textbf{Horizons.} Finally, we consider the role of extended predictions. In \textbf{1-step return}, we remove multi-step value predictions and use TD learning. In \textbf{No unroll}, we remove the dynamics unrolling in \ref{eqn:encoder_full_loss}, by setting the encoder horizon~$H_\text{Enc}=1$. 

\textbf{Discussion.} The results of our design study show the benefit of balancing theory with practical relaxations. 
The experiments further validate our design choices and hyperparameters. We highlight two results in particular: (1) increasing the model capacity in the ``non-linear model'' experiment, does not improve performance. This outcome suggests that maintaining an approximately linear relationship with the value function can be more impactful than increased capacity. (2) Our study also reveals a key distinction between the Gym and Atari benchmarks---while the ``MSE reward loss'' and ``No unroll'' variants offer moderate performance gains in Gym, they significantly degrade performance in Atari. This discrepancy highlights how hyperparameters can overfit to individual benchmarks, emphasizing the importance of evaluating algorithms across multiple benchmarks.

\section{Discussion and Conclusion}

This paper introduces MR.Q, a general model-free deep RL algorithm that achieves strong performance across diverse benchmarks and environments. Drawing inspiration from the theory of model-based representation learning, MR.Q demonstrates that model-free deep RL is a promising avenue for building general-purpose algorithms that achieve high performance across environments, while being simpler and less expensive than model-based alternatives.

Our work also reveals insights on which design choices matter when building general-purpose model-free deep RL algorithms and how common benchmarks respond to these design choices.

\textbf{Model-based and model-free RL.} MR.Q integrates model-based objectives with a model-free backbone during training, effectively blurring the boundary between traditional model-based and model-free RL. While MR.Q could be extended to the model-based setting by incorporating planning or simulated trajectories with the state-action encoder, these components can add significant execution time and increase the overall complexity and tuning required by a method. Moreover, the performance of MR.Q in these common RL benchmarks demonstrates that these model-based components may be simply unnecessary---suggesting that the representation itself could be the most valuable aspect of model-based learning, even in methods that do use planning. This argument is echoed by DreamerV3 and TD-MPC2, which rely on short planning horizons and trajectory generation, while including both value functions and traditional model-free policy updates. As such, it may be necessary to examine more complex settings, to reliably see a benefit from model-based search or planning, e.g.,~\citep{alphago}.

\textbf{Universality of RL benchmarks.} Our results demonstrate that there is a striking lack of positive transfer between benchmarks. For example, despite the similarities in tasks and the same underlying MuJoCo simulator, the top performers in Gym and DMC fail to replicate their success on the opposing benchmark. Similarly, although DreamerV3 excels at Atari, these performance benefits do not translate to continuous control environments, underperforming TD3 in Gym and outright failing to learn the Dog and Humanoid tasks in DMC~(see \cref{appendix:additional_results}). These findings show the limitations of single-benchmark evaluations, indicating that success on one benchmark may not translate easily to others, and highlights the need for more comprehensive benchmarks.

\textbf{Limitations.} MR.Q is only the first step towards a new generation of general-purpose model-free deep RL algorithms. Many challenges remains for a fully general algorithm. In particular, MR.Q is not equipped to handle settings such as hard exploration tasks or non-Markovian environments. Another limitation is our evaluation only considers standard RL benchmarks. Although this allows direct comparison with other methods, established algorithms such as PPO have demonstrated their effectiveness in highly unique settings, such as team video games~\citep{berner2019dota}, drone racing~\citep{kaufmann2023champion}, and large language models~\citep{achiam2023gpt, touvron2023llama}. To demonstrate similar versatility, new algorithms must undergo the same rigorous testing across a range of tasks that is beyond the scope of any single study.

As the community continues to push the boundaries of what is possible with deep RL, we believe that building simpler general-purpose algorithms has the potential to make this technology more accessible to a wider audience, ultimately enabling users to train agents with ease. Perhaps one day---with just the click of a button. \clearpage

\subsubsection*{Acknowledgments}

We would like to thank Brandon Amos, Mikhael Henaff, Luis Pineda, Paria Rashidinejad, and Qinqing Zheng for insightful discussions and comments. 

\bibliography{iclr2025_conference}

\begin{thebibliography}{119}
\providecommand{\natexlab}[1]{#1}
\providecommand{\url}[1]{\texttt{#1}}
\expandafter\ifx\csname urlstyle\endcsname\relax
  \providecommand{\doi}[1]{doi: #1}\else
  \providecommand{\doi}{doi: \begingroup \urlstyle{rm}\Url}\fi

\bibitem[Achiam et~al.(2023)Achiam, Adler, Agarwal, Ahmad, Akkaya, Aleman, Almeida, Altenschmidt, Altman, Anadkat, et~al.]{achiam2023gpt}
Josh Achiam, Steven Adler, Sandhini Agarwal, Lama Ahmad, Ilge Akkaya, Florencia~Leoni Aleman, Diogo Almeida, Janko Altenschmidt, Sam Altman, Shyamal Anadkat, et~al.
\newblock Gpt-4 technical report.
\newblock \emph{arXiv preprint arXiv:2303.08774}, 2023.

\bibitem[Agarwal et~al.(2020)Agarwal, Kakade, Krishnamurthy, and Sun]{agarwal2020flambe}
Alekh Agarwal, Sham Kakade, Akshay Krishnamurthy, and Wen Sun.
\newblock Flambe: Structural complexity and representation learning of low rank mdps.
\newblock \emph{Advances in neural information processing systems}, 33:\penalty0 20095--20107, 2020.

\bibitem[Assran et~al.(2023)Assran, Duval, Misra, Bojanowski, Vincent, Rabbat, LeCun, and Ballas]{assran2023self}
Mahmoud Assran, Quentin Duval, Ishan Misra, Piotr Bojanowski, Pascal Vincent, Michael Rabbat, Yann LeCun, and Nicolas Ballas.
\newblock Self-supervised learning from images with a joint-embedding predictive architecture.
\newblock In \emph{Proceedings of the IEEE/CVF Conference on Computer Vision and Pattern Recognition}, pp.\  15619--15629, 2023.

\bibitem[Ba et~al.(2016)Ba, Kiros, and Hinton]{ba2016layer}
Jimmy~Lei Ba, Jamie~Ryan Kiros, and Geoffrey~E Hinton.
\newblock Layer normalization.
\newblock \emph{arXiv preprint arXiv:1607.06450}, 2016.

\bibitem[Back(1996)]{back1996evolutionary}
Thomas Back.
\newblock \emph{Evolutionary algorithms in theory and practice: evolution strategies, evolutionary programming, genetic algorithms}.
\newblock Oxford university press, 1996.

\bibitem[Badia et~al.(2020)Badia, Piot, Kapturowski, Sprechmann, Vitvitskyi, Guo, and Blundell]{badia2020agent57}
Adri{\`a}~Puigdom{\`e}nech Badia, Bilal Piot, Steven Kapturowski, Pablo Sprechmann, Alex Vitvitskyi, Zhaohan~Daniel Guo, and Charles Blundell.
\newblock Agent57: Outperforming the atari human benchmark.
\newblock In \emph{International Conference on Machine Learning}, pp.\  507--517. PMLR, 2020.

\bibitem[Baird(1995)]{baird1995residual}
Leemon Baird.
\newblock Residual algorithms: Reinforcement learning with function approximation.
\newblock In \emph{Machine Learning Proceedings 1995}, pp.\  30--37. Elsevier, 1995.

\bibitem[Barth-Maron et~al.(2018)Barth-Maron, Hoffman, Budden, Dabney, Horgan, TB, Muldal, Heess, and Lillicrap]{barth-maron2018distributional}
Gabriel Barth-Maron, Matthew~W Hoffman, David Budden, Will Dabney, Dan Horgan, Dhruva TB, Alistair Muldal, Nicolas Heess, and Timothy Lillicrap.
\newblock Distributional policy gradients.
\newblock \emph{International Conference on Learning Representations}, 2018.

\bibitem[Bellemare et~al.(2013)Bellemare, Naddaf, Veness, and Bowling]{bellemare2013arcade}
Marc~G Bellemare, Yavar Naddaf, Joel Veness, and Michael Bowling.
\newblock The arcade learning environment: An evaluation platform for general agents.
\newblock \emph{Journal of Artificial Intelligence Research}, 47:\penalty0 253--279, 2013.

\bibitem[Bellman(1957)]{bellman1957markovian}
Richard Bellman.
\newblock A markovian decision process.
\newblock \emph{Journal of mathematics and mechanics}, pp.\  679--684, 1957.

\bibitem[Berner et~al.(2019)Berner, Brockman, Chan, Cheung, Debiak, Dennison, Farhi, Fischer, Hashme, Hesse, et~al.]{berner2019dota}
Christopher Berner, Greg Brockman, Brooke Chan, Vicki Cheung, Przemyslaw Debiak, Christy Dennison, David Farhi, Quirin Fischer, Shariq Hashme, Chris Hesse, et~al.
\newblock Dota 2 with large scale deep reinforcement learning.
\newblock \emph{arXiv preprint arXiv:1912.06680}, 2019.

\bibitem[Bjorck et~al.(2021)Bjorck, Gomes, and Weinberger]{bjorck2021high}
Johan Bjorck, Carla~P Gomes, and Kilian~Q Weinberger.
\newblock Is high variance unavoidable in rl? a case study in continuous control.
\newblock In \emph{International Conference on Learning Representations}, 2021.

\bibitem[Brockman et~al.(2016)Brockman, Cheung, Pettersson, Schneider, Schulman, Tang, and Zaremba]{OpenAIGym}
Greg Brockman, Vicki Cheung, Ludwig Pettersson, Jonas Schneider, John Schulman, Jie Tang, and Wojciech Zaremba.
\newblock Openai gym, 2016.

\bibitem[Castro(2020)]{castro2020scalable}
Pablo~Samuel Castro.
\newblock Scalable methods for computing state similarity in deterministic markov decision processes.
\newblock In \emph{Proceedings of the AAAI Conference on Artificial Intelligence}, volume~34, pp.\  10069--10076, 2020.

\bibitem[Castro et~al.(2018)Castro, Moitra, Gelada, Kumar, and Bellemare]{castro2018dopamine}
Pablo~Samuel Castro, Subhodeep Moitra, Carles Gelada, Saurabh Kumar, and Marc~G Bellemare.
\newblock Dopamine: A research framework for deep reinforcement learning.
\newblock \emph{arXiv preprint arXiv:1812.06110}, 2018.

\bibitem[Castro et~al.(2021)Castro, Kastner, Panangaden, and Rowland]{castro2021mico}
Pablo~Samuel Castro, Tyler Kastner, Prakash Panangaden, and Mark Rowland.
\newblock Mico: Improved representations via sampling-based state similarity for markov decision processes.
\newblock \emph{Advances in Neural Information Processing Systems}, 34:\penalty0 30113--30126, 2021.

\bibitem[Chung et~al.(2019)Chung, Nath, Joseph, and White]{chung2018twotimescale}
Wesley Chung, Somjit Nath, Ajin Joseph, and Martha White.
\newblock Two-timescale networks for nonlinear value function approximation.
\newblock In \emph{International Conference on Learning Representations}, 2019.

\bibitem[Cianflone et~al.(2019)Cianflone, Ahmed, Islam, Bose, and Hamilton]{cianflone2019discrete}
Andre Cianflone, Zafarali Ahmed, Riashat Islam, Avishek~Joey Bose, and William~L Hamilton.
\newblock Discrete off-policy policy gradient using continuous relaxations, 2019.

\bibitem[Clevert et~al.(2015)Clevert, Unterthiner, and Hochreiter]{clevert2015fast}
Djork-Arn{\'e} Clevert, Thomas Unterthiner, and Sepp Hochreiter.
\newblock Fast and accurate deep network learning by exponential linear units (elus).
\newblock \emph{arXiv preprint arXiv:1511.07289}, 2015.

\bibitem[Dayan(1993)]{dayan1993improving}
Peter Dayan.
\newblock Improving generalization for temporal difference learning: The successor representation.
\newblock \emph{Neural Computation}, 5\penalty0 (4):\penalty0 613--624, 1993.

\bibitem[Engstrom et~al.(2020)Engstrom, Ilyas, Santurkar, Tsipras, Janoos, Rudolph, and Madry]{Engstrom2020Implementation}
Logan Engstrom, Andrew Ilyas, Shibani Santurkar, Dimitris Tsipras, Firdaus Janoos, Larry Rudolph, and Aleksander Madry.
\newblock Implementation matters in deep rl: A case study on ppo and trpo.
\newblock In \emph{International Conference on Learning Representations}, 2020.

\bibitem[Ferns et~al.(2004)Ferns, Panangaden, and Precup]{ferns2004metrics}
Norm Ferns, Prakash Panangaden, and Doina Precup.
\newblock Metrics for finite markov decision processes.
\newblock In \emph{UAI}, volume~4, pp.\  162--169, 2004.

\bibitem[Ferns et~al.(2011)Ferns, Panangaden, and Precup]{ferns2011bisimulation}
Norm Ferns, Prakash Panangaden, and Doina Precup.
\newblock Bisimulation metrics for continuous markov decision processes.
\newblock \emph{SIAM Journal on Computing}, 40\penalty0 (6):\penalty0 1662--1714, 2011.

\bibitem[Finn et~al.(2016)Finn, Tan, Duan, Darrell, Levine, and Abbeel]{finn2016deep}
Chelsea Finn, Xin~Yu Tan, Yan Duan, Trevor Darrell, Sergey Levine, and Pieter Abbeel.
\newblock Deep spatial autoencoders for visuomotor learning.
\newblock In \emph{2016 IEEE International Conference on Robotics and Automation (ICRA)}, pp.\  512--519. IEEE, 2016.

\bibitem[Fujimoto et~al.(2018)Fujimoto, van Hoof, and Meger]{fujimoto2018addressing}
Scott Fujimoto, Herke van Hoof, and David Meger.
\newblock Addressing function approximation error in actor-critic methods.
\newblock In \emph{International Conference on Machine Learning}, volume~80, pp.\  1587--1596. PMLR, 2018.

\bibitem[Fujimoto et~al.(2020)Fujimoto, Meger, and Precup]{fujimoto2020equivalence}
Scott Fujimoto, David Meger, and Doina Precup.
\newblock An equivalence between loss functions and non-uniform sampling in experience replay.
\newblock \emph{Advances in Neural Information Processing Systems}, 33, 2020.

\bibitem[Fujimoto et~al.(2021)Fujimoto, Meger, and Precup]{fujimoto2021srdice}
Scott Fujimoto, David Meger, and Doina Precup.
\newblock A deep reinforcement learning approach to marginalized importance sampling with the successor representation.
\newblock In \emph{Proceedings of the 38th International Conference on Machine Learning}, volume 139, pp.\  3518--3529. PMLR, 2021.

\bibitem[Fujimoto et~al.(2022)Fujimoto, Meger, Precup, Nachum, and Gu]{fujimoto2022should}
Scott Fujimoto, David Meger, Doina Precup, Ofir Nachum, and Shixiang~Shane Gu.
\newblock Why should i trust you, bellman? {T}he {B}ellman error is a poor replacement for value error.
\newblock In \emph{International Conference on Machine Learning}, volume 162, pp.\  6918--6943. PMLR, 2022.

\bibitem[Fujimoto et~al.(2024)Fujimoto, Chang, Smith, Gu, Precup, and Meger]{fujimoto2024sale}
Scott Fujimoto, Wei-Di Chang, Edward~J. Smith, Shixiang~Shane Gu, Doina Precup, and David Meger.
\newblock For {SALE}: State-action representation learning for deep reinforcement learning.
\newblock In \emph{Thirty-seventh Conference on Neural Information Processing Systems}, 2024.

\bibitem[Gelada et~al.(2019)Gelada, Kumar, Buckman, Nachum, and Bellemare]{gelada2019deepmdp}
Carles Gelada, Saurabh Kumar, Jacob Buckman, Ofir Nachum, and Marc~G Bellemare.
\newblock Deepmdp: Learning continuous latent space models for representation learning.
\newblock In \emph{International Conference on Machine Learning}, pp.\  2170--2179. PMLR, 2019.

\bibitem[Glorot \& Bengio(2010)Glorot and Bengio]{glorot2010understanding}
Xavier Glorot and Yoshua Bengio.
\newblock Understanding the difficulty of training deep feedforward neural networks.
\newblock In \emph{Proceedings of the thirteenth international conference on artificial intelligence and statistics}, pp.\  249--256. JMLR Workshop and Conference Proceedings, 2010.

\bibitem[Grill et~al.(2020)Grill, Strub, Altch{\'e}, Tallec, Richemond, Buchatskaya, Doersch, Avila~Pires, Guo, Gheshlaghi~Azar, et~al.]{grill2020bootstrap}
Jean-Bastien Grill, Florian Strub, Florent Altch{\'e}, Corentin Tallec, Pierre Richemond, Elena Buchatskaya, Carl Doersch, Bernardo Avila~Pires, Zhaohan Guo, Mohammad Gheshlaghi~Azar, et~al.
\newblock Bootstrap your own latent-a new approach to self-supervised learning.
\newblock \emph{Advances in neural information processing systems}, 33:\penalty0 21271--21284, 2020.

\bibitem[Guo et~al.(2020)Guo, Pires, Piot, Grill, Altch{\'e}, Munos, and Azar]{guo2020bootstrap}
Zhaohan~Daniel Guo, Bernardo~Avila Pires, Bilal Piot, Jean-Bastien Grill, Florent Altch{\'e}, R{\'e}mi Munos, and Mohammad~Gheshlaghi Azar.
\newblock Bootstrap latent-predictive representations for multitask reinforcement learning.
\newblock In \emph{International Conference on Machine Learning}, pp.\  3875--3886. PMLR, 2020.

\bibitem[Guo et~al.(2022)Guo, Thakoor, P{\^\i}slar, Avila~Pires, Altch{\'e}, Tallec, Saade, Calandriello, Grill, Tang, et~al.]{guo2022byol}
Zhaohan~Daniel Guo, Shantanu Thakoor, Miruna P{\^\i}slar, Bernardo Avila~Pires, Florent Altch{\'e}, Corentin Tallec, Alaa Saade, Daniele Calandriello, Jean-Bastien Grill, Yunhao Tang, et~al.
\newblock Byol-explore: Exploration by bootstrapped prediction.
\newblock \emph{Advances in neural information processing systems}, 35:\penalty0 31855--31870, 2022.

\bibitem[Ha \& Schmidhuber(2018)Ha and Schmidhuber]{ha2018world}
David Ha and J{\"u}rgen Schmidhuber.
\newblock World models.
\newblock \emph{arXiv preprint arXiv:1803.10122}, 2018.

\bibitem[Haarnoja et~al.(2018)Haarnoja, Zhou, Abbeel, and Levine]{haarnoja2018soft}
Tuomas Haarnoja, Aurick Zhou, Pieter Abbeel, and Sergey Levine.
\newblock Soft actor-critic: Off-policy maximum entropy deep reinforcement learning with a stochastic actor.
\newblock In \emph{International Conference on Machine Learning}, volume~80, pp.\  1861--1870. PMLR, 2018.

\bibitem[Hafner et~al.(2019)Hafner, Lillicrap, Fischer, Villegas, Ha, Lee, and Davidson]{hafner2019learning}
Danijar Hafner, Timothy Lillicrap, Ian Fischer, Ruben Villegas, David Ha, Honglak Lee, and James Davidson.
\newblock Learning latent dynamics for planning from pixels.
\newblock In \emph{International conference on machine learning}, pp.\  2555--2565. PMLR, 2019.

\bibitem[Hafner et~al.(2023)Hafner, Pasukonis, Ba, and Lillicrap]{hafner2023mastering}
Danijar Hafner, Jurgis Pasukonis, Jimmy Ba, and Timothy Lillicrap.
\newblock Mastering diverse domains through world models.
\newblock \emph{arXiv preprint arXiv:2301.04104}, 2023.

\bibitem[Hansen et~al.(2024)Hansen, Su, and Wang]{hansen2024td}
Nicklas Hansen, Hao Su, and Xiaolong Wang.
\newblock Td-mpc2: Scalable, robust world models for continuous control.
\newblock In \emph{The Twelfth International Conference on Learning Representations}, 2024.

\bibitem[Hansen et~al.(2022)Hansen, Su, and Wang]{hansen2022temporal}
Nicklas~A Hansen, Hao Su, and Xiaolong Wang.
\newblock Temporal difference learning for model predictive control.
\newblock In \emph{International Conference on Machine Learning}, pp.\  8387--8406. PMLR, 2022.

\bibitem[Harris et~al.(2020)Harris, Millman, van~der Walt, Gommers, Virtanen, Cournapeau, Wieser, Taylor, Berg, Smith, Kern, Picus, Hoyer, van Kerkwijk, Brett, Haldane, del R{\'{i}}o, Wiebe, Peterson, G{\'{e}}rard-Marchant, Sheppard, Reddy, Weckesser, Abbasi, Gohlke, and Oliphant]{harris2020array}
Charles~R. Harris, K.~Jarrod Millman, St{\'{e}}fan~J. van~der Walt, Ralf Gommers, Pauli Virtanen, David Cournapeau, Eric Wieser, Julian Taylor, Sebastian Berg, Nathaniel~J. Smith, Robert Kern, Matti Picus, Stephan Hoyer, Marten~H. van Kerkwijk, Matthew Brett, Allan Haldane, Jaime~Fern{\'{a}}ndez del R{\'{i}}o, Mark Wiebe, Pearu Peterson, Pierre G{\'{e}}rard-Marchant, Kevin Sheppard, Tyler Reddy, Warren Weckesser, Hameer Abbasi, Christoph Gohlke, and Travis~E. Oliphant.
\newblock Array programming with {NumPy}.
\newblock \emph{Nature}, 585\penalty0 (7825):\penalty0 357--362, September 2020.

\bibitem[Haydari \& Y{\i}lmaz(2020)Haydari and Y{\i}lmaz]{haydari2020deep}
Ammar Haydari and Yasin Y{\i}lmaz.
\newblock Deep reinforcement learning for intelligent transportation systems: A survey.
\newblock \emph{IEEE Transactions on Intelligent Transportation Systems}, 23\penalty0 (1):\penalty0 11--32, 2020.

\bibitem[Hessel et~al.(2018)Hessel, Modayil, Van~Hasselt, Schaul, Ostrovski, Dabney, Horgan, Piot, Azar, and Silver]{hessel2018rainbow}
Matteo Hessel, Joseph Modayil, Hado Van~Hasselt, Tom Schaul, Georg Ostrovski, Will Dabney, Dan Horgan, Bilal Piot, Mohammad Azar, and David Silver.
\newblock Rainbow: Combining improvements in deep reinforcement learning.
\newblock In \emph{Thirty-second AAAI conference on artificial intelligence}, 2018.

\bibitem[Huang et~al.(2022)Huang, Dossa, Raffin, Kanervisto, and Wang]{shengyi2022the37implementation}
Shengyi Huang, Rousslan Fernand~Julien Dossa, Antonin Raffin, Anssi Kanervisto, and Weixun Wang.
\newblock The 37 implementation details of proximal policy optimization.
\newblock In \emph{ICLR Blog Track}, 2022.

\bibitem[Ibarz et~al.(2021)Ibarz, Tan, Finn, Kalakrishnan, Pastor, and Levine]{ibarz2021train}
Julian Ibarz, Jie Tan, Chelsea Finn, Mrinal Kalakrishnan, Peter Pastor, and Sergey Levine.
\newblock How to train your robot with deep reinforcement learning: lessons we have learned.
\newblock \emph{The International Journal of Robotics Research}, 40\penalty0 (4-5):\penalty0 698--721, 2021.

\bibitem[Jang et~al.(2017)Jang, Gu, and Poole]{jang2017categorical}
Eric Jang, Shixiang Gu, and Ben Poole.
\newblock Categorical reparameterization with gumbel-softmax.
\newblock In \emph{International Conference on Learning Representations}, 2017.

\bibitem[Jin et~al.(2020)Jin, Yang, Wang, and Jordan]{jin2020provably}
Chi Jin, Zhuoran Yang, Zhaoran Wang, and Michael~I Jordan.
\newblock Provably efficient reinforcement learning with linear function approximation.
\newblock In \emph{Conference on learning theory}, pp.\  2137--2143. PMLR, 2020.

\bibitem[Karl et~al.(2017)Karl, Soelch, Bayer, and van~der Smagt]{karl2017deep}
Maximilian Karl, Maximilian Soelch, Justin Bayer, and Patrick van~der Smagt.
\newblock Deep variational bayes filters: Unsupervised learning of state space models from raw data.
\newblock \emph{stat}, 1050:\penalty0 3, 2017.

\bibitem[Kaufmann et~al.(2023)Kaufmann, Bauersfeld, Loquercio, M{\"u}ller, Koltun, and Scaramuzza]{kaufmann2023champion}
Elia Kaufmann, Leonard Bauersfeld, Antonio Loquercio, Matthias M{\"u}ller, Vladlen Koltun, and Davide Scaramuzza.
\newblock Champion-level drone racing using deep reinforcement learning.
\newblock \emph{Nature}, 620\penalty0 (7976):\penalty0 982--987, 2023.

\bibitem[Kim et~al.(2022)Kim, Ha, and Kim]{kim2022self}
Kyungsoo Kim, Jeongsoo Ha, and Yusung Kim.
\newblock Self-predictive dynamics for generalization of vision-based reinforcement learning.
\newblock In \emph{IJCAI}, pp.\  3150--3156, 2022.

\bibitem[Kuznetsov et~al.(2020)Kuznetsov, Shvechikov, Grishin, and Vetrov]{kuznetsov2020controlling}
Arsenii Kuznetsov, Pavel Shvechikov, Alexander Grishin, and Dmitry Vetrov.
\newblock Controlling overestimation bias with truncated mixture of continuous distributional quantile critics.
\newblock In \emph{International Conference on Machine Learning}, pp.\  5556--5566. PMLR, 2020.

\bibitem[Lee et~al.(2020)Lee, Nagabandi, Abbeel, and Levine]{lee2020stochastic}
Alex~X Lee, Anusha Nagabandi, Pieter Abbeel, and Sergey Levine.
\newblock Stochastic latent actor-critic: Deep reinforcement learning with a latent variable model.
\newblock \emph{Advances in Neural Information Processing Systems}, 33:\penalty0 741--752, 2020.

\bibitem[Levine et~al.(2017)Levine, Zahavy, Mankowitz, Tamar, and Mannor]{levine2017shallow}
Nir Levine, Tom Zahavy, Daniel~J Mankowitz, Aviv Tamar, and Shie Mannor.
\newblock Shallow updates for deep reinforcement learning.
\newblock \emph{Advances in Neural Information Processing Systems}, 30, 2017.

\bibitem[Lillicrap et~al.(2015)Lillicrap, Hunt, Pritzel, Heess, Erez, Tassa, Silver, and Wierstra]{DDPG}
Timothy~P Lillicrap, Jonathan~J Hunt, Alexander Pritzel, Nicolas Heess, Tom Erez, Yuval Tassa, David Silver, and Daan Wierstra.
\newblock Continuous control with deep reinforcement learning.
\newblock \emph{arXiv preprint arXiv:1509.02971}, 2015.

\bibitem[Littman \& Sutton(2001)Littman and Sutton]{littman2001predictive}
Michael Littman and Richard~S Sutton.
\newblock Predictive representations of state.
\newblock \emph{Advances in neural information processing systems}, 14, 2001.

\bibitem[Loshchilov \& Hutter(2019)Loshchilov and Hutter]{loshchilov2018decoupled}
Ilya Loshchilov and Frank Hutter.
\newblock Decoupled weight decay regularization.
\newblock In \emph{International Conference on Learning Representations}, 2019.

\bibitem[Lowe et~al.(2017)Lowe, Wu, Tamar, Harb, Pieter~Abbeel, and Mordatch]{lowe2017multi}
Ryan Lowe, Yi~I Wu, Aviv Tamar, Jean Harb, OpenAI Pieter~Abbeel, and Igor Mordatch.
\newblock Multi-agent actor-critic for mixed cooperative-competitive environments.
\newblock \emph{Advances in neural information processing systems}, 30, 2017.

\bibitem[Luong et~al.(2019)Luong, Hoang, Gong, Niyato, Wang, Liang, and Kim]{luong2019applications}
Nguyen~Cong Luong, Dinh~Thai Hoang, Shimin Gong, Dusit Niyato, Ping Wang, Ying-Chang Liang, and Dong~In Kim.
\newblock Applications of deep reinforcement learning in communications and networking: A survey.
\newblock \emph{IEEE communications surveys \& tutorials}, 21\penalty0 (4):\penalty0 3133--3174, 2019.

\bibitem[Machado et~al.(2018)Machado, Bellemare, Talvitie, Veness, Hausknecht, and Bowling]{machado2018revisiting}
Marlos~C Machado, Marc~G Bellemare, Erik Talvitie, Joel Veness, Matthew Hausknecht, and Michael Bowling.
\newblock Revisiting the arcade learning environment: Evaluation protocols and open problems for general agents.
\newblock \emph{Journal of Artificial Intelligence Research}, 61:\penalty0 523--562, 2018.

\bibitem[McInroe et~al.(2021)McInroe, Sch{\"a}fer, and Albrecht]{mcinroe2021learning}
Trevor McInroe, Lukas Sch{\"a}fer, and Stefano~V Albrecht.
\newblock Learning temporally-consistent representations for data-efficient reinforcement learning.
\newblock \emph{arXiv preprint arXiv:2110.04935}, 2021.

\bibitem[Mnih et~al.(2015)Mnih, Kavukcuoglu, Silver, Rusu, Veness, Bellemare, Graves, Riedmiller, Fidjeland, Ostrovski, et~al.]{DQN}
Volodymyr Mnih, Koray Kavukcuoglu, David Silver, Andrei~A Rusu, Joel Veness, Marc~G Bellemare, Alex Graves, Martin Riedmiller, Andreas~K Fidjeland, Georg Ostrovski, et~al.
\newblock Human-level control through deep reinforcement learning.
\newblock \emph{Nature}, 518\penalty0 (7540):\penalty0 529--533, 2015.

\bibitem[Munk et~al.(2016)Munk, Kober, and Babu{\v{s}}ka]{munk2016learning}
Jelle Munk, Jens Kober, and Robert Babu{\v{s}}ka.
\newblock Learning state representation for deep actor-critic control.
\newblock In \emph{2016 IEEE 55th Conference on Decision and Control (CDC)}, pp.\  4667--4673. IEEE, 2016.

\bibitem[Ni et~al.(2024)Ni, Eysenbach, Seyedsalehi, Ma, Gehring, Mahajan, and Bacon]{ni2024bridging}
Tianwei Ni, Benjamin Eysenbach, Erfan Seyedsalehi, Michel Ma, Clement Gehring, Aditya Mahajan, and Pierre-Luc Bacon.
\newblock Bridging state and history representations: Understanding self-predictive rl.
\newblock In \emph{The Twelfth International Conference on Learning Representations}, 2024.

\bibitem[Oh et~al.(2015)Oh, Guo, Lee, Lewis, and Singh]{oh2015action}
Junhyuk Oh, Xiaoxiao Guo, Honglak Lee, Richard~L Lewis, and Satinder Singh.
\newblock Action-conditional video prediction using deep networks in atari games.
\newblock \emph{Advances in neural information processing systems}, 28, 2015.

\bibitem[Parr et~al.(2008)Parr, Li, Taylor, Painter-Wakefield, and Littman]{parr2008analysis}
Ronald Parr, Lihong Li, Gavin Taylor, Christopher Painter-Wakefield, and Michael~L Littman.
\newblock An analysis of linear models, linear value-function approximation, and feature selection for reinforcement learning.
\newblock In \emph{Proceedings of the 25th international conference on Machine learning}, pp.\  752--759, 2008.

\bibitem[Paszke et~al.(2019)Paszke, Gross, Massa, Lerer, Bradbury, Chanan, Killeen, Lin, Gimelshein, Antiga, et~al.]{paszke2019pytorch}
Adam Paszke, Sam Gross, Francisco Massa, Adam Lerer, James Bradbury, Gregory Chanan, Trevor Killeen, Zeming Lin, Natalia Gimelshein, Luca Antiga, et~al.
\newblock Pytorch: An imperative style, high-performance deep learning library.
\newblock In \emph{Advances in Neural Information Processing Systems}, pp.\  8024--8035, 2019.

\bibitem[Raffin et~al.(2021)Raffin, Hill, Gleave, Kanervisto, Ernestus, and Dormann]{stable-baselines3}
Antonin Raffin, Ashley Hill, Adam Gleave, Anssi Kanervisto, Maximilian Ernestus, and Noah Dormann.
\newblock Stable-baselines3: Reliable reinforcement learning implementations.
\newblock \emph{Journal of Machine Learning Research}, 22\penalty0 (268):\penalty0 1--8, 2021.
\newblock URL \url{http://jmlr.org/papers/v22/20-1364.html}.

\bibitem[Ravindran(2004)]{ravindran2004algebraic}
Balaraman Ravindran.
\newblock \emph{An algebraic approach to abstraction in reinforcement learning}.
\newblock University of Massachusetts Amherst, 2004.

\bibitem[Ravindran \& Barto(2002)Ravindran and Barto]{ravindran2002model}
Balaraman Ravindran and Andrew~G Barto.
\newblock Model minimization in hierarchical reinforcement learning.
\newblock In \emph{Abstraction, Reformulation, and Approximation: 5th International Symposium, SARA 2002 Kananaskis, Alberta, Canada August 2--4, 2002 Proceedings 5}, pp.\  196--211. Springer, 2002.

\bibitem[Rechenberg(1978)]{rechenberg1978evolutionsstrategien}
Ingo Rechenberg.
\newblock Evolutionsstrategien.
\newblock In \emph{Simulationsmethoden in der Medizin und Biologie: Workshop, Hannover, 29. Sept.--1. Okt. 1977}, pp.\  83--114. Springer, 1978.

\bibitem[Reed et~al.(2022)Reed, Zolna, Parisotto, Colmenarejo, Novikov, Barth-maron, Gim{\'e}nez, Sulsky, Kay, Springenberg, Eccles, Bruce, Razavi, Edwards, Heess, Chen, Hadsell, Vinyals, Bordbar, and de~Freitas]{reed2022generalist}
Scott Reed, Konrad Zolna, Emilio Parisotto, Sergio~G{\'o}mez Colmenarejo, Alexander Novikov, Gabriel Barth-maron, Mai Gim{\'e}nez, Yury Sulsky, Jackie Kay, Jost~Tobias Springenberg, Tom Eccles, Jake Bruce, Ali Razavi, Ashley Edwards, Nicolas Heess, Yutian Chen, Raia Hadsell, Oriol Vinyals, Mahyar Bordbar, and Nando de~Freitas.
\newblock A generalist agent.
\newblock \emph{Transactions on Machine Learning Research}, 2022.
\newblock ISSN 2835-8856.

\bibitem[Ren et~al.(2022)Ren, Zhang, Szepesv{\'a}ri, and Dai]{ren2022free}
Tongzheng Ren, Tianjun Zhang, Csaba Szepesv{\'a}ri, and Bo~Dai.
\newblock A free lunch from the noise: Provable and practical exploration for representation learning.
\newblock In \emph{Uncertainty in Artificial Intelligence}, pp.\  1686--1696. PMLR, 2022.

\bibitem[Ren et~al.(2023)Ren, Xiao, Zhang, Li, Wang, Schuurmans, Dai, et~al.]{ren2023latent}
Tongzheng Ren, Chenjun Xiao, Tianjun Zhang, Na~Li, Zhaoran Wang, Dale Schuurmans, Bo~Dai, et~al.
\newblock Latent variable representation for reinforcement learning.
\newblock In \emph{The Eleventh International Conference on Learning Representations}, 2023.

\bibitem[Rezaei-Shoshtari et~al.(2022)Rezaei-Shoshtari, Zhao, Panangaden, Meger, and Precup]{rezaei2022continuous}
Sahand Rezaei-Shoshtari, Rosie Zhao, Prakash Panangaden, David Meger, and Doina Precup.
\newblock Continuous mdp homomorphisms and homomorphic policy gradient.
\newblock In \emph{Advances in Neural Information Processing Systems}, 2022.

\bibitem[Rubinstein(1997)]{rubinstein1997optimization}
Reuven~Y Rubinstein.
\newblock Optimization of computer simulation models with rare events.
\newblock \emph{European Journal of Operational Research}, 99\penalty0 (1):\penalty0 89--112, 1997.

\bibitem[Rummery \& Niranjan(1994)Rummery and Niranjan]{rummery1994line}
Gavin~A Rummery and Mahesan Niranjan.
\newblock \emph{On-line Q-learning using connectionist systems}, volume~37.
\newblock University of Cambridge, Department of Engineering Cambridge, UK, 1994.

\bibitem[Salimans et~al.(2017)Salimans, Ho, Chen, Sidor, and Sutskever]{salimans2017evolution}
Tim Salimans, Jonathan Ho, Xi~Chen, Szymon Sidor, and Ilya Sutskever.
\newblock Evolution strategies as a scalable alternative to reinforcement learning.
\newblock \emph{arXiv preprint arXiv:1703.03864}, 2017.

\bibitem[Scannell et~al.(2024)Scannell, Kujanp{\"a}{\"a}, Zhao, Nakhaei, Solin, and Pajarinen]{scannell2024iqrl}
Aidan Scannell, Kalle Kujanp{\"a}{\"a}, Yi~Zhao, Mohammadreza Nakhaei, Arno Solin, and Joni Pajarinen.
\newblock iqrl--implicitly quantized representations for sample-efficient reinforcement learning.
\newblock \emph{arXiv preprint arXiv:2406.02696}, 2024.

\bibitem[Schaul et~al.(2016)Schaul, Quan, Antonoglou, and Silver]{PrioritizedExpReplay}
Tom Schaul, John Quan, Ioannis Antonoglou, and David Silver.
\newblock Prioritized experience replay.
\newblock In \emph{International Conference on Learning Representations}, Puerto Rico, 2016.

\bibitem[Schrittwieser et~al.(2020)Schrittwieser, Antonoglou, Hubert, Simonyan, Sifre, Schmitt, Guez, Lockhart, Hassabis, Graepel, et~al.]{schrittwieser2020mastering}
Julian Schrittwieser, Ioannis Antonoglou, Thomas Hubert, Karen Simonyan, Laurent Sifre, Simon Schmitt, Arthur Guez, Edward Lockhart, Demis Hassabis, Thore Graepel, et~al.
\newblock Mastering atari, go, chess and shogi by planning with a learned model.
\newblock \emph{Nature}, 588\penalty0 (7839):\penalty0 604--609, 2020.

\bibitem[Schrittwieser et~al.(2021)Schrittwieser, Hubert, Mandhane, Barekatain, Antonoglou, and Silver]{schrittwieser2021online}
Julian Schrittwieser, Thomas Hubert, Amol Mandhane, Mohammadamin Barekatain, Ioannis Antonoglou, and David Silver.
\newblock Online and offline reinforcement learning by planning with a learned model.
\newblock \emph{Advances in Neural Information Processing Systems}, 34:\penalty0 27580--27591, 2021.

\bibitem[Schulman et~al.(2015)Schulman, Levine, Abbeel, Jordan, and Moritz]{trpo}
John Schulman, Sergey Levine, Pieter Abbeel, Michael Jordan, and Philipp Moritz.
\newblock Trust region policy optimization.
\newblock In \emph{International Conference on Machine Learning}, pp.\  1889--1897, 2015.

\bibitem[Schulman et~al.(2017)Schulman, Wolski, Dhariwal, Radford, and Klimov]{ppo}
John Schulman, Filip Wolski, Prafulla Dhariwal, Alec Radford, and Oleg Klimov.
\newblock Proximal policy optimization algorithms.
\newblock \emph{arXiv preprint arXiv:1707.06347}, 2017.

\bibitem[Schwarzer et~al.(2020)Schwarzer, Anand, Goel, Hjelm, Courville, and Bachman]{schwarzer2020data}
Max Schwarzer, Ankesh Anand, Rishab Goel, R~Devon Hjelm, Aaron Courville, and Philip Bachman.
\newblock Data-efficient reinforcement learning with self-predictive representations.
\newblock In \emph{International Conference on Learning Representations}, 2020.

\bibitem[Schwarzer et~al.(2023)Schwarzer, Ceron, Courville, Bellemare, Agarwal, and Castro]{schwarzer2023bigger}
Max Schwarzer, Johan Samir~Obando Ceron, Aaron Courville, Marc~G Bellemare, Rishabh Agarwal, and Pablo~Samuel Castro.
\newblock Bigger, better, faster: Human-level atari with human-level efficiency.
\newblock In \emph{International Conference on Machine Learning}, pp.\  30365--30380. PMLR, 2023.

\bibitem[Seo et~al.(2022)Seo, Lee, James, and Abbeel]{seo2022reinforcement}
Younggyo Seo, Kimin Lee, Stephen~L James, and Pieter Abbeel.
\newblock Reinforcement learning with action-free pre-training from videos.
\newblock In \emph{International Conference on Machine Learning}, pp.\  19561--19579. PMLR, 2022.

\bibitem[Shribak et~al.(2024)Shribak, Gao, Li, Xiao, and Dai]{shribak2024diffusion}
Dmitry Shribak, Chen-Xiao Gao, Yitong Li, Chenjun Xiao, and Bo~Dai.
\newblock Diffusion spectral representation for reinforcement learning.
\newblock In \emph{The Thirty-eighth Annual Conference on Neural Information Processing Systems}, 2024.

\bibitem[Silver et~al.(2014)Silver, Lever, Heess, Degris, Wierstra, and Riedmiller]{DPG}
David Silver, Guy Lever, Nicolas Heess, Thomas Degris, Daan Wierstra, and Martin Riedmiller.
\newblock Deterministic policy gradient algorithms.
\newblock In \emph{International Conference on Machine Learning}, pp.\  387--395, 2014.

\bibitem[Silver et~al.(2016)Silver, Huang, Maddison, Guez, Sifre, Van Den~Driessche, Schrittwieser, Antonoglou, Panneershelvam, Lanctot, et~al.]{alphago}
David Silver, Aja Huang, Chris~J Maddison, Arthur Guez, Laurent Sifre, George Van Den~Driessche, Julian Schrittwieser, Ioannis Antonoglou, Veda Panneershelvam, Marc Lanctot, et~al.
\newblock Mastering the game of go with deep neural networks and tree search.
\newblock \emph{Nature}, 529\penalty0 (7587):\penalty0 484--489, 2016.

\bibitem[Sokar et~al.(2023)Sokar, Agarwal, Castro, and Evci]{sokar2023dormant}
Ghada Sokar, Rishabh Agarwal, Pablo~Samuel Castro, and Utku Evci.
\newblock The dormant neuron phenomenon in deep reinforcement learning.
\newblock In \emph{International Conference on Machine Learning}, pp.\  32145--32168. PMLR, 2023.

\bibitem[Song et~al.(2016)Song, Parr, Liao, and Carin]{song2016linear}
Zhao Song, Ronald~E Parr, Xuejun Liao, and Lawrence Carin.
\newblock Linear feature encoding for reinforcement learning.
\newblock \emph{Advances in neural information processing systems}, 29, 2016.

\bibitem[Sutton(1988)]{sutton1988tdlearning}
Richard~S Sutton.
\newblock Learning to predict by the methods of temporal differences.
\newblock \emph{Machine learning}, 3\penalty0 (1):\penalty0 9--44, 1988.

\bibitem[Sutton \& Barto(1998)Sutton and Barto]{suttonbarto}
Richard~S Sutton and Andrew~G Barto.
\newblock \emph{Reinforcement Learning: An Introduction}, volume~1.
\newblock MIT press Cambridge, 1998.

\bibitem[Sutton et~al.(1999)Sutton, McAllester, Singh, and Mansour]{sutton1999policy}
Richard~S Sutton, David McAllester, Satinder Singh, and Yishay Mansour.
\newblock Policy gradient methods for reinforcement learning with function approximation.
\newblock \emph{Advances in neural information processing systems}, 12, 1999.

\bibitem[Tang et~al.(2023)Tang, Guo, Richemond, Pires, Chandak, Munos, Rowland, Azar, Le~Lan, Lyle, et~al.]{tang2023understanding}
Yunhao Tang, Zhaohan~Daniel Guo, Pierre~Harvey Richemond, Bernardo~Avila Pires, Yash Chandak, R{\'e}mi Munos, Mark Rowland, Mohammad~Gheshlaghi Azar, Charline Le~Lan, Clare Lyle, et~al.
\newblock Understanding self-predictive learning for reinforcement learning.
\newblock In \emph{International Conference on Machine Learning}, pp.\  33632--33656. PMLR, 2023.

\bibitem[Tassa et~al.(2018)Tassa, Doron, Muldal, Erez, Li, Casas, Budden, Abdolmaleki, Merel, Lefrancq, et~al.]{tassa2018deepmind}
Yuval Tassa, Yotam Doron, Alistair Muldal, Tom Erez, Yazhe Li, Diego de~Las Casas, David Budden, Abbas Abdolmaleki, Josh Merel, Andrew Lefrancq, et~al.
\newblock Deepmind control suite.
\newblock \emph{arXiv preprint arXiv:1801.00690}, 2018.

\bibitem[Todorov et~al.(2012)Todorov, Erez, and Tassa]{mujoco}
Emanuel Todorov, Tom Erez, and Yuval Tassa.
\newblock Mujoco: A physics engine for model-based control.
\newblock In \emph{IEEE/RSJ International Conference on Intelligent Robots and Systems (IROS)}, pp.\  5026--5033. IEEE, 2012.

\bibitem[Touvron et~al.(2023)Touvron, Martin, Stone, Albert, Almahairi, Babaei, Bashlykov, Batra, Bhargava, Bhosale, et~al.]{touvron2023llama}
Hugo Touvron, Louis Martin, Kevin Stone, Peter Albert, Amjad Almahairi, Yasmine Babaei, Nikolay Bashlykov, Soumya Batra, Prajjwal Bhargava, Shruti Bhosale, et~al.
\newblock Llama 2: Open foundation and fine-tuned chat models.
\newblock \emph{arXiv preprint arXiv:2307.09288}, 2023.

\bibitem[Towers et~al.(2024)Towers, Kwiatkowski, Terry, Balis, De~Cola, Deleu, Goul{\~a}o, Kallinteris, Krimmel, KG, et~al.]{towers2024gymnasium}
Mark Towers, Ariel Kwiatkowski, Jordan Terry, John~U Balis, Gianluca De~Cola, Tristan Deleu, Manuel Goul{\~a}o, Andreas Kallinteris, Markus Krimmel, Arjun KG, et~al.
\newblock Gymnasium: A standard interface for reinforcement learning environments.
\newblock \emph{arXiv preprint arXiv:2407.17032}, 2024.

\bibitem[van~der Pol et~al.(2020{\natexlab{a}})van~der Pol, Kipf, Oliehoek, and Welling]{van2020plannable}
Elise van~der Pol, Thomas Kipf, Frans~A Oliehoek, and Max Welling.
\newblock Plannable approximations to mdp homomorphisms: Equivariance under actions.
\newblock In \emph{Proceedings of the 19th International Conference on Autonomous Agents and MultiAgent Systems}, pp.\  1431--1439, 2020{\natexlab{a}}.

\bibitem[van~der Pol et~al.(2020{\natexlab{b}})van~der Pol, Worrall, van Hoof, Oliehoek, and Welling]{van2020mdp}
Elise van~der Pol, Daniel Worrall, Herke van Hoof, Frans Oliehoek, and Max Welling.
\newblock Mdp homomorphic networks: Group symmetries in reinforcement learning.
\newblock \emph{Advances in Neural Information Processing Systems}, 33:\penalty0 4199--4210, 2020{\natexlab{b}}.

\bibitem[Van~Hoof et~al.(2016)Van~Hoof, Chen, Karl, van~der Smagt, and Peters]{van2016stable}
Herke Van~Hoof, Nutan Chen, Maximilian Karl, Patrick van~der Smagt, and Jan Peters.
\newblock Stable reinforcement learning with autoencoders for tactile and visual data.
\newblock In \emph{2016 IEEE/RSJ international conference on intelligent robots and systems (IROS)}, pp.\  3928--3934. IEEE, 2016.

\bibitem[Van~Rossum \& Drake~Jr(1995)Van~Rossum and Drake~Jr]{van1995python}
Guido Van~Rossum and Fred~L Drake~Jr.
\newblock \emph{Python tutorial}.
\newblock Centrum voor Wiskunde en Informatica Amsterdam, The Netherlands, 1995.

\bibitem[Wang et~al.(2024)Wang, Liu, Ye, You, and Gao]{wang2024efficientzero}
Shengjie Wang, Shaohuai Liu, Weirui Ye, Jiacheng You, and Yang Gao.
\newblock Efficientzero v2: Mastering discrete and continuous control with limited data.
\newblock \emph{arXiv preprint arXiv:2403.00564}, 2024.

\bibitem[Wang et~al.(2016)Wang, Schaul, Hessel, Hasselt, Lanctot, and Freitas]{wang2015dueling}
Ziyu Wang, Tom Schaul, Matteo Hessel, Hado Hasselt, Marc Lanctot, and Nando Freitas.
\newblock Dueling network architectures for deep reinforcement learning.
\newblock In \emph{International Conference on Machine Learning}, pp.\  1995--2003, 2016.

\bibitem[Watkins(1989)]{watkins1989qlearning}
Christopher John Cornish~Hellaby Watkins.
\newblock \emph{Learning from delayed rewards}.
\newblock PhD thesis, King's College, Cambridge, 1989.

\bibitem[Watter et~al.(2015)Watter, Springenberg, Boedecker, and Riedmiller]{watter2015embed}
Manuel Watter, Jost Springenberg, Joschka Boedecker, and Martin Riedmiller.
\newblock Embed to control: A locally linear latent dynamics model for control from raw images.
\newblock \emph{Advances in neural information processing systems}, 28, 2015.

\bibitem[Williams(1992)]{williams1992reinforce}
Ronald~J Williams.
\newblock Simple statistical gradient-following algorithms for connectionist reinforcement learning.
\newblock \emph{Machine learning}, 8\penalty0 (3-4):\penalty0 229--256, 1992.

\bibitem[Yarats et~al.(2022)Yarats, Fergus, Lazaric, and Pinto]{yarats2022mastering}
Denis Yarats, Rob Fergus, Alessandro Lazaric, and Lerrel Pinto.
\newblock Mastering visual continuous control: Improved data-augmented reinforcement learning.
\newblock In \emph{International Conference on Learning Representations}, 2022.

\bibitem[Ye et~al.(2021)Ye, Liu, Kurutach, Abbeel, and Gao]{ye2021mastering}
Weirui Ye, Shaohuai Liu, Thanard Kurutach, Pieter Abbeel, and Yang Gao.
\newblock Mastering atari games with limited data.
\newblock \emph{Advances in neural information processing systems}, 34:\penalty0 25476--25488, 2021.

\bibitem[Yu et~al.(2021)Yu, Lan, Zeng, Feng, Zhang, and Chen]{yu2021playvirtual}
Tao Yu, Cuiling Lan, Wenjun Zeng, Mingxiao Feng, Zhizheng Zhang, and Zhibo Chen.
\newblock Playvirtual: Augmenting cycle-consistent virtual trajectories for reinforcement learning.
\newblock \emph{Advances in Neural Information Processing Systems}, 34:\penalty0 5276--5289, 2021.

\bibitem[Yu et~al.(2022)Yu, Zhang, Lan, Lu, and Chen]{yu2022mask}
Tao Yu, Zhizheng Zhang, Cuiling Lan, Yan Lu, and Zhibo Chen.
\newblock Mask-based latent reconstruction for reinforcement learning.
\newblock \emph{Advances in Neural Information Processing Systems}, 35:\penalty0 25117--25131, 2022.

\bibitem[Zang et~al.(2022)Zang, Li, and Wang]{zang2022simsr}
Hongyu Zang, Xin Li, and Mingzhong Wang.
\newblock Simsr: Simple distance-based state representations for deep reinforcement learning.
\newblock In \emph{Proceedings of the AAAI Conference on Artificial Intelligence}, volume~36, pp.\  8997--9005, 2022.

\bibitem[Zhang et~al.(2018)Zhang, Satija, and Pineau]{zhang2018decoupling}
Amy Zhang, Harsh Satija, and Joelle Pineau.
\newblock Decoupling dynamics and reward for transfer learning.
\newblock \emph{arXiv preprint arXiv:1804.10689}, 2018.

\bibitem[Zhang et~al.(2020)Zhang, McAllister, Calandra, Gal, and Levine]{zhang2020learning}
Amy Zhang, Rowan~Thomas McAllister, Roberto Calandra, Yarin Gal, and Sergey Levine.
\newblock Learning invariant representations for reinforcement learning without reconstruction.
\newblock In \emph{International Conference on Learning Representations}, 2020.

\bibitem[Zhang et~al.(2022)Zhang, Ren, Yang, Gonzalez, Schuurmans, and Dai]{zhang2022making}
Tianjun Zhang, Tongzheng Ren, Mengjiao Yang, Joseph Gonzalez, Dale Schuurmans, and Bo~Dai.
\newblock Making linear mdps practical via contrastive representation learning.
\newblock In \emph{International Conference on Machine Learning}, pp.\  26447--26466. PMLR, 2022.

\bibitem[Zhao et~al.(2023)Zhao, Zhao, Boney, Kannala, and Pajarinen]{zhao2023simplified}
Yi~Zhao, Wenshuai Zhao, Rinu Boney, Juho Kannala, and Joni Pajarinen.
\newblock Simplified temporal consistency reinforcement learning.
\newblock In \emph{International Conference on Machine Learning}, pp.\  42227--42246. PMLR, 2023.

\bibitem[Zheng et~al.(2024)Zheng, Wang, Sun, Ma, Zhao, Xu, Daum{\'e}~III, and Huang]{zheng2024texttt}
Ruijie Zheng, Xiyao Wang, Yanchao Sun, Shuang Ma, Jieyu Zhao, Huazhe Xu, Hal Daum{\'e}~III, and Furong Huang.
\newblock Taco: Temporal latent action-driven contrastive loss for visual reinforcement learning.
\newblock \emph{Advances in Neural Information Processing Systems}, 36, 2024.

\bibitem[Zintgraf et~al.(2021)Zintgraf, Schulze, Lu, Feng, Igl, Shiarlis, Gal, Hofmann, and Whiteson]{zintgraf2021varibad}
Luisa Zintgraf, Sebastian Schulze, Cong Lu, Leo Feng, Maximilian Igl, Kyriacos Shiarlis, Yarin Gal, Katja Hofmann, and Shimon Whiteson.
\newblock Varibad: Variational bayes-adaptive deep rl via meta-learning.
\newblock \emph{Journal of Machine Learning Research}, 22\penalty0 (289):\penalty0 1--39, 2021.

\end{thebibliography}
\bibliographystyle{iclr2025_conference}

\clearpage

\appendix

\renewcommand \thepart{}
\renewcommand \partname{}

\part{Appendix}\label{appendix} %

\startcontents[appendix]

\vspace{-20pt}
\noindent\rule{\textwidth}{0.5pt}
\vspace{-20pt}

{\small

\printcontents[appendix]{l}{1}

}

\noindent\rule{\textwidth}{0.5pt}

\section{Proofs} \label{appendix:proofs}

\setcounter{theorem}{0}

\begin{theorem} \label{appendix:thm:equal}
    The fixed point of the model-free approach~(\ref{eqn:semi_gradient_TD}) and the solution of the model-based approach~(\ref{eqn:linear_model_based}) are the same.
\end{theorem}

\begin{proof}
Let $Z$ be a matrix containing state-action embeddings~$\mathbf{z}_{sa}$ for each state-action pair $(s,a) \in S \times A$. Let $Z'$ be the corresponding matrix of next state-action embeddings~$\mathbf{z}_{s'a'}$. Let $R$ be the vector of the corresponding rewards $r(s,a)$. 

The linear semi-gradient TD update:
\begin{align}
    \mathbf{w}_{t+1} :=&~ \mathbf{w}_{t} - \al Z^\top (Z \mathbf{w}_{t} - (R + \y Z' \mathbf{w}_{t})) \\
    =&~ \mathbf{w}_{t} - \al Z^\top Z \mathbf{w}_{t} + \al Z^\top R + \al \y Z^\top Z' \mathbf{w}_{t} \\
    =&~ (I - \al (Z^\top Z - \y Z^\top Z')) \mathbf{w}_{t} + \al Z^\top R \\ 
    =&~ (I - \al A) \mathbf{w}_{t} + \al B,
\end{align}
where $A := Z^\top Z - \y Z^\top Z'$ and $B := Z^\top R$.

The fixed point of the system:
\begin{align}
    \mathbf{w}_\text{mf} &= (I - \al A) \mathbf{w}_\text{mf} + \al B \\
    \mathbf{w}_\text{mf} - (I - \al A) \mathbf{w}_\text{mf} &= \al B \\
    \al A \mathbf{w}_\text{mf} &= \al B \\
    \mathbf{w}_\text{mf} &= A^{-1} B.
\end{align}

The least squares solution to $W_p$ and $\mathbf{w}_r$
\begin{align}
    W_p &:= \lp Z^\top Z \rp^{-1} Z^\top Z' \\
    \mathbf{w}_r &:= \lp Z^\top Z \rp^{-1} Z^\top R
\end{align}

By rolling out $W_p$ and $\mathbf{w}_r$, we arrive at a model-based solution:
\begin{equation}
    Q := Z \mathbf{w}_\text{mb} = Z \sum_{t=0}^{\infty} \y^t W_p^t \mathbf{w}_r.  
\end{equation}

Simplify $\mathbf{w}_\text{mb}$: 
\begin{align}
    \mathbf{w}_\text{mb} :=&~ \sum_{t=0}^{\infty} \y^t W_p^t \mathbf{w}_r \\
    \mathbf{w}_\text{mb} =&~ \lp I - \y W_p \rp^{-1} \mathbf{w}_r \\
    \mathbf{w}_\text{mb} =&~ \lp I - \y \lp Z^\top Z \rp^{-1} Z^\top Z' \rp^{-1} \lp Z^\top Z \rp^{-1} Z^\top R \\
    Z^\top Z \lp I - \y \lp Z^\top Z \rp^{-1} Z^\top Z' \rp \mathbf{w}_\text{mb} =&~ Z^\top R \\
    \lp Z^\top Z - \y Z^\top Z' \rp \mathbf{w}_\text{mb} =&~ Z^\top R \\
    \mathbf{w}_\text{mb} =&~ A^{-1} B \\ 
    \mathbf{w}_\text{mb} =&~ \mathbf{w}_\text{mf}. 
\end{align}
\end{proof}

\begin{theorem} The value error of the solution described by \ref{thm:equal} is bounded by the accuracy of the estimated dynamics and reward:
\begin{equation}
    |\textnormal{VE}(s,a)| \leq \frac{1}{1 - \y} \max_{(s,a) \in S \times A} \lp | \mathbf{z}_{sa}^\top \mathbf{w}_r - \E_{r|s,a}[r] | + \max_i | \mathbf{w}_i | \sum | \mathbf{z}_{sa}^\top W_p - \E_{s',a'|s,a} [ \mathbf{z}_{s'a'} ] | \rp.
\end{equation}  
\end{theorem}

\begin{proof} 
Let $\mathbf{w}$ be the solution described in \ref{appendix:thm:equal}, i.e.\ $\mathbf{w} = \mathbf{w}_\text{mb} = \mathbf{w}_\text{mf}$. Let $p^\pi(s,a)$ be the discounted state-action visitation distribution according to the policy~$\pi$ starting from the state-action pair~$(s,a)$. 

Firstly from \ref{appendix:thm:equal}, we can show that 
\begin{align}
    &~\mathbf{w} = (I - \y W_p)^{-1} \mathbf{w}_r \\
    \Rightarrow&~ (I - \y W_p) \mathbf{w} = \mathbf{w}_r \\
    \Rightarrow&~ \mathbf{w} - \y W_p \mathbf{w} = \mathbf{w}_r. 
\end{align}

Simplify $\text{VE}(s,a)$: 
\begin{align}
    \text{VE}(s,a) :=&~Q(s,a) - Q^\pi(s,a) \\
    =&~Q(s,a) - Q^\pi(s,a) \\
    =&~Q(s,a) - \E_{r, s',a'} \lb r + \y Q^\pi(s',a') \rb \\
    =&~Q(s,a) - \E_{r, s',a'} \lb r + \y \lp Q(s',a') - \text{VE}(s',a') \rp \rb \\
    =&~Q(s,a) - \E_{r, s',a'} \lb r + \y Q(s',a') \rb + \y \E_{s',a'} \lb \text{VE}(s',a') \rb \\
    =&~Q(s,a) - \E_{r, s',a'} \lb r - \mathbf{z}_{sa}^\top \mathbf{w}_r + \mathbf{z}_{sa}^\top \mathbf{w}_r + \y Q(s',a') \rb + \y \E_{s',a'} \lb \text{VE}(s',a') \rb \\
    \begin{split}
    =&~Q(s,a) - \E_{r, s',a'} \lb r - \mathbf{z}_{sa}^\top \mathbf{w}_r + \mathbf{z}_{sa}^\top \mathbf{w}_r + \y \lp \mathbf{z}_{s'a'}^\top \mathbf{w} - \mathbf{z}_{sa}^\top W_p \mathbf{w} + \mathbf{z}_{sa}^\top W_p \mathbf{w} \rp \rb \\
    &+ \y \E_{s',a'} \lb \text{VE}(s',a') \rb     
    \end{split}\\
    \begin{split}
    =&~\mathbf{z}_{sa}^\top \mathbf{w} - \E_{r, s',a'} \lb r - \mathbf{z}_{sa}^\top \mathbf{w}_r + \mathbf{z}_{sa}^\top \mathbf{w}_r + \y \lp \mathbf{z}_{s'a'}^\top \mathbf{w} - \mathbf{z}_{sa}^\top W_p \mathbf{w} + \mathbf{z}_{sa}^\top W_p \mathbf{w} \rp \rb \\
    &+ \y \E_{s',a'} \lb \text{VE}(s',a') \rb     
    \end{split}\\
    \begin{split}
    =&~\mathbf{z}_{sa}^\top \mathbf{w} 
    - \E_{r} \lb r - \mathbf{z}_{sa}^\top \mathbf{w}_r + \mathbf{z}_{sa}^\top \mathbf{w}_r \rb 
    - \y \E_{s',a'} \lb 
    \mathbf{z}_{s'a'}^\top \mathbf{w} - \mathbf{z}_{sa}^\top W_p \mathbf{w} + \mathbf{z}_{sa}^\top W_p \mathbf{w} \rb \\
    &+ \y \E_{s',a'} \lb \text{VE}(s',a') \rb     
    \end{split}\\
    \begin{split}
    =&~\mathbf{z}_{sa}^\top \mathbf{w} - \mathbf{z}_{sa}^\top \mathbf{w}_r - \y \mathbf{z}_{sa}^\top W_p \mathbf{w}
    - \E_{r} \lb r - \mathbf{z}_{sa}^\top \mathbf{w}_r \rb 
    - \y \E_{s',a'} \lb 
    \mathbf{z}_{s'a'}^\top \mathbf{w} - \mathbf{z}_{sa}^\top W_p \mathbf{w}  \rb \\
    &+ \y \E_{s',a'} \lb \text{VE}(s',a') \rb     
    \end{split}\\
    \begin{split}
    =&~\mathbf{z}_{sa}^\top \lp \mathbf{w} - \y W_p \mathbf{w} - \mathbf{w}_r \rp
    - \E_{r} \lb r - \mathbf{z}_{sa}^\top \mathbf{w}_r \rb 
    - \y \E_{s',a'} \lb 
    \mathbf{z}_{s'a'}^\top \mathbf{w} - \mathbf{z}_{sa}^\top W_p \mathbf{w}  \rb \\
    &+ \y \E_{s',a'} \lb \text{VE}(s',a') \rb     
    \end{split}\\
    =&~ -\E_{r} \lb r - \mathbf{z}_{sa}^\top \mathbf{w}_r \rb 
    - \y \E_{s',a'} \lb 
    \mathbf{z}_{s'a'}^\top \mathbf{w} - \mathbf{z}_{sa}^\top W_p \mathbf{w}  \rb + \y \E_{s',a'} \lb \text{VE}(s',a') \rb \\
    =&~ \lp \mathbf{z}_{sa}^\top \mathbf{w}_r - \E_{r} \lb r \rb \rp
    + \y \lp \mathbf{z}_{sa}^\top W_p - \E_{s',a'} \lb 
    \mathbf{z}_{s'a'}^\top \rb \rp \mathbf{w} + \y \E_{s',a'} \lb \text{VE}(s',a') \rb.
\end{align}

Then given the recursive relationship, akin to the Bellman equation~\citep{suttonbarto}, the value error~$\text{VE}$ recursively expands to the discounted state-action visitation distribution~$p^\pi$. For $(\hat s, \hat a) \in S \times A$:
\begin{align}
\text{VE}(\hat s, \hat a) = \frac{1}{1 - \y} \E_{(s, a) \sim p^\pi(\hat s, \hat a)} \lb \lp \mathbf{z}_{sa}^\top \mathbf{w}_r - \E_{r | s, a} \lb r \rb \rp
    + \y \lp \mathbf{z}_{s a}^\top W_p - \E_{s',a' | s, a} \lb 
    \mathbf{z}_{s'a'}^\top \rb \rp \mathbf{w} \rb.
\end{align}

Taking the absolute value:
\begin{align}
|\text{VE}(\hat s, \hat a)| &= \left | \frac{1}{1 - \y} \E_{(s, a) \sim p^\pi(\hat s, \hat a)} \lb \lp \mathbf{z}_{sa}^\top \mathbf{w}_r - \E_{r | s, a} \lb r \rb \rp
    + \y \lp \mathbf{z}_{s a}^\top W_p - \E_{s',a' | s, a} \lb 
    \mathbf{z}_{s'a'}^\top \rb \rp \mathbf{w} \rb \right | \\
|\text{VE}(\hat s, \hat a)| &\leq \frac{1}{1 - \y} \E_{(s, a) \sim p^\pi(\hat s, \hat a)} \lb \left | \mathbf{z}_{sa}^\top \mathbf{w}_r - \E_{r | s, a} \lb r \rb \right |
    + \y \left | \lp \mathbf{z}_{s a}^\top W_p - \E_{s',a' | s, a} \lb 
    \mathbf{z}_{s'a'}^\top \rb \rp \mathbf{w} \right |  \rb  \\ 
&= \frac{1}{1 - \y} \max_{(s,a) \in S \times A} \lp \left | \mathbf{z}_{sa}^\top \mathbf{w}_r - \E_{r | s, a} \lb r \rb \right |
    + \y \left | \lp \mathbf{z}_{s a}^\top W_p - \E_{s',a' | s, a} \lb 
    \mathbf{z}_{s'a'}^\top \rb \rp \mathbf{w} \right | \rp \\ 
&\leq \frac{1}{1 - \y} \max_{(s,a) \in S \times A} \lp \left | \mathbf{z}_{sa}^\top \mathbf{w}_r - \E_{r | s, a} \lb r \rb \right |
    + \max_i | \mathbf{w}_i | \sum \left | \mathbf{z}_{s a}^\top W_p - \E_{s',a' | s, a} \lb 
    \mathbf{z}_{s'a'} \rb \right | \rp.  
\end{align}

\end{proof}

\begin{theorem}
Given functions $f(s)=\mathbf{z}_s$ and $g(\mathbf{z}_s, a)=\mathbf{z}_{sa}$, then if there exists functions $\hat p$ and $\hat R$ such that for all $(s,a) \in S \times A$: 
\begin{align} \label{eqn:appendix:optimal_pi}
    &\E_{\hat R} [ \hat R(\mathbf{z}_{sa}) ] = \E_R \lb R(s,a) \rb, 
    &\hat p(\mathbf{z}_{s'}| \mathbf{z}_{sa}) = \sum_{\hat s : \mathbf{z}_{\hat s}=\mathbf{z}_{s'}} p(\hat s|s, a),
\end{align}
then for any policy~$\pi$ where there exists a corresponding policy~$\hat \pi(a | \mathbf{z}_s) = \pi(a | s)$, there exists a function~$\hat Q$ equal to the true value function $Q^\pi$ over all possible state-action pairs~$(s,a) \in S \times A$:
    \begin{equation}
    \hat Q(\mathbf{z}_{sa}) = Q^\pi(s,a).    
    \end{equation} 
Furthermore, \ref{eqn:appendix:optimal_pi} guarantees the existence of an optimal policy~$\hat \pi^*(a | \mathbf{z}_s) = \pi^*(a|s)$. 
\end{theorem}

\begin{proof}

Let 
\begin{align}
Q^\pi_h(s,a) &= \sum_{t=0}^h \y^t \E_\pi [ R(s_t, a_t) |s_0=s, a_0=a ] \\
\hat Q_h (\mathbf{z}_{sa}) &= \sum_{t=0}^h \y^t \E_\pi [ \hat R(\mathbf{z}_{s_t a_t}) |s_0=s, a_0=a ] 
\end{align}

Then
\begin{align}
    Q^\pi_0(s,a) &= \E_R [ R(s,a) ] \\
    &= \E_{\hat R} [ \hat R (\mathbf{z}_{sa}) ] \\
    &= \hat Q_0(\mathbf{z}_{sa}).
\end{align}

Assuming $Q^\pi_{n-1}(s,a) = \hat Q_{n-1}(\mathbf{z}_{sa})$ then noting that $\hat p(\mathbf{z} | \mathbf{z}_{sa})=0$ if $\mathbf{z}$ that is not in the image of $f(s) = \mathbf{z}_s$. 
\begin{align}
    Q^\pi_n(s,a) &= \E_R [ R(s,a) ] + \y \E_{s',a'} [ Q^\pi_{n-1}(s',a') ] \\
    &= \E_{\hat R} [ \hat R(s,a) ] + \y \E_{s',a'} [ \hat Q_{n-1}(\mathbf{z}_{s'a'}) ] \\
    &= \E_{\hat R} [ \hat R(s,a) ] + \y \sum_{s'} \sum_{a'} p(s'|s,a) \pi(a'|s') \hat Q_{n-1}(\mathbf{z}_{s'a'}) \\
    &= \E_{\hat R} [ \hat R(s,a) ] + \y \sum_{z_{s'}} \sum_{a'} \hat p(\mathbf{z}_{s'}|\mathbf{z}_{sa}) \hat \pi(a'|\mathbf{z}_{s'}) \hat Q_{n-1}(\mathbf{z}_{s'a'}) \\
    &= \hat Q_n (\mathbf{z}_{sa}).
\end{align}
Thus $\hat Q (\mathbf{z}_{sa}) = \lim_{n \rightarrow \infty} \hat Q_n (\mathbf{z}_{sa})$ exists, as $\hat Q_n$ can be defined as a function of $\hat p$, $\hat R$, and $\hat \pi$ for all~$n$.  

Similarly, let $\pi$ be an optimal policy. Repeating the same arguments we see that
\begin{align}
    Q^\pi_n(s,a) &= \E_R [ R(s,a) ] + \y \E_{s',a'} [ Q^\pi_{n-1}(s',a') ] \\
    &= \E_R [ R(s,a) ] + \y \sum_{s'} p(s'|s,a) \max_{a'} Q^\pi_{n-1}(s',a') \\
    &= \E_{\hat R} [ \hat R(s,a) ] + \y \sum_{z_{s'}} \hat p(\mathbf{z}_{s'}|\mathbf{z}_{sa})  \max_{a'} \hat Q_{n-1}(\mathbf{z}_{s'a'}) \\
    &= \hat Q_n (\mathbf{z}_{sa}).
\end{align}
Thus there exists a function $\hat Q (g(\mathbf{z}_s, a))=Q^*(s,a)$, consequently, there exists an optimal policy $\hat \pi^*(a|\mathbf{z}_s) = \argmax_a \hat Q(s,a)$.

\end{proof}

\clearpage

\section{Experimental Details} \label{appendix:exp_details}

\subsection{Hyperparameters}

\begin{table}[ht]
\caption{\textbf{MR.Q Hyperparameters.} Hyperparameters values are kept fixed across all benchmarks.}\label{table:hyperparameters}
\centering
\small
\begin{tabular}{cll}
\toprule
& Hyperparameter & Value \\
\midrule
& Dynamics loss weight $\lambda_\text{Dynamics}$ & $1$ \\
& Reward loss weight $\lambda_\text{Reward}$ & $0.1$ \\
& Terminal loss weight $\lambda_\text{Terminal}$ & $0.1$ \\
& Pre-activation loss weight $\lambda_\text{pre-activ}$ & $1\text{e}-5$ \\
& Encoder horizon $H_\text{Enc}$ & $5$ \\
& Multi-step returns horizon $H_Q$ & $3$ \\ 
\midrule
\multirow{2}{*}{\shortstack{TD3\\\citep{fujimoto2018addressing}}} 
& Target policy noise~$\sigma$      & $\N(0,0.2^2)$ \\
& Target policy noise clipping~$c$  & $(-0.3, 0.3)$ \\
\midrule
\multirow{2}{*}{\shortstack{LAP\\\citep{fujimoto2020equivalence}}} 
& Probability smoothing $\alpha$    & $0.4$ \\
& Minimum priority                  & $1$ \\
\midrule
\multirow{2}{*}{Exploration} & Initial random exploration time steps & $10$k \\
& Exploration noise    & $\N(0,0.2^2)$ \\
\midrule
\multirow{5}{*}{Common}
& Discount factor~$\y$      & $0.99$ \\
& Replay buffer capacity    & $1$M \\
& Mini-batch size           & $256$ \\
& Target update frequency~$T_\text{target}$   & $250$ \\
& Replay ratio              & $1$ \\ 
\midrule
\multirow{12}{*}{Encoder Network} 
& Optimizer        & AdamW~\citep{loshchilov2018decoupled} \\
& Learning rate    & $1\text{e}-4$ \\
& Weight decay     & $1\text{e}-4$ \\
& $\mathbf{z}_s$ dim & $512$ \\
& $\mathbf{z}_{sa}$ dim & $512$ \\
& $\mathbf{z}_a$ dim (only used within architecture) & $256$ \\
& Hidden dim & $512$ \\
& Activation function & ELU~\citep{clevert2015fast} \\
& Weight initialization & Xavier uniform~\citep{glorot2010understanding}
 \\
& Bias initialization & $0$ \\
& Reward bins & $65$ \\
& Reward range & $[-10,10]$ (effective: $[-22\text{k}, 22\text{k}]$) \\
\midrule
\multirow{7}{*}{Value Network} 
& Optimizer        & AdamW \\
& Learning rate    & $3\text{e}-4$ \\
& Hidden dim & $512$ \\
& Activation function & ELU \\
& Weight initialization & Xavier uniform \\
& Bias initialization & $0$ \\
& Gradient clip norm & $20$ \\
\midrule
\multirow{7}{*}{Policy Network} 
& Optimizer        & AdamW \\
& Learning rate    & $3\text{e}-4$ \\
& Hidden dim & $512$ \\
& Activation function & ReLU \\
& Weight initialization & Xavier uniform \\
& Bias initialization & $0$ \\
& Gumbel-Softmax~$\tau$~\citep{jang2017categorical} & $10$ \\ 
\bottomrule
\end{tabular}
\end{table} %

\clearpage

\subsection{Network Architecture}

This section describes the networks used in our method using PyTorch code blocks~\citep{paszke2019pytorch}. The state encoder and state-action encoder are described as separate networks for clarity but are trained end-to-end as a single network. The value and policy networks are trained independently from the encoders.

\textbf{Preamble}
\begin{lstlisting}
import torch
import torch.nn as nn
import torch.nn.functional as F
from functools import partial

zs_dim = 512
za_dim = 256
zsa_dim = 512

def ln_activ(self, x):
    x = F.layer_norm(x, (x.shape[-1],))
    return self.activ(x)
\end{lstlisting}

\textbf{State Encoder $f$ Network}

For image inputs, four convolutional layers are used, each with $32$ output channels, kernel size of~$3$, strides of $(2,2,2,1)$, and ELU activations~\citep{clevert2015fast}. The convolutional layers are followed by a linear layer taking in the flattened output followed by LayerNorm~\citep{ba2016layer} and a final ELU activation.

For vector inputs, a three layer multilayer perceptron (MLP) is used, with hidden dimension $512$ and LayerNorm followed by ELU activations after each layer.

The resulting state embedding $\mathbf{z}_s$ is trained end-to-end with the state-action encoder. It is also used downstream by the policy network (without propagating gradients).
\begin{lstlisting}
if image_observation_space:
    self.zs_cnn1 = nn.Conv2d(state_channels, 32, 3, stride=2)
    self.zs_cnn2 = nn.Conv2d(32, 32, 3, stride=2)
    self.zs_cnn3 = nn.Conv2d(32, 32, 3, stride=2)
    self.zs_cnn4 = nn.Conv2d(32, 32, 3, stride=1)
    # Assumes 84 x 84 input
    self.zs_lin = nn.Linear(1568, zs_dim)
else:
    self.zs_mlp1 = nn.Linear(state_dim, 512)
    self.zs_mlp2 = nn.Linear(512, 512)
    self.zs_mlp3 = nn.Linear(512, zs_dim)

self.activ = F.elu

def cnn_forward(self, state):
    state = state/255. - 0.5
    zs = self.activ(self.zs_cnn1(state))
    zs = self.activ(self.zs_cnn2(zs))
    zs = self.activ(self.zs_cnn3(zs))
    zs = self.activ(self.zs_cnn4(zs))
    zs = zs.reshape(batch_size, 1568)
    return ln_activ(self.zs_lin(zs))

def mlp_forward(self, state):
    zs = self.ln_activ(self.zs_mlp1(state))
    zs = self.ln_activ(self.zs_mlp2(zs))
    return self.ln_activ(self.zs_mlp3(zs))
\end{lstlisting}

\clearpage

\textbf{State-Action Encoder $g$ Network}

Action input is processed by a linear layer followed by an ELU activation. Afterwards, the processed action is concatenated with the state embedding and processed by a three layer MLP with hidden dimension $512$, and LayerNorm followed by ELU activations after the first two layers.

The resulting state-action embedding $\mathbf{z}_{sa}$ is used by a linear layer to make predictions about reward, the next state embedding, and the terminal signal. It is also used downstream by the value network (without propagating gradients).

\begin{lstlisting}
self.za = nn.Linear(action_dim, za_dim)
self.zsa1 = nn.Linear(zs_dim + za_dim, 512)
self.zsa2 = nn.Linear(512, 512)
self.zsa3 = nn.Linear(512, zsa_dim)
self.model = nn.Linear(zsa_dim, output_dim)
self.activ = F.elu

def forward(self, zs, action):
    za = self.activ(self.za(action))
    zsa = torch.cat([zs, za], 1)
    zsa = self.ln_activ(self.zsa1(zsa))
    zsa = self.ln_activ(self.zsa2(zsa))
    zsa = self.zsa3(zsa)
    return self.model(zsa), zsa
\end{lstlisting}

\textbf{Value $Q$ Networks}

The value network is a four layer MLP with hidden dimension $512$, and LayerNorm followed by ELU activations after the first three layers.

Two value networks are used with the same network and forward pass.

\begin{lstlisting}
self.l1 = nn.Linear(zsa_dim, 512)
self.l2 = nn.Linear(512, 512)
self.l3 = nn.Linear(512, 512)
self.l4 = nn.Linear(512, 1)
self.activ = F.elu

def forward(self, zsa):
    q = self.ln_activ(self.l1(zsa))
    q = self.ln_activ(self.l2(q))
    q = self.ln_activ(self.l3(q))
    return self.l4(q)
\end{lstlisting}

\textbf{Policy $\pi$ Network}

The policy network is a three layer MLP with hidden dimension $512$, and LayerNorm followed by ReLU activations after the first two layers.

For discrete actions, the final activation is the Gumbel Softmax with $\tau=10$. For continous actions, the final activation is a tanh function.
\begin{lstlisting}
self.l1 = nn.Linear(zs_dim, 512)
self.l2 = nn.Linear(hdim, 512)
self.l3 = nn.Linear(512, action_dim)
self.activ = F.relu

if discrete_action_space:
    self.final_activ = partial(F.gumbel_softmax, tau=10) 
else: 
    self.final_activ = torch.tanh

def forward(self, zs):
    a = self.ln_activ(self.l1(zs))
    a = self.ln_activ(self.l2(a))
    return self.final_activ(self.l3(a))
\end{lstlisting}

\clearpage

\subsection{Environments} \label{appendix:envs}

All main experiments were run for 10 seeds (the design study is based on 5 seeds). Evaluations are based on the average performance over 10 episodes, measured every 5k time steps for Gym and DM control and every 100k time steps for Atari. 

\textbf{Gym - Locomotion.} For the gym locomotion tasks~\citep{mujoco, OpenAIGym, towers2024gymnasium}, we choose the five most common environments that appear in prior work~\citep{fujimoto2018addressing, fujimoto2024sale, haarnoja2018soft, kuznetsov2020controlling}. We use the -v4 version. No preprocessing is applied. When aggregating scores, we use normalize with the TD3 scores obtained from TD7~\citep{fujimoto2024sale}:
\begin{equation}
    \text{TD3-Normalized}(x) := \frac{x - \text{random score}}{\text{TD3 score} - \text{random score}}.
\end{equation}

\begin{table}[ht]
\small
\centering
\begin{tabular}{lrr}
\toprule
& Random & TD3 \\
\midrule
Ant-v4 & -70.288 & \z3942\\
HalfCheetah-v4 & -289.415 & 10574\\
Hopper-v4 & 18.791 & \z3226\\
Humanoid-v4 & 120.423 & \z5165 \\
Walker2d-v4 & 2.791 & \z3946 \\
\bottomrule
\end{tabular}
\end{table}

\textbf{DM Control Suite.} For the DM control suite~\citep{tassa2018deepmind},  we choose the 28 default environments that appear either in the evaluation of TD-MPC2 or DreamerV3. We omit any custom environments included by the TD-MPC2 authors. The same subset of tasks are used in the evaluation of proprioceptive and visual control. Like prior work, for both observation spaces, we use an action repeat of $2$~\citep{hansen2024td}. For visual control, the state (network input) is composed of the previous $3$ observations which are resized to $84 \times 84$ pixels in RGB format~\citep{tassa2018deepmind}.

\textbf{Atari.} For the Atari games~\citep{bellemare2013arcade, OpenAIGym, towers2024gymnasium}, we use the 57 games in the Atari-57 benchmark that appears in prior work~\citep{hessel2018rainbow, schrittwieser2020mastering, badia2020agent57, hafner2023mastering}. For DQN and Rainbow, two games (Defender and Surround) are missing from the Dopamine framework~\citep{castro2018dopamine} and are omitted. We use the -v5 version. For MR.Q, we use the common preprocessing steps~\citep{DQN, machado2018revisiting, castro2018dopamine}, where an action repeat of $4$ is used and the observations are grayscaled, resized to $84 \times 84$ pixels and set to the max between the 3rd and 4th frame. The state~(network input) is composed of the previous $4$ observations.

Consider the $16$ frame sequence used by a single state, where $f_i$ is the $i$th grayscaled and resized frame and $o_j$ is the $j$th observation set to the max of two frames 
\begin{equation}
\overbracket[\fontdimen8\textfont3]{f_0, f_1, \underbracket[\fontdimen8\textfont3]{f_2, f_3}_{o_0 = \max (f_2, f_3)}
}^{\text{action } a_0}, 
\overbracket[\fontdimen8\textfont3]{f_4, f_5, \underbracket[\fontdimen8\textfont3]{f_6, f_7}_{o_1 = \max (f_6, f_7)}}^{\text{action } a_1},
\overbracket[\fontdimen8\textfont3]{f_8, f_9, \underbracket[\fontdimen8\textfont3]{f_{10}, f_{11}}_{o_2 = \max (f_{10}, f_{11})}}^{\text{action } a_2},
\overbracket[\fontdimen8\textfont3]{f_{12}, f_{13}, \underbracket[\fontdimen8\textfont3]{f_{14}, f_{15}}_{o_3 = \max (f_{14}, f_{15})}}^{\text{action } a_3}, 
\end{equation}
then the state is defined as follows: 
\begin{equation}
s = \begin{bmatrix}
       o_0 = \max (f_2, f_3) \phantom{_{11}} \\
       o_1 = \max (f_6, f_7) \phantom{_{11}} \\
       o_2 = \max (f_{10}, f_{11}) \\
       o_3 = \max (f_{14}, f_{15})
     \end{bmatrix}. 
\end{equation}

When aggregating scores, we normalize with Human scores obtained from \citep{wang2015dueling}:
\begin{equation}
    \text{Human-Normalized}(x) := \frac{x - \text{random score}}{\text{Human score} - \text{random score}}.
\end{equation}

\begin{table}[ht]
\small
\centering
\begin{tabular}{lrr}
\toprule
& Random & Human \\
\midrule
Alien & 227.8 & 7127.7 \\
Amidar & 5.8 & 1719.5 \\
Assault & 222.4 & 742.0 \\
Asterix & 210.0 & 8503.3 \\
Asteroids & 719.1 & 47388.7 \\
Atlantis & 12850.0 & 29028.1 \\
BankHeist & 14.2 & 753.1 \\
BattleZone & 2360.0 & 37187.5 \\
BeamRider & 363.9 & 16926.5 \\
Berzerk & 123.7 & 2630.4 \\
Bowling & 23.1 & 160.7 \\
Boxing & 0.1 & 12.1 \\
Breakout & 1.7 & 30.5 \\
Centipede & 2090.9 & 12017.0 \\
ChopperCommand & 811.0 & 7387.8 \\
CrazyClimber & 10780.5 & 35829.4 \\
Defender (not used) & 2874.5 & 18688.9 \\
DemonAttack & 152.1 & 1971.0 \\
DoubleDunk & -18.6 & -16.4 \\
Enduro & 0.0 & 860.5 \\
FishingDerby & -91.7 & -38.7 \\
Freeway & 0.0 & 29.6 \\
Frostbite & 65.2 & 4334.7 \\
Gopher & 257.6 & 2412.5 \\
Gravitar & 173.0 & 3351.4 \\
Hero & 1027.0 & 30826.4 \\
IceHockey & -11.2 & 0.9 \\
Jamesbond & 29.0 & 302.8 \\
Kangaroo & 52.0 & 3035.0 \\
Krull & 1598.0 & 2665.5 \\
KungFuMaster & 258.5 & 22736.3 \\
MontezumaRevenge & 0.0 & 4753.3 \\
MsPacman & 307.3 & 6951.6 \\
NameThisGame & 2292.3 & 8049.0 \\
Phoenix & 761.4 & 7242.6 \\
Pitfall & -229.4 & 6463.7 \\
Pong & -20.7 & 14.6 \\
PrivateEye & 24.9 & 69571.3 \\
Qbert & 163.9 & 13455.0 \\
Riverraid & 1338.5 & 17118.0 \\
RoadRunner & 11.5 & 7845.0 \\
Robotank & 2.2 & 11.9 \\
Seaquest & 68.4 & 42054.7 \\
Skiing & -17098.1 & -4336.9 \\
Solaris & 1236.3 & 12326.7 \\
SpaceInvaders & 148.0 & 1668.7 \\
StarGunner & 664.0 & 10250.0 \\
Surround (not used) & -10.0 & 6.5 \\
Tennis & -23.8 & -8.3 \\
TimePilot & 3568.0 & 5229.2 \\
Tutankham & 11.4 & 167.6 \\
UpNDown & 533.4 & 11693.2 \\
Venture & 0.0 & 1187.5 \\
VideoPinball & 16256.9 & 17667.9 \\
WizardOfWor & 563.5 & 4756.5 \\
YarsRevenge & 3092.9 & 54576.9 \\
Zaxxon & 32.5 & 9173.3 \\
\bottomrule
\end{tabular}
\end{table}

\clearpage

\subsection{Baselines}

\textbf{DreamerV3.} \citep{hafner2023mastering}. Results for Gym and DMC were obtained by re-running the authors' code (\url{https://github.com/danijar/dreamerv3} - Commit \textcolor{gray}{251910d04c9f38dd9dc385775bb0d6\-efa0e57a95}) over 10 seeds, using the author-suggested hyperparameters from the DMC benchmark. Code was modified slightly to match our evaluation protocol. Atari results are based on the authors' reported results. 

\textbf{DrQ-v2.} \citep{yarats2022mastering}. We use the authors' reported results whenever possible. For missing any results, we re-ran the authors' code (\url{https://github.com/facebookresearch/drqv2} - Commit \textcolor{gray}{c0c650b76c6e5d22a7eb5f2edffd1440fe94f8ef}) for 10 seeds.

\textbf{DQN.} \citep{DQN}. Results were obtained from the Dopamine framework~\citep{castro2018dopamine}.

\textbf{PPO.} \citep{ppo}. Results were gathered using Stable Baselines 3~\citep{stable-baselines3} and default hyperparameters. The default MLP policy was used for Gym and DMC-proprioceptive and the default CNN policy was used for DMC-visual and Atari.

\textbf{Rainbow.} \citep{hessel2018rainbow}. Results were obtained from the Dopamine framework~\citep{castro2018dopamine}.

\textbf{TD-MPC2.} \citep{hansen2024td}. Results for DMC were obtained by re-running the authors' code on their main branch (\url{https://github.com/nicklashansen/tdmpc2} - Commit \textcolor{gray}{5f6fadec0fec78304b4b53e8171d348b58cac486}). As the Gym environments include a termination signal, results for Gym were obtained by running their episodic branch (\url{https://github.com/nicklashansen/tdmpc2/tree/episodic-rl} - Commit \textcolor{gray}{3789fcd5b872079ad610fa3299ff47c3a427a04a}). All experiments were run for 10 seeds and use the default author-suggested hyperparameters for all tasks.

\textbf{TD7.} \citep{fujimoto2024sale}. Results for Gym were obtained from the authors. Results for DMC were obtained by re-running the authors' code (\url{https://github.com/sfujim/TD7} - Commit \textcolor{gray}{c1c280de1513f474488061b4cf39642b75dd84bd}) using our setup for DMC. All experiments use 10 seeds and use the default author-suggested hyperparameters from the Gym benchmark.

\subsection{Software Versions}

\begin{itemize}[nosep]
    \item Gymnasium 0.29.1~\citep{towers2024gymnasium}
    \item MuJoCo 3.2.2~\citep{mujoco}
    \item NumPy 2.1.1~\citep{harris2020array}
    \item Python 3.11.8~\citep{van1995python}
    \item PyTorch 2.4.1~\citep{paszke2019pytorch}
\end{itemize}

\clearpage

\section{Complete Main Results} \label{appendix:additional_results}

\subsection{Gym}

\begin{table}[ht]
\centering
\scriptsize
\setlength{\tabcolsep}{0pt}
\newcolumntype{Y}{>{\centering\arraybackslash}X} %
\caption{\textbf{Gym - Locomotion final results.} Final average performance at 1M time steps over 10 seeds. \tabledescript TD3-normalized score (see \cref{appendix:envs}).}
\begin{tabularx}{\textwidth}{X@{\hspace{4pt}} r@{}X@{\hspace{4pt}} r@{}X@{\hspace{4pt}} r@{}X@{\hspace{4pt}} r@{}X@{\hspace{4pt}} r@{}l@{}}
\toprule
Task & & \multicolumn{1}{l}{TD7} & & \multicolumn{1}{l}{PPO} & & \multicolumn{1}{l}{TD-MPC2} & & \multicolumn{1}{l}{DreamerV3} & & \multicolumn{1}{l}{MR.Q} \\
\midrule
Ant & 8509 & ~\textcolor{gray}{[8164, 8852]} & 1584 & ~\textcolor{gray}{[1355, 1802]} & 4751 & ~\textcolor{gray}{[3012, 6261]} & 1947 & ~\textcolor{gray}{[1121, 2751]} & 6901 & ~\textcolor{gray}{[6261, 7482]} \\
HalfCheetah & 17433 & ~\textcolor{gray}{[17284, 17550]} & 1744 & ~\textcolor{gray}{[1525, 2120]} & 15078 & ~\textcolor{gray}{[14050, 16012]} & 5502 & ~\textcolor{gray}{[3887, 7117]} & 12939 & ~\textcolor{gray}{[11663, 13762]} \\
Hopper & 3511 & ~\textcolor{gray}{[3245, 3746]} & 3022 & ~\textcolor{gray}{[2587, 3356]} & 2081 & ~\textcolor{gray}{[1233, 2916]} & 2666 & ~\textcolor{gray}{[2071, 3201]} & 2692 & ~\textcolor{gray}{[2131, 3309]} \\
Humanoid & 7428 & ~\textcolor{gray}{[7300, 7555]} & 477 & ~\textcolor{gray}{[431, 522]} & 6071 & ~\textcolor{gray}{[5767, 6327]} & 4217 & ~\textcolor{gray}{[2791, 5481]} & 10223 & ~\textcolor{gray}{[9929, 10498]} \\
Walker2d & 6096 & ~\textcolor{gray}{[5535, 6521]} & 2487 & ~\textcolor{gray}{[1875, 3067]} & 3008 & ~\textcolor{gray}{[1659, 4220]} & 4519 & ~\textcolor{gray}{[3746, 5190]} & 6039 & ~\textcolor{gray}{[5644, 6386]} \\
\midrule
Mean & 1.57 & ~\textcolor{gray}{[1.54, 1.60]} & 0.45 & ~\textcolor{gray}{[0.41, 0.48]} & 1.04 & ~\textcolor{gray}{[0.90, 1.16]} & 0.76 & ~\textcolor{gray}{[0.67, 0.85]} & 1.46 & ~\textcolor{gray}{[1.41, 1.52]} \\
Median & 1.55 & ~\textcolor{gray}{[1.45, 1.63]} & 0.41 & ~\textcolor{gray}{[0.36, 0.47]} & 1.18 & ~\textcolor{gray}{[0.80, 1.23]} & 0.81 & ~\textcolor{gray}{[0.56, 0.90]} & 1.53 & ~\textcolor{gray}{[1.43, 1.61]} \\
IQM & 1.54 & ~\textcolor{gray}{[1.49, 1.58]} & 0.41 & ~\textcolor{gray}{[0.35, 0.46]} & 1.05 & ~\textcolor{gray}{[0.87, 1.19]} & 0.72 & ~\textcolor{gray}{[0.62, 0.85]} & 1.50 & ~\textcolor{gray}{[1.44, 1.55]} \\
\bottomrule
\end{tabularx}
\end{table}

\begin{figure}[ht]
\centering
\includegraphics[width=\linewidth,trim=1.5in 7.15in 1.5in 1.15in,clip]{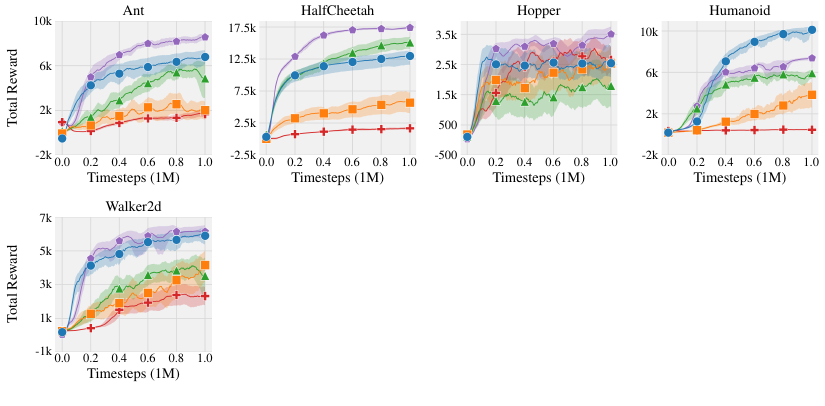}

\fcolorbox{gray!10}{gray!10}{
\small
\coloredcircle{sb_blue}{sb_blue!25!light_gray} MR.Q \quad
\coloredsquare{sb_orange}{sb_orange!25!light_gray} DreamerV3 \quad
\coloredtriangle{sb_green}{sb_green!25!light_gray} TD-MPC2 \quad
\coloredplus{sb_red}{sb_red!25!light_gray} PPO \quad
\coloredpoly{sb_purple}{sb_purple!25!light_gray} TD7
}

\captionof{figure}{\textbf{Gym - Locomotion learning curves.} Results are over 10 seeds. \figuredescript
} %
\end{figure}

\clearpage

\subsection{DMC - Proprioceptive}

\begin{table}[ht]
\centering
\scriptsize
\setlength{\tabcolsep}{0pt}
\newcolumntype{Y}{>{\centering\arraybackslash}X} %
\caption{\textbf{DMC - Proprioceptive final results.} Final average performance at 500k time steps (1M time steps in the original environment due to action repeat) over 10 seeds. \tabledescript default reward.}
\begin{tabularx}{\textwidth}{l@{\hspace{4pt}} r@{}X@{\hspace{4pt}} r@{}X@{\hspace{4pt}} r@{}X@{\hspace{4pt}} r@{}X@{\hspace{4pt}} r@{}l@{}}
\toprule
Task & & \multicolumn{1}{l}{TD7} & & \multicolumn{1}{l}{PPO} & & \multicolumn{1}{l}{TD-MPC2} & & \multicolumn{1}{l}{DreamerV3} & & \multicolumn{1}{l}{MR.Q} \\
\midrule
acrobot-swingup & 58 & ~\textcolor{gray}{[38, 75]} & 39 & ~\textcolor{gray}{[33, 45]} & 584 & ~\textcolor{gray}{[551, 615]} & 230 & ~\textcolor{gray}{[193, 266]} & 567 & ~\textcolor{gray}{[523, 616]} \\
ball\_in\_cup-catch & 983 & ~\textcolor{gray}{[981, 985]} & 769 & ~\textcolor{gray}{[689, 841]} & 984 & ~\textcolor{gray}{[982, 986]} & 968 & ~\textcolor{gray}{[965, 973]} & 981 & ~\textcolor{gray}{[979, 984]} \\
cartpole-balance & 999 & ~\textcolor{gray}{[998, 1000]} & 999 & ~\textcolor{gray}{[1000, 1000]} & 996 & ~\textcolor{gray}{[995, 998]} & 998 & ~\textcolor{gray}{[997, 1000]} & 999 & ~\textcolor{gray}{[999, 1000]} \\
cartpole-balance\_sparse & 1000 & ~\textcolor{gray}{[1000, 1000]} & 1000 & ~\textcolor{gray}{[1000, 1000]} & 1000 & ~\textcolor{gray}{[1000, 1000]} & 999 & ~\textcolor{gray}{[1000, 1000]} & 1000 & ~\textcolor{gray}{[1000, 1000]} \\
cartpole-swingup & 869 & ~\textcolor{gray}{[866, 873]} & 776 & ~\textcolor{gray}{[661, 853]} & 875 & ~\textcolor{gray}{[870, 880]} & 736 & ~\textcolor{gray}{[591, 838]} & 866 & ~\textcolor{gray}{[866, 866]} \\
cartpole-swingup\_sparse & 573 & ~\textcolor{gray}{[333, 806]} & 391 & ~\textcolor{gray}{[159, 625]} & 845 & ~\textcolor{gray}{[839, 849]} & 702 & ~\textcolor{gray}{[560, 792]} & 798 & ~\textcolor{gray}{[780, 818]} \\
cheetah-run & 821 & ~\textcolor{gray}{[642, 913]} & 269 & ~\textcolor{gray}{[247, 295]} & 917 & ~\textcolor{gray}{[915, 920]} & 699 & ~\textcolor{gray}{[655, 744]} & 914 & ~\textcolor{gray}{[911, 917]} \\
dog-run & 69 & ~\textcolor{gray}{[36, 101]} & 26 & ~\textcolor{gray}{[26, 28]} & 265 & ~\textcolor{gray}{[166, 342]} & 4 & ~\textcolor{gray}{[4, 5]} & 569 & ~\textcolor{gray}{[547, 595]} \\
dog-stand & 582 & ~\textcolor{gray}{[432, 741]} & 129 & ~\textcolor{gray}{[122, 139]} & 506 & ~\textcolor{gray}{[266, 715]} & 22 & ~\textcolor{gray}{[20, 27]} & 967 & ~\textcolor{gray}{[960, 975]} \\
dog-trot & 21 & ~\textcolor{gray}{[13, 30]} & 31 & ~\textcolor{gray}{[30, 34]} & 407 & ~\textcolor{gray}{[265, 530]} & 10 & ~\textcolor{gray}{[6, 17]} & 877 & ~\textcolor{gray}{[845, 898]} \\
dog-walk & 52 & ~\textcolor{gray}{[19, 116]} & 40 & ~\textcolor{gray}{[37, 43]} & 486 & ~\textcolor{gray}{[240, 704]} & 17 & ~\textcolor{gray}{[15, 21]} & 916 & ~\textcolor{gray}{[908, 924]} \\
finger-spin & 335 & ~\textcolor{gray}{[99, 596]} & 459 & ~\textcolor{gray}{[420, 497]} & 986 & ~\textcolor{gray}{[986, 988]} & 666 & ~\textcolor{gray}{[577, 763]} & 937 & ~\textcolor{gray}{[917, 956]} \\
finger-turn\_easy & 912 & ~\textcolor{gray}{[774, 983]} & 182 & ~\textcolor{gray}{[153, 211]} & 979 & ~\textcolor{gray}{[975, 983]} & 906 & ~\textcolor{gray}{[883, 927]} & 953 & ~\textcolor{gray}{[931, 974]} \\
finger-turn\_hard & 470 & ~\textcolor{gray}{[199, 727]} & 58 & ~\textcolor{gray}{[35, 79]} & 947 & ~\textcolor{gray}{[916, 977]} & 864 & ~\textcolor{gray}{[812, 900]} & 950 & ~\textcolor{gray}{[910, 974]} \\
fish-swim & 86 & ~\textcolor{gray}{[64, 120]} & 103 & ~\textcolor{gray}{[84, 128]} & 659 & ~\textcolor{gray}{[615, 706]} & 813 & ~\textcolor{gray}{[808, 819]} & 792 & ~\textcolor{gray}{[773, 810]} \\
hopper-hop & 87 & ~\textcolor{gray}{[25, 160]} & 10 & ~\textcolor{gray}{[0, 23]} & 425 & ~\textcolor{gray}{[368, 500]} & 116 & ~\textcolor{gray}{[66, 165]} & 251 & ~\textcolor{gray}{[195, 301]} \\
hopper-stand & 670 & ~\textcolor{gray}{[466, 829]} & 128 & ~\textcolor{gray}{[56, 216]} & 952 & ~\textcolor{gray}{[944, 958]} & 747 & ~\textcolor{gray}{[669, 806]} & 951 & ~\textcolor{gray}{[948, 955]} \\
humanoid-run & 57 & ~\textcolor{gray}{[23, 92]} & 0 & ~\textcolor{gray}{[1, 1]} & 181 & ~\textcolor{gray}{[121, 231]} & 0 & ~\textcolor{gray}{[1, 1]} & 200 & ~\textcolor{gray}{[170, 236]} \\
humanoid-stand & 317 & ~\textcolor{gray}{[117, 516]} & 5 & ~\textcolor{gray}{[5, 6]} & 658 & ~\textcolor{gray}{[506, 745]} & 5 & ~\textcolor{gray}{[5, 6]} & 868 & ~\textcolor{gray}{[822, 903]} \\
humanoid-walk & 176 & ~\textcolor{gray}{[42, 320]} & 1 & ~\textcolor{gray}{[1, 2]} & 754 & ~\textcolor{gray}{[725, 791]} & 1 & ~\textcolor{gray}{[1, 2]} & 662 & ~\textcolor{gray}{[610, 724]} \\
pendulum-swingup & 500 & ~\textcolor{gray}{[251, 743]} & 115 & ~\textcolor{gray}{[70, 164]} & 846 & ~\textcolor{gray}{[830, 862]} & 774 & ~\textcolor{gray}{[740, 802]} & 748 & ~\textcolor{gray}{[597, 829]} \\
quadruped-run & 645 & ~\textcolor{gray}{[567, 713]} & 144 & ~\textcolor{gray}{[122, 170]} & 942 & ~\textcolor{gray}{[938, 947]} & 130 & ~\textcolor{gray}{[92, 169]} & 947 & ~\textcolor{gray}{[940, 954]} \\
quadruped-walk & 949 & ~\textcolor{gray}{[939, 957]} & 122 & ~\textcolor{gray}{[103, 142]} & 963 & ~\textcolor{gray}{[959, 967]} & 193 & ~\textcolor{gray}{[137, 243]} & 963 & ~\textcolor{gray}{[959, 967]} \\
reacher-easy & 970 & ~\textcolor{gray}{[951, 982]} & 367 & ~\textcolor{gray}{[188, 558]} & 983 & ~\textcolor{gray}{[980, 986]} & 966 & ~\textcolor{gray}{[964, 970]} & 983 & ~\textcolor{gray}{[983, 985]} \\
reacher-hard & 898 & ~\textcolor{gray}{[861, 936]} & 125 & ~\textcolor{gray}{[40, 234]} & 960 & ~\textcolor{gray}{[936, 979]} & 919 & ~\textcolor{gray}{[864, 955]} & 977 & ~\textcolor{gray}{[975, 980]} \\
walker-run & 804 & ~\textcolor{gray}{[783, 825]} & 97 & ~\textcolor{gray}{[91, 104]} & 854 & ~\textcolor{gray}{[851, 859]} & 510 & ~\textcolor{gray}{[430, 588]} & 793 & ~\textcolor{gray}{[765, 815]} \\
walker-stand & 983 & ~\textcolor{gray}{[974, 989]} & 431 & ~\textcolor{gray}{[363, 495]} & 991 & ~\textcolor{gray}{[990, 994]} & 941 & ~\textcolor{gray}{[934, 948]} & 988 & ~\textcolor{gray}{[987, 990]} \\
walker-walk & 977 & ~\textcolor{gray}{[975, 980]} & 283 & ~\textcolor{gray}{[253, 312]} & 981 & ~\textcolor{gray}{[979, 984]} & 898 & ~\textcolor{gray}{[875, 919]} & 978 & ~\textcolor{gray}{[978, 980]} \\
\midrule
Mean & 566 & ~\textcolor{gray}{[544, 590]} & 254 & ~\textcolor{gray}{[241, 267]} & 783 & ~\textcolor{gray}{[769, 797]} & 530 & ~\textcolor{gray}{[520, 539]} & 835 & ~\textcolor{gray}{[829, 842]} \\
Median & 613 & ~\textcolor{gray}{[548, 718]} & 127 & ~\textcolor{gray}{[112, 145]} & 896 & ~\textcolor{gray}{[893, 899]} & 700 & ~\textcolor{gray}{[644, 741]} & 927 & ~\textcolor{gray}{[914, 934]} \\
IQM & 612 & ~\textcolor{gray}{[569, 657]} & 154 & ~\textcolor{gray}{[135, 167]} & 868 & ~\textcolor{gray}{[860, 880]} & 577 & ~\textcolor{gray}{[557, 594]} & 907 & ~\textcolor{gray}{[903, 914]} \\
\bottomrule
\end{tabularx}
\end{table}

\begin{figure}[ht]
\centering
\includegraphics[width=\linewidth,trim=1.5in 1.75in 1.5in 1.15in,clip]{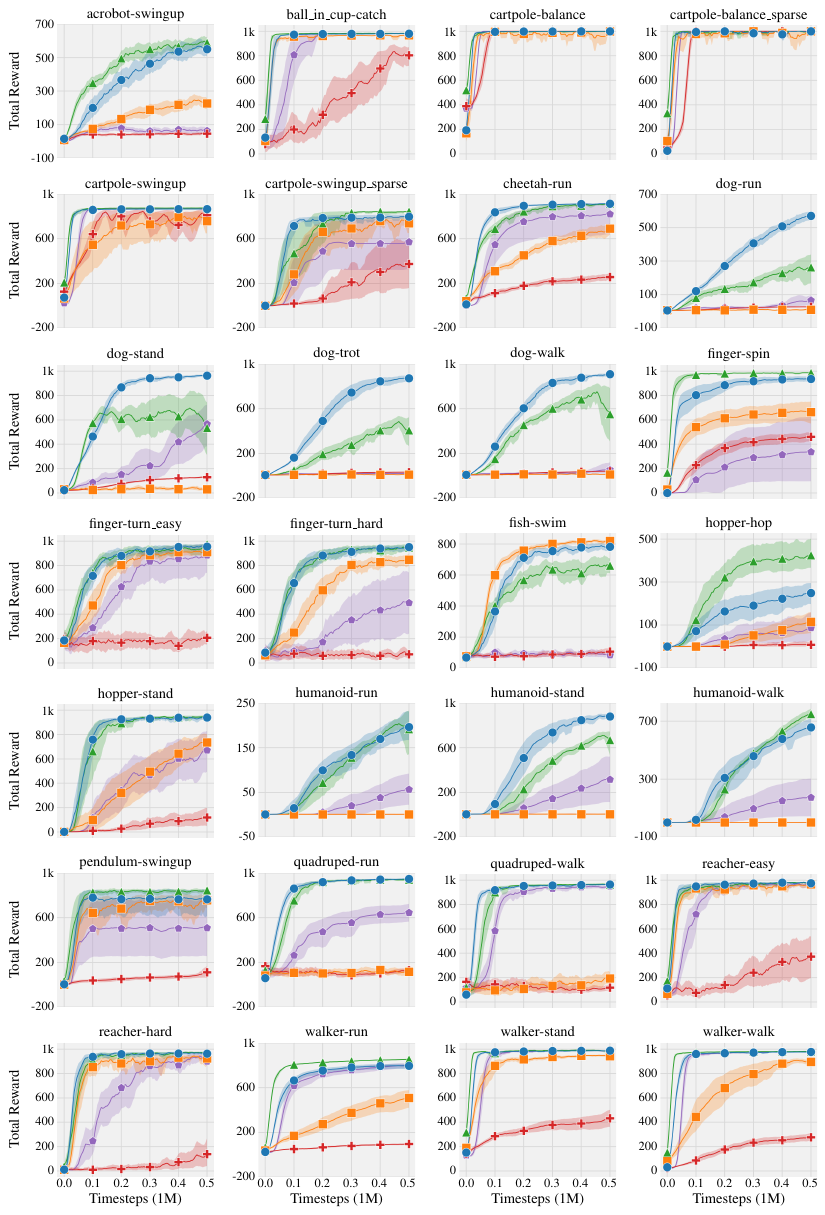} 

\fcolorbox{gray!10}{gray!10}{
\small
\coloredcircle{sb_blue}{sb_blue!25!light_gray} MR.Q \quad
\coloredsquare{sb_orange}{sb_orange!25!light_gray} DreamerV3 \quad
\coloredtriangle{sb_green}{sb_green!25!light_gray} TD-MPC2 \quad
\coloredplus{sb_red}{sb_red!25!light_gray} PPO \quad
\coloredpoly{sb_purple}{sb_purple!25!light_gray} TD7
}

\captionof{figure}{\textbf{DMC - Proprioceptive learning curves.} Time steps consider the number of environment interactions, where 500k time steps equals 1M frames in the original environment. Results are over 10 seeds. \figuredescript
} %
\end{figure}

\clearpage

\subsection{DMC - Visual}

\begin{table}[ht]
\centering
\scriptsize
\setlength{\tabcolsep}{0pt}
\newcolumntype{Y}{>{\centering\arraybackslash}X} %
\caption{\textbf{DMC - Visual final results.} Final average performance at 500k time steps (1M time steps in the original environment due to action repeat) over 10 seeds. \tabledescript default reward.}
\begin{tabularx}{\textwidth}{l@{\hspace{4pt}} r@{}X@{\hspace{4pt}} r@{}X@{\hspace{4pt}} r@{}X@{\hspace{4pt}} r@{}X@{\hspace{4pt}} r@{}l@{}}
\toprule
Task & & \multicolumn{1}{l}{DrQ-v2} & & \multicolumn{1}{l}{PPO} & & \multicolumn{1}{l}{TD-MPC2} & & \multicolumn{1}{l}{DreamerV3} & & \multicolumn{1}{l}{MR.Q} \\
\midrule
acrobot-swingup & 168 & ~\textcolor{gray}{[127, 219]} & 2 & ~\textcolor{gray}{[1, 4]} & 197 & ~\textcolor{gray}{[179, 217]} & 121 & ~\textcolor{gray}{[106, 145]} & 287 & ~\textcolor{gray}{[254, 316]} \\
ball\_in\_cup-catch & 909 & ~\textcolor{gray}{[821, 973]} & 105 & ~\textcolor{gray}{[5, 282]} & 932 & ~\textcolor{gray}{[899, 961]} & 971 & ~\textcolor{gray}{[969, 973]} & 977 & ~\textcolor{gray}{[975, 980]} \\
cartpole-balance & 993 & ~\textcolor{gray}{[990, 996]} & 353 & ~\textcolor{gray}{[231, 485]} & 972 & ~\textcolor{gray}{[948, 991]} & 998 & ~\textcolor{gray}{[997, 1000]} & 999 & ~\textcolor{gray}{[999, 999]} \\
cartpole-balance\_sparse & 962 & ~\textcolor{gray}{[887, 1000]} & 487 & ~\textcolor{gray}{[233, 751]} & 1000 & ~\textcolor{gray}{[1000, 1000]} & 999 & ~\textcolor{gray}{[999, 1000]} & 1000 & ~\textcolor{gray}{[1000, 1000]} \\
cartpole-swingup & 864 & ~\textcolor{gray}{[854, 873]} & 596 & ~\textcolor{gray}{[437, 723]} & 690 & ~\textcolor{gray}{[521, 813]} & 725 & ~\textcolor{gray}{[603, 807]} & 868 & ~\textcolor{gray}{[860, 875]} \\
cartpole-swingup\_sparse & 774 & ~\textcolor{gray}{[741, 805]} & 0 & ~\textcolor{gray}{[0, 0]} & 636 & ~\textcolor{gray}{[404, 804]} & 547 & ~\textcolor{gray}{[351, 726]} & 797 & ~\textcolor{gray}{[777, 816]} \\
cheetah-run & 728 & ~\textcolor{gray}{[701, 753]} & 155 & ~\textcolor{gray}{[110, 210]} & 431 & ~\textcolor{gray}{[267, 556]} & 618 & ~\textcolor{gray}{[576, 661]} & 775 & ~\textcolor{gray}{[752, 807]} \\
dog-run & 10 & ~\textcolor{gray}{[9, 12]} & 11 & ~\textcolor{gray}{[9, 14]} & 14 & ~\textcolor{gray}{[10, 18]} & 9 & ~\textcolor{gray}{[6, 14]} & 60 & ~\textcolor{gray}{[44, 80]} \\
dog-stand & 43 & ~\textcolor{gray}{[37, 49]} & 51 & ~\textcolor{gray}{[48, 56]} & 117 & ~\textcolor{gray}{[72, 148]} & 61 & ~\textcolor{gray}{[30, 92]} & 216 & ~\textcolor{gray}{[201, 232]} \\
dog-trot & 14 & ~\textcolor{gray}{[11, 18]} & 13 & ~\textcolor{gray}{[12, 15]} & 20 & ~\textcolor{gray}{[14, 25]} & 14 & ~\textcolor{gray}{[13, 16]} & 65 & ~\textcolor{gray}{[55, 79]} \\
dog-walk & 22 & ~\textcolor{gray}{[18, 29]} & 16 & ~\textcolor{gray}{[14, 18]} & 22 & ~\textcolor{gray}{[17, 28]} & 11 & ~\textcolor{gray}{[11, 12]} & 77 & ~\textcolor{gray}{[71, 83]} \\
finger-spin & 860 & ~\textcolor{gray}{[787, 922]} & 241 & ~\textcolor{gray}{[107, 377]} & 786 & ~\textcolor{gray}{[492, 984]} & 656 & ~\textcolor{gray}{[544, 765]} & 965 & ~\textcolor{gray}{[938, 982]} \\
finger-turn\_easy & 503 & ~\textcolor{gray}{[399, 615]} & 189 & ~\textcolor{gray}{[144, 233]} & 562 & ~\textcolor{gray}{[317, 779]} & 491 & ~\textcolor{gray}{[447, 542]} & 953 & ~\textcolor{gray}{[927, 974]} \\
finger-turn\_hard & 223 & ~\textcolor{gray}{[121, 340]} & 60 & ~\textcolor{gray}{[1, 120]} & 903 & ~\textcolor{gray}{[870, 940]} & 494 & ~\textcolor{gray}{[401, 571]} & 932 & ~\textcolor{gray}{[905, 957]} \\
fish-swim & 84 & ~\textcolor{gray}{[65, 107]} & 77 & ~\textcolor{gray}{[64, 92]} & 43 & ~\textcolor{gray}{[21, 64]} & 90 & ~\textcolor{gray}{[84, 96]} & 79 & ~\textcolor{gray}{[68, 93]} \\
hopper-hop & 224 & ~\textcolor{gray}{[170, 278]} & 0 & ~\textcolor{gray}{[0, 0]} & 187 & ~\textcolor{gray}{[119, 238]} & 205 & ~\textcolor{gray}{[125, 287]} & 270 & ~\textcolor{gray}{[230, 315]} \\
hopper-stand & 917 & ~\textcolor{gray}{[903, 931]} & 1 & ~\textcolor{gray}{[0, 2]} & 582 & ~\textcolor{gray}{[321, 794]} & 888 & ~\textcolor{gray}{[875, 900]} & 852 & ~\textcolor{gray}{[703, 930]} \\
humanoid-run & 1 & ~\textcolor{gray}{[1, 1]} & 1 & ~\textcolor{gray}{[1, 1]} & 0 & ~\textcolor{gray}{[1, 1]} & 1 & ~\textcolor{gray}{[1, 1]} & 1 & ~\textcolor{gray}{[1, 2]} \\
humanoid-stand & 6 & ~\textcolor{gray}{[7, 7]} & 6 & ~\textcolor{gray}{[6, 7]} & 5 & ~\textcolor{gray}{[5, 7]} & 5 & ~\textcolor{gray}{[5, 7]} & 7 & ~\textcolor{gray}{[7, 8]} \\
humanoid-walk & 2 & ~\textcolor{gray}{[2, 2]} & 1 & ~\textcolor{gray}{[1, 1]} & 1 & ~\textcolor{gray}{[1, 2]} & 1 & ~\textcolor{gray}{[2, 2]} & 2 & ~\textcolor{gray}{[2, 3]} \\
pendulum-swingup & 838 & ~\textcolor{gray}{[813, 861]} & 0 & ~\textcolor{gray}{[0, 1]} & 748 & ~\textcolor{gray}{[574, 850]} & 761 & ~\textcolor{gray}{[709, 807]} & 829 & ~\textcolor{gray}{[816, 842]} \\
quadruped-run & 459 & ~\textcolor{gray}{[412, 507]} & 118 & ~\textcolor{gray}{[98, 139]} & 262 & ~\textcolor{gray}{[184, 330]} & 328 & ~\textcolor{gray}{[255, 397]} & 498 & ~\textcolor{gray}{[476, 522]} \\
quadruped-walk & 750 & ~\textcolor{gray}{[699, 796]} & 149 & ~\textcolor{gray}{[113, 184]} & 246 & ~\textcolor{gray}{[179, 310]} & 316 & ~\textcolor{gray}{[260, 379]} & 833 & ~\textcolor{gray}{[797, 867]} \\
reacher-easy & 938 & ~\textcolor{gray}{[903, 973]} & 113 & ~\textcolor{gray}{[55, 192]} & 956 & ~\textcolor{gray}{[932, 978]} & 735 & ~\textcolor{gray}{[678, 796]} & 979 & ~\textcolor{gray}{[978, 982]} \\
reacher-hard & 705 & ~\textcolor{gray}{[580, 831]} & 10 & ~\textcolor{gray}{[0, 30]} & 911 & ~\textcolor{gray}{[867, 946]} & 338 & ~\textcolor{gray}{[227, 461]} & 965 & ~\textcolor{gray}{[945, 977]} \\
walker-run & 546 & ~\textcolor{gray}{[475, 612]} & 39 & ~\textcolor{gray}{[35, 44]} & 665 & ~\textcolor{gray}{[566, 719]} & 669 & ~\textcolor{gray}{[615, 708]} & 615 & ~\textcolor{gray}{[571, 655]} \\
walker-stand & 980 & ~\textcolor{gray}{[977, 984]} & 253 & ~\textcolor{gray}{[210, 310]} & 937 & ~\textcolor{gray}{[907, 962]} & 969 & ~\textcolor{gray}{[966, 973]} & 980 & ~\textcolor{gray}{[977, 985]} \\
walker-walk & 766 & ~\textcolor{gray}{[489, 957]} & 47 & ~\textcolor{gray}{[40, 56]} & 958 & ~\textcolor{gray}{[952, 965]} & 942 & ~\textcolor{gray}{[936, 949]} & 970 & ~\textcolor{gray}{[968, 973]} \\
\midrule
Mean & 510 & ~\textcolor{gray}{[497, 523]} & 110 & ~\textcolor{gray}{[98, 125]} & 492 & ~\textcolor{gray}{[471, 512]} & 463 & ~\textcolor{gray}{[452, 475]} & 602 & ~\textcolor{gray}{[595, 608]} \\
Median & 626 & ~\textcolor{gray}{[528, 665]} & 49 & ~\textcolor{gray}{[32, 53]} & 572 & ~\textcolor{gray}{[419, 654]} & 493 & ~\textcolor{gray}{[420, 532]} & 813 & ~\textcolor{gray}{[779, 822]} \\
IQM & 545 & ~\textcolor{gray}{[519, 564]} & 58 & ~\textcolor{gray}{[46, 67]} & 501 & ~\textcolor{gray}{[458, 537]} & 452 & ~\textcolor{gray}{[430, 473]} & 692 & ~\textcolor{gray}{[678, 703]} \\
\bottomrule
\end{tabularx}
\end{table}

\begin{figure}[ht]
\centering
\includegraphics[width=\linewidth,trim=1.5in 1.75in 1.5in 1.15in,clip]{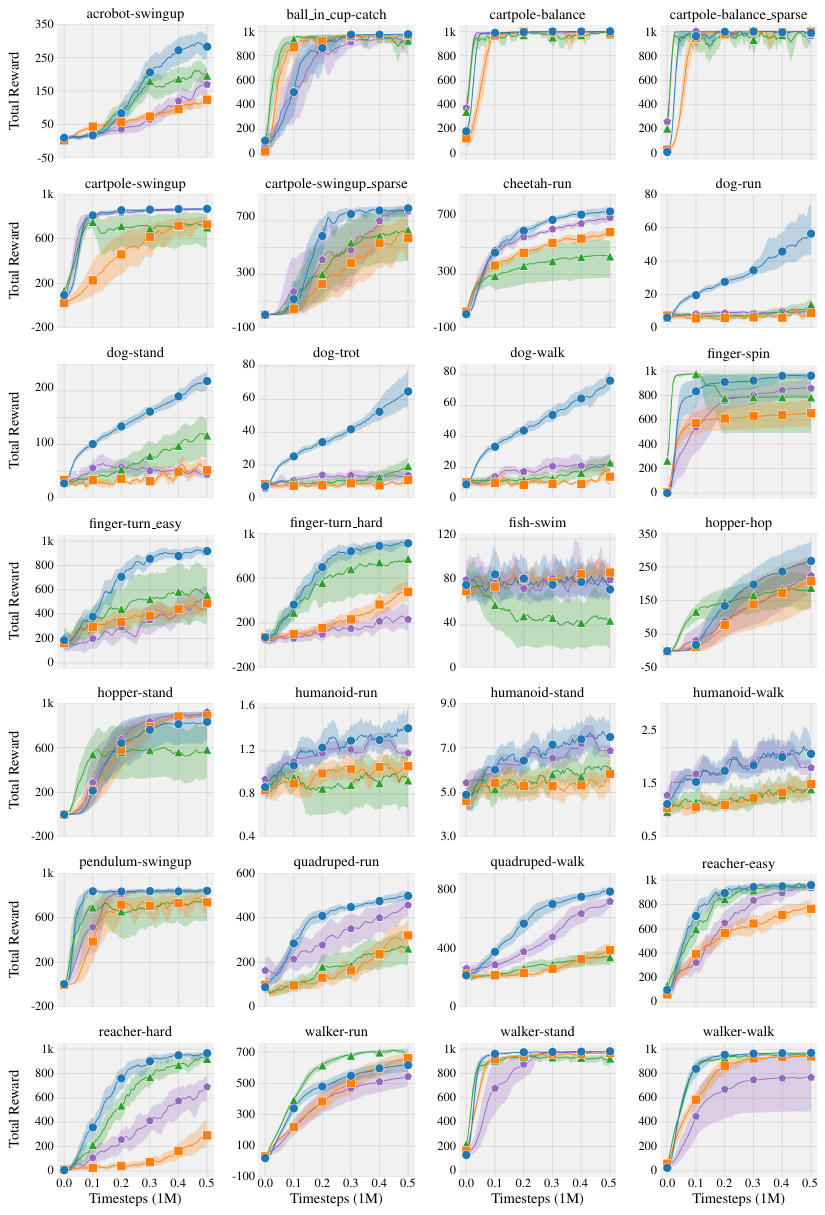} 

\fcolorbox{gray!10}{gray!10}{
\small
\coloredcircle{sb_blue}{sb_blue!25!light_gray} MR.Q \quad
\coloredsquare{sb_orange}{sb_orange!25!light_gray} DreamerV3 \quad
\coloredtriangle{sb_green}{sb_green!25!light_gray} TD-MPC2 \quad
\coloredplus{sb_red}{sb_red!25!light_gray} PPO \quad
\coloreddiamond{sb_brown}{sb_brown!25!light_gray} DrQ-v2
}

\captionof{figure}{\textbf{DMC - Visual learning curves.} 
Time steps consider the number of environment interactions, where 500k time steps equals 1M frames in the original environment. Results are over 10 seeds. \figuredescript 
} %
\end{figure}

\clearpage

\subsection{Atari}

\begin{table}[ht]
\centering
\tiny
\setlength{\tabcolsep}{0pt}
\newcolumntype{Y}{>{\centering\arraybackslash}X} %
\caption{\textbf{Atari final results.} Final average performance at 2.5M time steps (10M time steps in the original environment due to action repeat) over 10 seeds. \tabledescript human-normalized score.}
\begin{tabularx}{\textwidth}{l@{\hspace{4pt}} r@{}X@{\hspace{4pt}} r@{}X@{\hspace{4pt}} r@{}X@{\hspace{4pt}} r@{}X@{\hspace{4pt}} r@{}l@{}}
\toprule
Task & & \multicolumn{1}{l}{DQN} & & \multicolumn{1}{l}{Rainbow} & & \multicolumn{1}{l}{PPO} & & \multicolumn{1}{l}{DreamerV3} & & \multicolumn{1}{l}{MR.Q} \\
\midrule
Alien & 925 & ~\textcolor{gray}{[879, 968]} & 1220 & ~\textcolor{gray}{[1191, 1268]} & 320 & ~\textcolor{gray}{[251, 383]} & 4838 & ~\textcolor{gray}{[3863, 5813]} & 2834 & ~\textcolor{gray}{[2241, 3388]} \\
Amidar & 178 & ~\textcolor{gray}{[169, 186]} & 301 & ~\textcolor{gray}{[280, 330]} & 126 & ~\textcolor{gray}{[90, 167]} & 470 & ~\textcolor{gray}{[419, 524]} & 595 & ~\textcolor{gray}{[525, 657]} \\
Assault & 988 & ~\textcolor{gray}{[957, 1011]} & 1430 & ~\textcolor{gray}{[1392, 1475]} & 423 & ~\textcolor{gray}{[271, 581]} & 3518 & ~\textcolor{gray}{[2969, 4179]} & 1296 & ~\textcolor{gray}{[1254, 1343]} \\
Asterix & 2381 & ~\textcolor{gray}{[2313, 2469]} & 2699 & ~\textcolor{gray}{[2598, 2783]} & 296 & ~\textcolor{gray}{[216, 403]} & 7319 & ~\textcolor{gray}{[6251, 8354]} & 3358 & ~\textcolor{gray}{[3004, 3797]} \\
Asteroids & 423 & ~\textcolor{gray}{[408, 436]} & 754 & ~\textcolor{gray}{[711, 816]} & 206 & ~\textcolor{gray}{[180, 232]} & 1359 & ~\textcolor{gray}{[1243, 1482]} & 715 & ~\textcolor{gray}{[638, 796]} \\
Atlantis & 7365 & ~\textcolor{gray}{[6893, 7742]} & 80837 & ~\textcolor{gray}{[51139, 126780]} & 2000 & ~\textcolor{gray}{[2000, 2000]} & 664529 & ~\textcolor{gray}{[197588, 973362]} & 556845 & ~\textcolor{gray}{[469425, 660043]} \\
BankHeist & 474 & ~\textcolor{gray}{[448, 493]} & 895 & ~\textcolor{gray}{[889, 901]} & 187 & ~\textcolor{gray}{[41, 421]} & 801 & ~\textcolor{gray}{[691, 1002]} & 809 & ~\textcolor{gray}{[639, 960]} \\
BattleZone & 3598 & ~\textcolor{gray}{[3235, 3878]} & 20209 & ~\textcolor{gray}{[17157, 22375]} & 2200 & ~\textcolor{gray}{[1460, 3100]} & 22599 & ~\textcolor{gray}{[21055, 24669]} & 19880 & ~\textcolor{gray}{[13450, 26060]} \\
BeamRider & 869 & ~\textcolor{gray}{[728, 1065]} & 5982 & ~\textcolor{gray}{[5664, 6268]} & 479 & ~\textcolor{gray}{[348, 581]} & 5635 & ~\textcolor{gray}{[3161, 7962]} & 2299 & ~\textcolor{gray}{[1921, 2813]} \\
Berzerk & 488 & ~\textcolor{gray}{[466, 508]} & 443 & ~\textcolor{gray}{[413, 484]} & 384 & ~\textcolor{gray}{[310, 469]} & 758 & ~\textcolor{gray}{[681, 823]} & 523 & ~\textcolor{gray}{[456, 588]} \\
Bowling & 29 & ~\textcolor{gray}{[27, 32]} & 44 & ~\textcolor{gray}{[36, 52]} & 51 & ~\textcolor{gray}{[38, 60]} & 101 & ~\textcolor{gray}{[69, 138]} & 59 & ~\textcolor{gray}{[45, 72]} \\
Boxing & 37 & ~\textcolor{gray}{[31, 44]} & 68 & ~\textcolor{gray}{[66, 71]} & -3 & ~\textcolor{gray}{[-6, 0]} & 97 & ~\textcolor{gray}{[97, 99]} & 96 & ~\textcolor{gray}{[95, 97]} \\
Breakout & 21 & ~\textcolor{gray}{[19, 25]} & 41 & ~\textcolor{gray}{[40, 44]} & 9 & ~\textcolor{gray}{[8, 11]} & 137 & ~\textcolor{gray}{[110, 162]} & 34 & ~\textcolor{gray}{[28, 42]} \\
Centipede & 2832 & ~\textcolor{gray}{[2418, 3215]} & 4992 & ~\textcolor{gray}{[4784, 5138]} & 4239 & ~\textcolor{gray}{[2222, 6622]} & 20067 & ~\textcolor{gray}{[17410, 22758]} & 17835 & ~\textcolor{gray}{[16161, 19817]} \\
ChopperCommand & 997 & ~\textcolor{gray}{[971, 1022]} & 2265 & ~\textcolor{gray}{[2160, 2357]} & 688 & ~\textcolor{gray}{[501, 878]} & 15172 & ~\textcolor{gray}{[12940, 17219]} & 5748 & ~\textcolor{gray}{[4822, 6651]} \\
CrazyClimber & 64611 & ~\textcolor{gray}{[46203, 78709]} & 103539 & ~\textcolor{gray}{[99749, 106850]} & 896 & ~\textcolor{gray}{[174, 1727]} & 132811 & ~\textcolor{gray}{[128446, 135930]} & 116954 & ~\textcolor{gray}{[111371, 122032]} \\
Defender & 116954 & ~\textcolor{gray}{[111371, 122032]} & 116954 & ~\textcolor{gray}{[111371, 122032]} & 1333 & ~\textcolor{gray}{[705, 2094]} & 34187 & ~\textcolor{gray}{[29814, 39261]} & 40457 & ~\textcolor{gray}{[36892, 43638]} \\
DemonAttack & 1503 & ~\textcolor{gray}{[1282, 1690]} & 2477 & ~\textcolor{gray}{[2269, 2678]} & 139 & ~\textcolor{gray}{[116, 165]} & 4836 & ~\textcolor{gray}{[3443, 6231]} & 5924 & ~\textcolor{gray}{[4491, 7289]} \\
DoubleDunk & -18 & ~\textcolor{gray}{[-20, -18]} & -18 & ~\textcolor{gray}{[-19, -19]} & -1 & ~\textcolor{gray}{[-3, 0]} & 21 & ~\textcolor{gray}{[20, 22]} & -10 & ~\textcolor{gray}{[-15, -9]} \\
Enduro & 589 & ~\textcolor{gray}{[567, 617]} & 1601 & ~\textcolor{gray}{[1555, 1635]} & 13 & ~\textcolor{gray}{[9, 17]} & 476 & ~\textcolor{gray}{[175, 782]} & 1845 & ~\textcolor{gray}{[1758, 1938]} \\
FishingDerby & -42 & ~\textcolor{gray}{[-62, -17]} & 10 & ~\textcolor{gray}{[5, 15]} & -89 & ~\textcolor{gray}{[-91, -87]} & 40 & ~\textcolor{gray}{[32, 47]} & 10 & ~\textcolor{gray}{[2, 18]} \\
Freeway & 8 & ~\textcolor{gray}{[0, 19]} & 32 & ~\textcolor{gray}{[32, 32]} & 15 & ~\textcolor{gray}{[11, 18]} & 19 & ~\textcolor{gray}{[6, 32]} & 32 & ~\textcolor{gray}{[32, 32]} \\
Frostbite & 269 & ~\textcolor{gray}{[238, 294]} & 2510 & ~\textcolor{gray}{[2040, 2823]} & 245 & ~\textcolor{gray}{[231, 259]} & 5183 & ~\textcolor{gray}{[2151, 8291]} & 4561 & ~\textcolor{gray}{[3299, 5740]} \\
Gopher & 1470 & ~\textcolor{gray}{[1316, 1590]} & 4279 & ~\textcolor{gray}{[4139, 4425]} & 126 & ~\textcolor{gray}{[80, 174]} & 38711 & ~\textcolor{gray}{[26066, 48187]} & 19174 & ~\textcolor{gray}{[14932, 23587]} \\
Gravitar & 167 & ~\textcolor{gray}{[153, 183]} & 202 & ~\textcolor{gray}{[184, 218]} & 63 & ~\textcolor{gray}{[31, 98]} & 831 & ~\textcolor{gray}{[768, 900]} & 397 & ~\textcolor{gray}{[320, 490]} \\
Hero & 2679 & ~\textcolor{gray}{[2404, 2945]} & 9323 & ~\textcolor{gray}{[7914, 10863]} & 1741 & ~\textcolor{gray}{[1062, 2302]} & 20582 & ~\textcolor{gray}{[19845, 21583]} & 13450 & ~\textcolor{gray}{[11915, 14781]} \\
IceHockey & -9 & ~\textcolor{gray}{[-10, -9]} & -5 & ~\textcolor{gray}{[-6, -5]} & -8 & ~\textcolor{gray}{[-10, -8]} & 14 & ~\textcolor{gray}{[13, 16]} & 0 & ~\textcolor{gray}{[-1, 2]} \\
Jamesbond & 47 & ~\textcolor{gray}{[42, 52]} & 514 & ~\textcolor{gray}{[509, 520]} & 85 & ~\textcolor{gray}{[62, 106]} & 836 & ~\textcolor{gray}{[568, 1119]} & 624 & ~\textcolor{gray}{[588, 662]} \\
Kangaroo & 539 & ~\textcolor{gray}{[525, 553]} & 5501 & ~\textcolor{gray}{[3853, 7151]} & 402 & ~\textcolor{gray}{[280, 520]} & 8825 & ~\textcolor{gray}{[5234, 12418]} & 9807 & ~\textcolor{gray}{[7851, 11591]} \\
Krull & 4229 & ~\textcolor{gray}{[3942, 4490]} & 5972 & ~\textcolor{gray}{[5903, 6047]} & 421 & ~\textcolor{gray}{[136, 735]} & 23092 & ~\textcolor{gray}{[14679, 28172]} & 9309 & ~\textcolor{gray}{[8646, 9953]} \\
KungFuMaster & 15997 & ~\textcolor{gray}{[13182, 18813]} & 18074 & ~\textcolor{gray}{[16041, 20864]} & 52 & ~\textcolor{gray}{[18, 95]} & 70703 & ~\textcolor{gray}{[50114, 94578]} & 29369 & ~\textcolor{gray}{[26954, 31595]} \\
MontezumaRevenge & 0 & ~\textcolor{gray}{[0, 0]} & 0 & ~\textcolor{gray}{[0, 0]} & 0 & ~\textcolor{gray}{[0, 0]} & 1310 & ~\textcolor{gray}{[598, 2180]} & 50 & ~\textcolor{gray}{[0, 140]} \\
MsPacman & 2187 & ~\textcolor{gray}{[2121, 2247]} & 2347 & ~\textcolor{gray}{[2292, 2403]} & 457 & ~\textcolor{gray}{[352, 578]} & 4484 & ~\textcolor{gray}{[3539, 5511]} & 4922 & ~\textcolor{gray}{[4191, 5843]} \\
NameThisGame & 4000 & ~\textcolor{gray}{[3814, 4187]} & 8604 & ~\textcolor{gray}{[8252, 8931]} & 1084 & ~\textcolor{gray}{[663, 1501]} & 15742 & ~\textcolor{gray}{[14542, 17103]} & 8693 & ~\textcolor{gray}{[8071, 9199]} \\
Phoenix & 4948 & ~\textcolor{gray}{[4236, 5627]} & 4830 & ~\textcolor{gray}{[4707, 4968]} & 101 & ~\textcolor{gray}{[81, 120]} & 15827 & ~\textcolor{gray}{[14903, 16429]} & 5173 & ~\textcolor{gray}{[5025, 5322]} \\
Pitfall & -60 & ~\textcolor{gray}{[-89, -35]} & -14 & ~\textcolor{gray}{[-29, -6]} & -16 & ~\textcolor{gray}{[-38, -2]} & 0 & ~\textcolor{gray}{[0, 0]} & -20 & ~\textcolor{gray}{[-60, 0]} \\
Pong & -4 & ~\textcolor{gray}{[-14, 3]} & 15 & ~\textcolor{gray}{[14, 16]} & -5 & ~\textcolor{gray}{[-8, -3]} & 16 & ~\textcolor{gray}{[16, 17]} & 17 & ~\textcolor{gray}{[16, 19]} \\
PrivateEye & 118 & ~\textcolor{gray}{[78, 181]} & 111 & ~\textcolor{gray}{[78, 166]} & -17 & ~\textcolor{gray}{[-592, 762]} & 3046 & ~\textcolor{gray}{[975, 5118]} & 100 & ~\textcolor{gray}{[100, 100]} \\
Qbert & 1658 & ~\textcolor{gray}{[1246, 2139]} & 5353 & ~\textcolor{gray}{[4363, 6783]} & 484 & ~\textcolor{gray}{[393, 570]} & 16807 & ~\textcolor{gray}{[16073, 17564]} & 3938 & ~\textcolor{gray}{[3210, 4327]} \\
Riverraid & 3198 & ~\textcolor{gray}{[3167, 3222]} & 4272 & ~\textcolor{gray}{[4060, 4440]} & 1045 & ~\textcolor{gray}{[833, 1241]} & 9160 & ~\textcolor{gray}{[8177, 10077]} & 10791 & ~\textcolor{gray}{[9307, 12511]} \\
RoadRunner & 27980 & ~\textcolor{gray}{[27269, 28692]} & 33412 & ~\textcolor{gray}{[32459, 34435]} & 723 & ~\textcolor{gray}{[454, 940]} & 66453 & ~\textcolor{gray}{[40606, 104163]} & 49579 & ~\textcolor{gray}{[47425, 51426]} \\
Robotank & 4 & ~\textcolor{gray}{[4, 5]} & 19 & ~\textcolor{gray}{[18, 20]} & 4 & ~\textcolor{gray}{[2, 6]} & 51 & ~\textcolor{gray}{[47, 55]} & 13 & ~\textcolor{gray}{[12, 15]} \\
Seaquest & 299 & ~\textcolor{gray}{[277, 318]} & 1641 & ~\textcolor{gray}{[1621, 1661]} & 250 & ~\textcolor{gray}{[214, 282]} & 3416 & ~\textcolor{gray}{[2665, 4426]} & 3522 & ~\textcolor{gray}{[2401, 4850]} \\
Skiing & -19568 & ~\textcolor{gray}{[-19793, -19362]} & -24070 & ~\textcolor{gray}{[-25305, -22667]} & -27901 & ~\textcolor{gray}{[-30000, -23704]} & -30043 & ~\textcolor{gray}{[-30394, -29764]} & -30000 & ~\textcolor{gray}{[-30000, -30000]} \\
Solaris & 1645 & ~\textcolor{gray}{[1480, 1804]} & 1289 & ~\textcolor{gray}{[1143, 1451]} & 0 & ~\textcolor{gray}{[0, 2]} & 2340 & ~\textcolor{gray}{[1882, 2799]} & 1103 & ~\textcolor{gray}{[799, 1430]} \\
SpaceInvaders & 663 & ~\textcolor{gray}{[651, 675]} & 743 & ~\textcolor{gray}{[721, 764]} & 294 & ~\textcolor{gray}{[235, 354]} & 1433 & ~\textcolor{gray}{[1039, 1943]} & 701 & ~\textcolor{gray}{[626, 768]} \\
StarGunner & 692 & ~\textcolor{gray}{[662, 719]} & 1488 & ~\textcolor{gray}{[1470, 1506]} & 415 & ~\textcolor{gray}{[316, 499]} & 2090 & ~\textcolor{gray}{[1678, 2649]} & 3488 & ~\textcolor{gray}{[1032, 8241]} \\
Surround & 3488 & ~\textcolor{gray}{[1032, 8241]} & 3488 & ~\textcolor{gray}{[1032, 8241]} & -9 & ~\textcolor{gray}{[-10, -10]} & 5 & ~\textcolor{gray}{[4, 7]} & -2 & ~\textcolor{gray}{[-4, -2]} \\
Tennis & -21 & ~\textcolor{gray}{[-24, -19]} & -1 & ~\textcolor{gray}{[-2, 0]} & -20 & ~\textcolor{gray}{[-22, -19]} & -3 & ~\textcolor{gray}{[-11, 0]} & 0 & ~\textcolor{gray}{[0, 0]} \\
TimePilot & 1539 & ~\textcolor{gray}{[1479, 1613]} & 2703 & ~\textcolor{gray}{[2627, 2787]} & 548 & ~\textcolor{gray}{[450, 690]} & 7779 & ~\textcolor{gray}{[3128, 13016]} & 4382 & ~\textcolor{gray}{[4208, 4528]} \\
Tutankham & 112 & ~\textcolor{gray}{[97, 123]} & 179 & ~\textcolor{gray}{[165, 191]} & 29 & ~\textcolor{gray}{[17, 43]} & 253 & ~\textcolor{gray}{[240, 269]} & 164 & ~\textcolor{gray}{[145, 185]} \\
UpNDown & 7669 & ~\textcolor{gray}{[7116, 8147]} & 12397 & ~\textcolor{gray}{[11489, 13312]} & 595 & ~\textcolor{gray}{[428, 737]} & 284807 & ~\textcolor{gray}{[178615, 391388]} & 73095 & ~\textcolor{gray}{[40836, 108810]} \\
Venture & 25 & ~\textcolor{gray}{[6, 45]} & 19 & ~\textcolor{gray}{[14, 25]} & 2 & ~\textcolor{gray}{[0, 6]} & 0 & ~\textcolor{gray}{[0, 0]} & 112 & ~\textcolor{gray}{[0, 304]} \\
VideoPinball & 5129 & ~\textcolor{gray}{[4611, 5649]} & 26245 & ~\textcolor{gray}{[23075, 29067]} & 1005 & ~\textcolor{gray}{[0, 2485]} & 22345 & ~\textcolor{gray}{[20669, 23955]} & 53826 & ~\textcolor{gray}{[40600, 67972]} \\
WizardOfWor & 481 & ~\textcolor{gray}{[396, 542]} & 2213 & ~\textcolor{gray}{[1827, 2617]} & 225 & ~\textcolor{gray}{[185, 264]} & 7086 & ~\textcolor{gray}{[6518, 7730]} & 2599 & ~\textcolor{gray}{[2259, 2942]} \\
YarsRevenge & 9426 & ~\textcolor{gray}{[9177, 9656]} & 10708 & ~\textcolor{gray}{[10405, 11071]} & 1891 & ~\textcolor{gray}{[925, 2964]} & 62209 & ~\textcolor{gray}{[57783, 67113]} & 34861 & ~\textcolor{gray}{[29734, 40020]} \\
Zaxxon & 112 & ~\textcolor{gray}{[15, 230]} & 3661 & ~\textcolor{gray}{[3131, 4192]} & 0 & ~\textcolor{gray}{[0, 0]} & 17347 & ~\textcolor{gray}{[15320, 19385]} & 8850 & ~\textcolor{gray}{[8045, 9740]} \\
\midrule
Mean & 0.25 & ~\textcolor{gray}{[0.24, 0.26]} & 1.08 & ~\textcolor{gray}{[1.02, 1.14]} & -0.09 & ~\textcolor{gray}{[-0.10, -0.07]} & 3.74 & ~\textcolor{gray}{[3.29, 4.13]} & 2.54 & ~\textcolor{gray}{[2.34, 2.75]} \\
Median & 0.12 & ~\textcolor{gray}{[0.10, 0.12]} & 0.40 & ~\textcolor{gray}{[0.40, 0.47]} & 0.01 & ~\textcolor{gray}{[0.00, 0.01]} & 1.25 & ~\textcolor{gray}{[1.11, 1.47]} & 0.96 & ~\textcolor{gray}{[0.78, 0.98]} \\
IQM & 0.17 & ~\textcolor{gray}{[0.16, 0.17]} & 0.61 & ~\textcolor{gray}{[0.60, 0.62]} & 0.02 & ~\textcolor{gray}{[0.01, 0.02]} & 1.46 & ~\textcolor{gray}{[1.34, 1.51]} & 0.90 & ~\textcolor{gray}{[0.88, 0.94]} \\
\bottomrule
\end{tabularx}
\end{table}

\begin{figure}[ht]
\centering
\includegraphics[width=\linewidth,trim=1.5in 1.85in 1.5in 1.15in,clip]{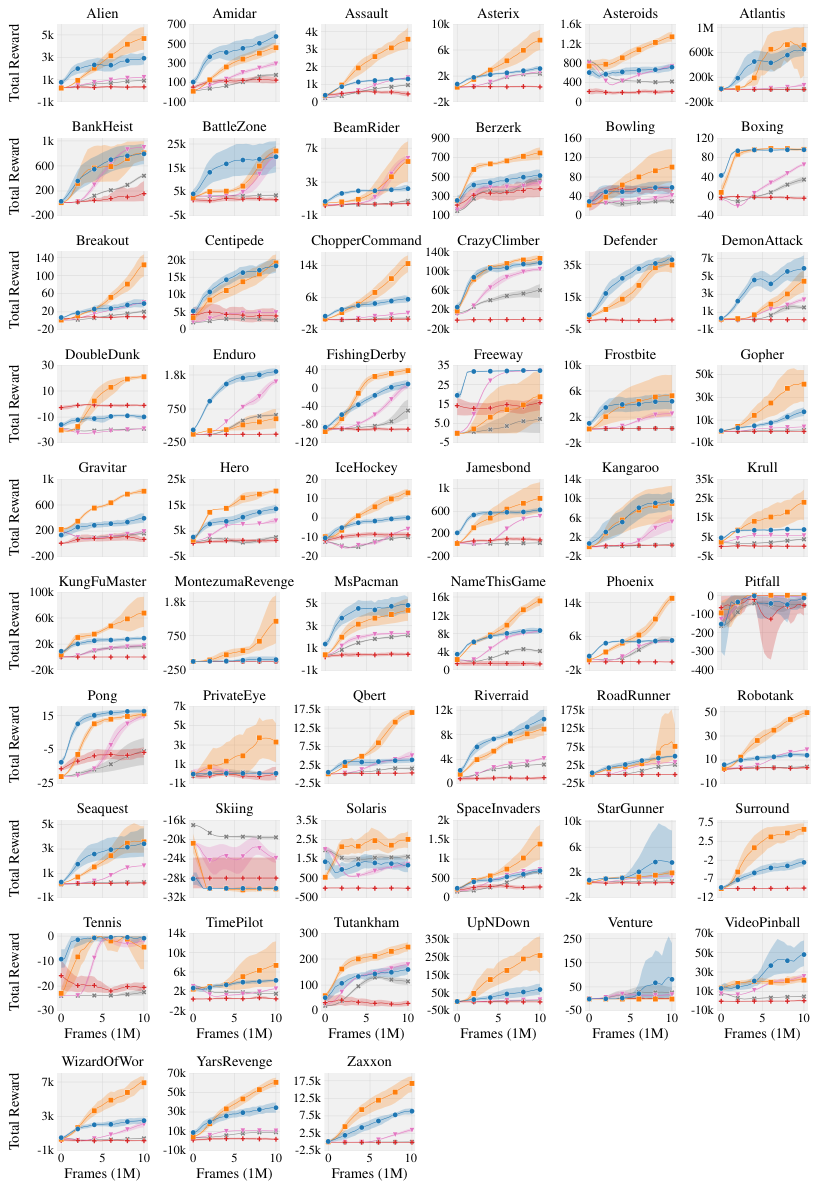} 

\fcolorbox{gray!10}{gray!10}{
\small
\coloredcircle{sb_blue}{sb_blue!25!light_gray} MR.Q \quad
\coloredsquare{sb_orange}{sb_orange!25!light_gray} DreamerV3 \quad
\coloredplus{sb_red}{sb_red!25!light_gray} PPO \quad
\coloredinvertedtriangle{sb_pink}{sb_pink!25!light_gray} Rainbow \quad
\coloredx{sb_gray}{sb_gray!25!light_gray} DQN
}

\captionof{figure}{\textbf{Atari learning curves.} Time steps consider the number of environment interactions, where 2.5M time steps equals 10M frames in the original environment. Results are over 10 seeds. \figuredescript}
\end{figure}

\clearpage

\section{Complete Ablation Results}

In this section, we show a per-environment breakdown of each variation in the design study in \cref{sec:design}. Each table reports the raw score for each environment. \tabledescript the difference in the normalized score. We use TD3 to normalize for Gym, raw scores divided by $1000$ for DMC and human scores to normalize for Atari (see \cref{appendix:envs}). 
Highlighting is used to designate the scale of the difference in normalized score:
\begin{itemize}[nosep]
    \item \hlfancy{sb_red!40}{$(\leq -0.5)$}
    \item \hlfancy{sb_red!25}{$[-0.2,-0.5)$}
    \item \hlfancy{sb_red!10}{$[-0.01,-0.2)$}
    \item \hlfancy{sb_green!10}{$[0.01,0.2)$}
    \item \hlfancy{sb_green!25}{$[0.2,0.5)$}
    \item \hlfancy{sb_green!40}{$(\geq 0.5)$}
\end{itemize}

\subsection{Gym}

\begin{figure}[ht]

\centering
\scriptsize
\setlength\tabcolsep{0pt}
\begin{tabularx}{\textwidth}{l@{\hspace{6pt}} r@{}X@{\hspace{2pt}} r@{}X@{\hspace{2pt}} r@{}X@{\hspace{2pt}} r@{}X} %
\toprule
Task & & \multicolumn{1}{l}{MR.Q} & \multicolumn{2}{c}{Linear value function} & \multicolumn{2}{c}{Dynamics target} & \multicolumn{2}{c}{No target encoder} \\
\midrule
Ant & \cellcolor{white} 6901 & \cellcolor{white}~\textcolor{gray}{[6261, 7482]} & \cellcolor{sb_red!40} 1844 & \cellcolor{sb_red!40}~\textcolor{gray}{[1663, 2018]} & \cellcolor{sb_red!25} 5867 & \cellcolor{sb_red!25}~\textcolor{gray}{[5543, 6289]} & \cellcolor{sb_red!40} 3970 & \cellcolor{sb_red!40}~\textcolor{gray}{[2468, 5509]} \\
HalfCheetah & \cellcolor{white} 12939 & \cellcolor{white}~\textcolor{gray}{[11663, 13762]} & \cellcolor{sb_red!40} 3383 & \cellcolor{sb_red!40}~\textcolor{gray}{[3054, 3732]} & \cellcolor{sb_green!10} 14019 & \cellcolor{sb_green!10}~\textcolor{gray}{[13746, 14285]} & \cellcolor{white} 12838 & \cellcolor{white}~\textcolor{gray}{[12459, 13266]} \\
Hopper & \cellcolor{white} 2692 & \cellcolor{white}~\textcolor{gray}{[2131, 3309]} & \cellcolor{sb_red!40} 968 & \cellcolor{sb_red!40}~\textcolor{gray}{[720, 1210]} & \cellcolor{sb_green!10} 2890 & \cellcolor{sb_green!10}~\textcolor{gray}{[2030, 3747]} & \cellcolor{sb_green!10} 3007 & \cellcolor{sb_green!10}~\textcolor{gray}{[2164, 3852]} \\
Humanoid & \cellcolor{white} 10223 & \cellcolor{white}~\textcolor{gray}{[9929, 10498]} & \cellcolor{sb_red!40} 461 & \cellcolor{sb_red!40}~\textcolor{gray}{[395, 532]} & \cellcolor{sb_red!25} 8370 & \cellcolor{sb_red!25}~\textcolor{gray}{[7651, 8988]} & \cellcolor{sb_red!40} 305 & \cellcolor{sb_red!40}~\textcolor{gray}{[272, 356]} \\
Walker2d & \cellcolor{white} 6039 & \cellcolor{white}~\textcolor{gray}{[5644, 6386]} & \cellcolor{sb_red!40} 1117 & \cellcolor{sb_red!40}~\textcolor{gray}{[999, 1238]} & \cellcolor{sb_red!10} 5844 & \cellcolor{sb_red!10}~\textcolor{gray}{[5146, 6477]} & \cellcolor{sb_red!10} 5944 & \cellcolor{sb_red!10}~\textcolor{gray}{[5570, 6323]} \\
\bottomrule
\end{tabularx}

\bigskip

\centering
\scriptsize
\setlength\tabcolsep{0pt}
\begin{tabularx}{\textwidth}{l@{\hspace{6pt}} r@{}X@{\hspace{2pt}} r@{}X@{\hspace{2pt}} r@{}X@{\hspace{2pt}} r@{}X} %
\toprule
Task & & \multicolumn{1}{l}{MR.Q} & \multicolumn{2}{c}{Revert} & \multicolumn{2}{c}{Non-linear model} & \multicolumn{2}{c}{MSE reward loss} \\
\midrule
Ant & \cellcolor{white} 6901 & \cellcolor{white}~\textcolor{gray}{[6261, 7482]} & \cellcolor{sb_red!40} -422 & \cellcolor{sb_red!40}~\textcolor{gray}{[-1770, 846]} & \cellcolor{sb_green!10} 7215 & \cellcolor{sb_green!10}~\textcolor{gray}{[6971, 7466]} & \cellcolor{sb_green!10} 7153 & \cellcolor{sb_green!10}~\textcolor{gray}{[5991, 7815]} \\
HalfCheetah & \cellcolor{white} 12939 & \cellcolor{white}~\textcolor{gray}{[11663, 13762]} & \cellcolor{sb_red!40} -658 & \cellcolor{sb_red!40}~\textcolor{gray}{[-750, -604]} & \cellcolor{sb_green!10} 13370 & \cellcolor{sb_green!10}~\textcolor{gray}{[12649, 14053]} & \cellcolor{sb_green!10} 14413 & \cellcolor{sb_green!10}~\textcolor{gray}{[14096, 14710]} \\
Hopper & \cellcolor{white} 2692 & \cellcolor{white}~\textcolor{gray}{[2131, 3309]} & \cellcolor{sb_red!40} 103 & \cellcolor{sb_red!40}~\textcolor{gray}{[39, 189]} & \cellcolor{sb_red!10} 2492 & \cellcolor{sb_red!10}~\textcolor{gray}{[1835, 3424]} & \cellcolor{sb_green!10} 2869 & \cellcolor{sb_green!10}~\textcolor{gray}{[2090, 3689]} \\
Humanoid & \cellcolor{white} 10223 & \cellcolor{white}~\textcolor{gray}{[9929, 10498]} & \cellcolor{sb_red!40} 189 & \cellcolor{sb_red!40}~\textcolor{gray}{[104, 277]} & \cellcolor{white} 10257 & \cellcolor{white}~\textcolor{gray}{[9612, 10688]} & \cellcolor{sb_green!10} 10592 & \cellcolor{sb_green!10}~\textcolor{gray}{[10017, 10983]} \\
Walker2d & \cellcolor{white} 6039 & \cellcolor{white}~\textcolor{gray}{[5644, 6386]} & \cellcolor{sb_red!40} 260 & \cellcolor{sb_red!40}~\textcolor{gray}{[-5, 638]} & \cellcolor{sb_red!10} 5548 & \cellcolor{sb_red!10}~\textcolor{gray}{[4980, 6117]} & \cellcolor{sb_green!10} 6626 & \cellcolor{sb_green!10}~\textcolor{gray}{[5256, 7984]} \\
\bottomrule
\end{tabularx}

\bigskip

\centering
\scriptsize
\setlength\tabcolsep{0pt}
\begin{tabularx}{\textwidth}{l@{\hspace{6pt}} r@{}X@{\hspace{2pt}} r@{}X@{\hspace{2pt}} r@{}X@{\hspace{2pt}} r@{}X} %
\toprule
Task & & \multicolumn{1}{l}{MR.Q} & \multicolumn{2}{c}{No reward scaling} & \multicolumn{2}{c}{No min} & \multicolumn{2}{c}{No LAP} \\
\midrule
Ant & \cellcolor{white} 6901 & \cellcolor{white}~\textcolor{gray}{[6261, 7482]} & \cellcolor{white} 6866 & \cellcolor{white}~\textcolor{gray}{[6227, 7547]} & \cellcolor{white} 6936 & \cellcolor{white}~\textcolor{gray}{[6582, 7329]} & \cellcolor{sb_red!10} 6817 & \cellcolor{sb_red!10}~\textcolor{gray}{[6616, 7039]} \\
HalfCheetah & \cellcolor{white} 12939 & \cellcolor{white}~\textcolor{gray}{[11663, 13762]} & \cellcolor{sb_green!10} 13502 & \cellcolor{sb_green!10}~\textcolor{gray}{[13333, 13673]} & \cellcolor{sb_green!10} 14143 & \cellcolor{sb_green!10}~\textcolor{gray}{[13819, 14515]} & \cellcolor{sb_green!10} 13185 & \cellcolor{sb_green!10}~\textcolor{gray}{[13085, 13299]} \\
Hopper & \cellcolor{white} 2692 & \cellcolor{white}~\textcolor{gray}{[2131, 3309]} & \cellcolor{sb_red!10} 2551 & \cellcolor{sb_red!10}~\textcolor{gray}{[2090, 3064]} & \cellcolor{sb_red!10} 2113 & \cellcolor{sb_red!10}~\textcolor{gray}{[1728, 2626]} & \cellcolor{white} 2681 & \cellcolor{white}~\textcolor{gray}{[1883, 3465]} \\
Humanoid & \cellcolor{white} 10223 & \cellcolor{white}~\textcolor{gray}{[9929, 10498]} & \cellcolor{sb_red!10} 9515 & \cellcolor{sb_red!10}~\textcolor{gray}{[8520, 10245]} & \cellcolor{sb_green!10} 10528 & \cellcolor{sb_green!10}~\textcolor{gray}{[10202, 10837]} & \cellcolor{sb_red!25} 8441 & \cellcolor{sb_red!25}~\textcolor{gray}{[6206, 9738]} \\
Walker2d & \cellcolor{white} 6039 & \cellcolor{white}~\textcolor{gray}{[5644, 6386]} & \cellcolor{sb_red!10} 5743 & \cellcolor{sb_red!10}~\textcolor{gray}{[5362, 6102]} & \cellcolor{sb_red!25} 4293 & \cellcolor{sb_red!25}~\textcolor{gray}{[3547, 5107]} & \cellcolor{sb_red!10} 5463 & \cellcolor{sb_red!10}~\textcolor{gray}{[4134, 6376]} \\
\bottomrule
\end{tabularx}

\bigskip

\centering
\scriptsize
\setlength\tabcolsep{0pt}
\begin{tabularx}{\textwidth}{l@{\hspace{6pt}} r@{}X@{\hspace{2pt}} r@{}X@{\hspace{2pt}} r@{}X@{\hspace{2pt}} r@{}X} %
\toprule
Task & & \multicolumn{1}{l}{MR.Q} & \multicolumn{2}{c}{No MR} & \multicolumn{2}{c}{1-step return} & \multicolumn{2}{c}{No unroll} \\
\midrule
Ant & \cellcolor{white} 6901 & \cellcolor{white}~\textcolor{gray}{[6261, 7482]} & \cellcolor{sb_red!40} 4195 & \cellcolor{sb_red!40}~\textcolor{gray}{[2573, 5819]} & \cellcolor{sb_green!25} 7757 & \cellcolor{sb_green!25}~\textcolor{gray}{[7729, 7799]} & \cellcolor{sb_green!10} 7528 & \cellcolor{sb_green!10}~\textcolor{gray}{[7224, 7830]} \\
HalfCheetah & \cellcolor{white} 12939 & \cellcolor{white}~\textcolor{gray}{[11663, 13762]} & \cellcolor{sb_red!10} 11249 & \cellcolor{sb_red!10}~\textcolor{gray}{[9238, 12495]} & \cellcolor{sb_green!10} 13123 & \cellcolor{sb_green!10}~\textcolor{gray}{[10691, 14653]} & \cellcolor{sb_green!10} 14409 & \cellcolor{sb_green!10}~\textcolor{gray}{[13817, 15002]} \\
Hopper & \cellcolor{white} 2692 & \cellcolor{white}~\textcolor{gray}{[2131, 3309]} & \cellcolor{sb_red!25} 1877 & \cellcolor{sb_red!25}~\textcolor{gray}{[1524, 2153]} & \cellcolor{sb_green!10} 2737 & \cellcolor{sb_green!10}~\textcolor{gray}{[2131, 3343]} & \cellcolor{sb_red!10} 2578 & \cellcolor{sb_red!10}~\textcolor{gray}{[1857, 3414]} \\
Humanoid & \cellcolor{white} 10223 & \cellcolor{white}~\textcolor{gray}{[9929, 10498]} & \cellcolor{sb_red!40} 3942 & \cellcolor{sb_red!40}~\textcolor{gray}{[3262, 4624]} & \cellcolor{sb_red!40} 2328 & \cellcolor{sb_red!40}~\textcolor{gray}{[1491, 3337]} & \cellcolor{sb_green!10} 10617 & \cellcolor{sb_green!10}~\textcolor{gray}{[10504, 10731]} \\
Walker2d & \cellcolor{white} 6039 & \cellcolor{white}~\textcolor{gray}{[5644, 6386]} & \cellcolor{sb_red!25} 4155 & \cellcolor{sb_red!25}~\textcolor{gray}{[3251, 4897]} & \cellcolor{sb_red!25} 4747 & \cellcolor{sb_red!25}~\textcolor{gray}{[3197, 6229]} & \cellcolor{white} 6077 & \cellcolor{white}~\textcolor{gray}{[5752, 6355]} \\
\bottomrule
\end{tabularx}

\end{figure}

\clearpage

\subsection{DMC - Proprioceptive}

\begin{figure}[ht]

\centering
\scriptsize
\setlength\tabcolsep{0pt}
\begin{tabularx}{\textwidth}{l@{\hspace{6pt}} r@{}X@{\hspace{2pt}} r@{}X@{\hspace{2pt}} r@{}X@{\hspace{2pt}} r@{}X} %
\toprule
Task & & \multicolumn{1}{l}{MR.Q} & \multicolumn{2}{c}{Linear value function} & \multicolumn{2}{c}{Dynamics target} & \multicolumn{2}{c}{No target encoder} \\
\midrule
acrobot-swingup & \cellcolor{white} 567 & \cellcolor{white}~\textcolor{gray}{[517, 621]} & \cellcolor{sb_red!40} 30 & \cellcolor{sb_red!40}~\textcolor{gray}{[15, 46]} & \cellcolor{sb_green!10} 626 & \cellcolor{sb_green!10}~\textcolor{gray}{[578, 684]} & \cellcolor{sb_red!40} 16 & \cellcolor{sb_red!40}~\textcolor{gray}{[9, 25]} \\
ball\_in\_cup-catch & \cellcolor{white} 981 & \cellcolor{white}~\textcolor{gray}{[979, 983]} & \cellcolor{sb_red!10} 820 & \cellcolor{sb_red!10}~\textcolor{gray}{[658, 922]} & \cellcolor{white} 980 & \cellcolor{white}~\textcolor{gray}{[978, 983]} & \cellcolor{sb_red!25} 569 & \cellcolor{sb_red!25}~\textcolor{gray}{[436, 719]} \\
cartpole-balance & \cellcolor{white} 999 & \cellcolor{white}~\textcolor{gray}{[999, 1000]} & \cellcolor{sb_red!40} 449 & \cellcolor{sb_red!40}~\textcolor{gray}{[380, 520]} & \cellcolor{white} 999 & \cellcolor{white}~\textcolor{gray}{[999, 1000]} & \cellcolor{white} 992 & \cellcolor{white}~\textcolor{gray}{[986, 997]} \\
cartpole-balance\_sparse & \cellcolor{white} 1000 & \cellcolor{white}~\textcolor{gray}{[1000, 1000]} & \cellcolor{sb_red!40} 183 & \cellcolor{sb_red!40}~\textcolor{gray}{[142, 224]} & \cellcolor{white} 1000 & \cellcolor{white}~\textcolor{gray}{[1000, 1000]} & \cellcolor{white} 1000 & \cellcolor{white}~\textcolor{gray}{[1000, 1000]} \\
cartpole-swingup & \cellcolor{white} 866 & \cellcolor{white}~\textcolor{gray}{[866, 866]} & \cellcolor{sb_red!40} 267 & \cellcolor{sb_red!40}~\textcolor{gray}{[225, 310]} & \cellcolor{white} 869 & \cellcolor{white}~\textcolor{gray}{[866, 876]} & \cellcolor{sb_red!10} 852 & \cellcolor{sb_red!10}~\textcolor{gray}{[840, 861]} \\
\bottomrule
\end{tabularx}

\bigskip

\centering
\scriptsize
\setlength\tabcolsep{0pt}
\begin{tabularx}{\textwidth}{l@{\hspace{6pt}} r@{}X@{\hspace{2pt}} r@{}X@{\hspace{2pt}} r@{}X@{\hspace{2pt}} r@{}X} %
\toprule
Task & & \multicolumn{1}{l}{MR.Q} & \multicolumn{2}{c}{Revert} & \multicolumn{2}{c}{Non-linear model} & \multicolumn{2}{c}{MSE reward loss} \\
\midrule
acrobot-swingup & \cellcolor{white} 567 & \cellcolor{white}~\textcolor{gray}{[517, 621]} & \cellcolor{sb_red!40} 11 & \cellcolor{sb_red!40}~\textcolor{gray}{[8, 18]} & \cellcolor{sb_red!10} 553 & \cellcolor{sb_red!10}~\textcolor{gray}{[478, 629]} & \cellcolor{sb_green!10} 577 & \cellcolor{sb_green!10}~\textcolor{gray}{[547, 612]} \\
ball\_in\_cup-catch & \cellcolor{white} 981 & \cellcolor{white}~\textcolor{gray}{[979, 983]} & \cellcolor{sb_red!40} 301 & \cellcolor{sb_red!40}~\textcolor{gray}{[233, 365]} & \cellcolor{white} 982 & \cellcolor{white}~\textcolor{gray}{[982, 984]} & \cellcolor{white} 983 & \cellcolor{white}~\textcolor{gray}{[982, 985]} \\
cartpole-balance & \cellcolor{white} 999 & \cellcolor{white}~\textcolor{gray}{[999, 1000]} & \cellcolor{sb_red!40} 272 & \cellcolor{sb_red!40}~\textcolor{gray}{[206, 332]} & \cellcolor{white} 999 & \cellcolor{white}~\textcolor{gray}{[999, 1000]} & \cellcolor{white} 999 & \cellcolor{white}~\textcolor{gray}{[998, 1000]} \\
cartpole-balance\_sparse & \cellcolor{white} 1000 & \cellcolor{white}~\textcolor{gray}{[1000, 1000]} & \cellcolor{sb_red!40} 197 & \cellcolor{sb_red!40}~\textcolor{gray}{[177, 214]} & \cellcolor{white} 1000 & \cellcolor{white}~\textcolor{gray}{[1000, 1000]} & \cellcolor{white} 1000 & \cellcolor{white}~\textcolor{gray}{[1000, 1000]} \\
cartpole-swingup & \cellcolor{white} 866 & \cellcolor{white}~\textcolor{gray}{[866, 866]} & \cellcolor{sb_red!40} 191 & \cellcolor{sb_red!40}~\textcolor{gray}{[99, 263]} & \cellcolor{white} 866 & \cellcolor{white}~\textcolor{gray}{[866, 867]} & \cellcolor{white} 865 & \cellcolor{white}~\textcolor{gray}{[865, 866]} \\
\bottomrule
\end{tabularx}

\end{figure}

\clearpage

\begin{figure}

\centering
\scriptsize
\setlength\tabcolsep{0pt}
\begin{tabularx}{\textwidth}{l@{\hspace{6pt}} r@{}X@{\hspace{2pt}} r@{}X@{\hspace{2pt}} r@{}X@{\hspace{2pt}} r@{}X} %
\toprule
Task & & \multicolumn{1}{l}{MR.Q} & \multicolumn{2}{c}{No reward scaling} & \multicolumn{2}{c}{No min} & \multicolumn{2}{c}{No LAP} \\
\midrule
acrobot-swingup & \cellcolor{white} 567 & \cellcolor{white}~\textcolor{gray}{[517, 621]} & \cellcolor{sb_green!10} 593 & \cellcolor{sb_green!10}~\textcolor{gray}{[532, 672]} & \cellcolor{sb_green!10} 623 & \cellcolor{sb_green!10}~\textcolor{gray}{[572, 673]} & \cellcolor{white} 566 & \cellcolor{white}~\textcolor{gray}{[520, 612]} \\
ball\_in\_cup-catch & \cellcolor{white} 981 & \cellcolor{white}~\textcolor{gray}{[979, 983]} & \cellcolor{white} 983 & \cellcolor{white}~\textcolor{gray}{[982, 984]} & \cellcolor{white} 982 & \cellcolor{white}~\textcolor{gray}{[980, 984]} & \cellcolor{white} 981 & \cellcolor{white}~\textcolor{gray}{[980, 984]} \\
cartpole-balance & \cellcolor{white} 999 & \cellcolor{white}~\textcolor{gray}{[999, 1000]} & \cellcolor{white} 998 & \cellcolor{white}~\textcolor{gray}{[999, 999]} & \cellcolor{white} 999 & \cellcolor{white}~\textcolor{gray}{[1000, 1000]} & \cellcolor{white} 999 & \cellcolor{white}~\textcolor{gray}{[998, 1000]} \\
cartpole-balance\_sparse & \cellcolor{white} 1000 & \cellcolor{white}~\textcolor{gray}{[1000, 1000]} & \cellcolor{white} 1000 & \cellcolor{white}~\textcolor{gray}{[1000, 1000]} & \cellcolor{white} 1000 & \cellcolor{white}~\textcolor{gray}{[1000, 1000]} & \cellcolor{white} 992 & \cellcolor{white}~\textcolor{gray}{[982, 1000]} \\
cartpole-swingup & \cellcolor{white} 866 & \cellcolor{white}~\textcolor{gray}{[866, 866]} & \cellcolor{white} 865 & \cellcolor{white}~\textcolor{gray}{[864, 866]} & \cellcolor{white} 868 & \cellcolor{white}~\textcolor{gray}{[866, 874]} & \cellcolor{white} 865 & \cellcolor{white}~\textcolor{gray}{[865, 866]} \\
\bottomrule
\end{tabularx}

\bigskip

\centering
\scriptsize
\setlength\tabcolsep{0pt}
\begin{tabularx}{\textwidth}{l@{\hspace{6pt}} r@{}X@{\hspace{2pt}} r@{}X@{\hspace{2pt}} r@{}X@{\hspace{2pt}} r@{}X} %
\toprule
Task & & \multicolumn{1}{l}{MR.Q} & \multicolumn{2}{c}{No MR} & \multicolumn{2}{c}{1-step return} & \multicolumn{2}{c}{No unroll} \\
\midrule
acrobot-swingup & \cellcolor{white} 567 & \cellcolor{white}~\textcolor{gray}{[517, 621]} & \cellcolor{white} 576 & \cellcolor{white}~\textcolor{gray}{[483, 665]} & \cellcolor{sb_red!10} 440 & \cellcolor{sb_red!10}~\textcolor{gray}{[360, 528]} & \cellcolor{sb_red!10} 515 & \cellcolor{sb_red!10}~\textcolor{gray}{[455, 598]} \\
ball\_in\_cup-catch & \cellcolor{white} 981 & \cellcolor{white}~\textcolor{gray}{[979, 983]} & \cellcolor{white} 981 & \cellcolor{white}~\textcolor{gray}{[980, 984]} & \cellcolor{white} 984 & \cellcolor{white}~\textcolor{gray}{[983, 985]} & \cellcolor{white} 982 & \cellcolor{white}~\textcolor{gray}{[981, 984]} \\
cartpole-balance & \cellcolor{white} 999 & \cellcolor{white}~\textcolor{gray}{[999, 1000]} & \cellcolor{white} 994 & \cellcolor{white}~\textcolor{gray}{[991, 999]} & \cellcolor{white} 999 & \cellcolor{white}~\textcolor{gray}{[1000, 1000]} & \cellcolor{white} 999 & \cellcolor{white}~\textcolor{gray}{[999, 1000]} \\
cartpole-balance\_sparse & \cellcolor{white} 1000 & \cellcolor{white}~\textcolor{gray}{[1000, 1000]} & \cellcolor{white} 1000 & \cellcolor{white}~\textcolor{gray}{[1000, 1000]} & \cellcolor{sb_red!10} 961 & \cellcolor{sb_red!10}~\textcolor{gray}{[886, 1000]} & \cellcolor{white} 1000 & \cellcolor{white}~\textcolor{gray}{[1000, 1000]} \\
cartpole-swingup & \cellcolor{white} 866 & \cellcolor{white}~\textcolor{gray}{[866, 866]} & \cellcolor{white} 870 & \cellcolor{white}~\textcolor{gray}{[864, 878]} & \cellcolor{sb_green!10} 881 & \cellcolor{sb_green!10}~\textcolor{gray}{[879, 882]} & \cellcolor{white} 864 & \cellcolor{white}~\textcolor{gray}{[861, 867]} \\
\bottomrule
\end{tabularx}

\end{figure}

\clearpage

\subsection{DMC - Visual}

\begin{figure}[ht]

\centering
\scriptsize
\setlength\tabcolsep{0pt}
\begin{tabularx}{\textwidth}{l@{\hspace{6pt}} r@{}X@{\hspace{2pt}} r@{}X@{\hspace{2pt}} r@{}X@{\hspace{2pt}} r@{}X} %
\toprule
Task & & \multicolumn{1}{l}{MR.Q} & \multicolumn{2}{c}{Linear value function} & \multicolumn{2}{c}{Dynamics target} & \multicolumn{2}{c}{No target encoder} \\
\midrule
acrobot-swingup & \cellcolor{white} 287 & \cellcolor{white}~\textcolor{gray}{[253, 317]} & \cellcolor{sb_red!25} 15 & \cellcolor{sb_red!25}~\textcolor{gray}{[5, 22]} & \cellcolor{white} 296 & \cellcolor{white}~\textcolor{gray}{[281, 323]} & \cellcolor{sb_red!25} 16 & \cellcolor{sb_red!25}~\textcolor{gray}{[0, 38]} \\
ball\_in\_cup-catch & \cellcolor{white} 977 & \cellcolor{white}~\textcolor{gray}{[975, 980]} & \cellcolor{sb_red!25} 644 & \cellcolor{sb_red!25}~\textcolor{gray}{[328, 904]} & \cellcolor{white} 972 & \cellcolor{white}~\textcolor{gray}{[965, 976]} & \cellcolor{sb_red!25} 605 & \cellcolor{sb_red!25}~\textcolor{gray}{[496, 726]} \\
cartpole-balance & \cellcolor{white} 999 & \cellcolor{white}~\textcolor{gray}{[999, 999]} & \cellcolor{sb_red!40} 306 & \cellcolor{sb_red!40}~\textcolor{gray}{[254, 349]} & \cellcolor{white} 998 & \cellcolor{white}~\textcolor{gray}{[998, 999]} & \cellcolor{sb_red!10} 978 & \cellcolor{sb_red!10}~\textcolor{gray}{[947, 997]} \\
cartpole-balance\_sparse & \cellcolor{white} 1000 & \cellcolor{white}~\textcolor{gray}{[1000, 1000]} & \cellcolor{sb_red!40} 243 & \cellcolor{sb_red!40}~\textcolor{gray}{[183, 327]} & \cellcolor{white} 1000 & \cellcolor{white}~\textcolor{gray}{[1000, 1000]} & \cellcolor{white} 1000 & \cellcolor{white}~\textcolor{gray}{[1000, 1000]} \\
cartpole-swingup & \cellcolor{white} 868 & \cellcolor{white}~\textcolor{gray}{[861, 875]} & \cellcolor{sb_red!40} 229 & \cellcolor{sb_red!40}~\textcolor{gray}{[181, 294]} & \cellcolor{white} 861 & \cellcolor{white}~\textcolor{gray}{[859, 865]} & \cellcolor{sb_red!10} 689 & \cellcolor{sb_red!10}~\textcolor{gray}{[487, 808]} \\
\bottomrule
\end{tabularx}

\bigskip

\centering
\scriptsize
\setlength\tabcolsep{0pt}
\begin{tabularx}{\textwidth}{l@{\hspace{6pt}} r@{}X@{\hspace{2pt}} r@{}X@{\hspace{2pt}} r@{}X@{\hspace{2pt}} r@{}X} %
\toprule
Task & & \multicolumn{1}{l}{MR.Q} & \multicolumn{2}{c}{Revert} & \multicolumn{2}{c}{Non-linear model} & \multicolumn{2}{c}{MSE reward loss} \\
\midrule
acrobot-swingup & \cellcolor{white} 287 & \cellcolor{white}~\textcolor{gray}{[253, 317]} & \cellcolor{sb_red!25} 19 & \cellcolor{sb_red!25}~\textcolor{gray}{[13, 23]} & \cellcolor{white} 279 & \cellcolor{white}~\textcolor{gray}{[235, 314]} & \cellcolor{sb_red!10} 265 & \cellcolor{sb_red!10}~\textcolor{gray}{[242, 294]} \\
ball\_in\_cup-catch & \cellcolor{white} 977 & \cellcolor{white}~\textcolor{gray}{[975, 980]} & \cellcolor{sb_red!40} 195 & \cellcolor{sb_red!40}~\textcolor{gray}{[91, 297]} & \cellcolor{white} 973 & \cellcolor{white}~\textcolor{gray}{[970, 977]} & \cellcolor{white} 974 & \cellcolor{white}~\textcolor{gray}{[969, 978]} \\
cartpole-balance & \cellcolor{white} 999 & \cellcolor{white}~\textcolor{gray}{[999, 999]} & \cellcolor{sb_red!40} 190 & \cellcolor{sb_red!40}~\textcolor{gray}{[163, 212]} & \cellcolor{white} 999 & \cellcolor{white}~\textcolor{gray}{[999, 999]} & \cellcolor{white} 998 & \cellcolor{white}~\textcolor{gray}{[998, 1000]} \\
cartpole-balance\_sparse & \cellcolor{white} 1000 & \cellcolor{white}~\textcolor{gray}{[1000, 1000]} & \cellcolor{sb_red!40} 346 & \cellcolor{sb_red!40}~\textcolor{gray}{[193, 639]} & \cellcolor{white} 1000 & \cellcolor{white}~\textcolor{gray}{[1000, 1000]} & \cellcolor{white} 1000 & \cellcolor{white}~\textcolor{gray}{[1000, 1000]} \\
cartpole-swingup & \cellcolor{white} 868 & \cellcolor{white}~\textcolor{gray}{[861, 875]} & \cellcolor{sb_red!40} 115 & \cellcolor{sb_red!40}~\textcolor{gray}{[82, 175]} & \cellcolor{sb_red!10} 849 & \cellcolor{sb_red!10}~\textcolor{gray}{[819, 873]} & \cellcolor{white} 876 & \cellcolor{white}~\textcolor{gray}{[875, 878]} \\
\bottomrule
\end{tabularx}

\end{figure}

\clearpage

\begin{figure}

\centering
\scriptsize
\setlength\tabcolsep{0pt}
\begin{tabularx}{\textwidth}{l@{\hspace{6pt}} r@{}X@{\hspace{2pt}} r@{}X@{\hspace{2pt}} r@{}X@{\hspace{2pt}} r@{}X} %
\toprule
Task & & \multicolumn{1}{l}{MR.Q} & \multicolumn{2}{c}{No reward scaling} & \multicolumn{2}{c}{No min} & \multicolumn{2}{c}{No LAP} \\
\midrule
acrobot-swingup & \cellcolor{white} 287 & \cellcolor{white}~\textcolor{gray}{[253, 317]} & \cellcolor{sb_green!10} 323 & \cellcolor{sb_green!10}~\textcolor{gray}{[275, 368]} & \cellcolor{sb_green!10} 332 & \cellcolor{sb_green!10}~\textcolor{gray}{[301, 391]} & \cellcolor{white} 280 & \cellcolor{white}~\textcolor{gray}{[235, 354]} \\
ball\_in\_cup-catch & \cellcolor{white} 977 & \cellcolor{white}~\textcolor{gray}{[975, 980]} & \cellcolor{white} 973 & \cellcolor{white}~\textcolor{gray}{[971, 977]} & \cellcolor{white} 975 & \cellcolor{white}~\textcolor{gray}{[974, 976]} & \cellcolor{white} 971 & \cellcolor{white}~\textcolor{gray}{[966, 978]} \\
cartpole-balance & \cellcolor{white} 999 & \cellcolor{white}~\textcolor{gray}{[999, 999]} & \cellcolor{white} 999 & \cellcolor{white}~\textcolor{gray}{[999, 999]} & \cellcolor{white} 998 & \cellcolor{white}~\textcolor{gray}{[998, 999]} & \cellcolor{white} 998 & \cellcolor{white}~\textcolor{gray}{[997, 999]} \\
cartpole-balance\_sparse & \cellcolor{white} 1000 & \cellcolor{white}~\textcolor{gray}{[1000, 1000]} & \cellcolor{white} 1000 & \cellcolor{white}~\textcolor{gray}{[1000, 1000]} & \cellcolor{white} 1000 & \cellcolor{white}~\textcolor{gray}{[1000, 1000]} & \cellcolor{white} 1000 & \cellcolor{white}~\textcolor{gray}{[1000, 1000]} \\
cartpole-swingup & \cellcolor{white} 868 & \cellcolor{white}~\textcolor{gray}{[861, 875]} & \cellcolor{white} 860 & \cellcolor{white}~\textcolor{gray}{[829, 877]} & \cellcolor{sb_green!10} 879 & \cellcolor{sb_green!10}~\textcolor{gray}{[879, 880]} & \cellcolor{sb_red!10} 798 & \cellcolor{sb_red!10}~\textcolor{gray}{[785, 821]} \\
\bottomrule
\end{tabularx}

\bigskip

\centering
\scriptsize
\setlength\tabcolsep{0pt}
\begin{tabularx}{\textwidth}{l@{\hspace{6pt}} r@{}X@{\hspace{2pt}} r@{}X@{\hspace{2pt}} r@{}X@{\hspace{2pt}} r@{}X} %
\toprule
Task & & \multicolumn{1}{l}{MR.Q} & \multicolumn{2}{c}{No MR} & \multicolumn{2}{c}{1-step return} & \multicolumn{2}{c}{No unroll} \\
\midrule
acrobot-swingup & \cellcolor{white} 287 & \cellcolor{white}~\textcolor{gray}{[253, 317]} & \cellcolor{sb_green!10} 362 & \cellcolor{sb_green!10}~\textcolor{gray}{[305, 421]} & \cellcolor{sb_red!10} 91 & \cellcolor{sb_red!10}~\textcolor{gray}{[76, 112]} & \cellcolor{sb_red!10} 126 & \cellcolor{sb_red!10}~\textcolor{gray}{[77, 159]} \\
ball\_in\_cup-catch & \cellcolor{white} 977 & \cellcolor{white}~\textcolor{gray}{[975, 980]} & \cellcolor{sb_red!10} 898 & \cellcolor{sb_red!10}~\textcolor{gray}{[746, 977]} & \cellcolor{white} 980 & \cellcolor{white}~\textcolor{gray}{[979, 982]} & \cellcolor{white} 976 & \cellcolor{white}~\textcolor{gray}{[973, 980]} \\
cartpole-balance & \cellcolor{white} 999 & \cellcolor{white}~\textcolor{gray}{[999, 999]} & \cellcolor{white} 998 & \cellcolor{white}~\textcolor{gray}{[998, 999]} & \cellcolor{white} 998 & \cellcolor{white}~\textcolor{gray}{[998, 1000]} & \cellcolor{white} 998 & \cellcolor{white}~\textcolor{gray}{[999, 999]} \\
cartpole-balance\_sparse & \cellcolor{white} 1000 & \cellcolor{white}~\textcolor{gray}{[1000, 1000]} & \cellcolor{white} 1000 & \cellcolor{white}~\textcolor{gray}{[1000, 1000]} & \cellcolor{white} 1000 & \cellcolor{white}~\textcolor{gray}{[1000, 1000]} & \cellcolor{white} 1000 & \cellcolor{white}~\textcolor{gray}{[1000, 1000]} \\
cartpole-swingup & \cellcolor{white} 868 & \cellcolor{white}~\textcolor{gray}{[861, 875]} & \cellcolor{white} 871 & \cellcolor{white}~\textcolor{gray}{[864, 878]} & \cellcolor{white} 858 & \cellcolor{white}~\textcolor{gray}{[835, 875]} & \cellcolor{white} 872 & \cellcolor{white}~\textcolor{gray}{[863, 879]} \\
\bottomrule
\end{tabularx}

\end{figure}

\clearpage

\subsection{Atari}

\begin{figure}[ht]

\centering
\scriptsize
\setlength\tabcolsep{0pt}
\begin{tabularx}{\textwidth}{l@{\hspace{6pt}} r@{}X@{\hspace{2pt}} r@{}X@{\hspace{2pt}} r@{}X@{\hspace{2pt}} r@{}X} %
\toprule
Task & \multicolumn{2}{c}{MR.Q} & \multicolumn{2}{c}{Linear value function} & \multicolumn{2}{c}{Dynamics target} & \multicolumn{2}{c}{No target encoder} \\
\midrule
Alien & \cellcolor{white} 2471 & \cellcolor{white}~\textcolor{gray}{[1848, 3155]} & \cellcolor{sb_red!25} 596 & \cellcolor{sb_red!25}~\textcolor{gray}{[561, 631]} & \cellcolor{sb_red!10} 1176 & \cellcolor{sb_red!10}~\textcolor{gray}{[1138, 1215]} & \cellcolor{sb_red!10} 2040 & \cellcolor{sb_red!10}~\textcolor{gray}{[1585, 2495]} \\
Amidar & \cellcolor{white} 443 & \cellcolor{white}~\textcolor{gray}{[376, 499]} & \cellcolor{sb_red!25} 48 & \cellcolor{sb_red!25}~\textcolor{gray}{[38, 61]} & \cellcolor{sb_red!10} 214 & \cellcolor{sb_red!10}~\textcolor{gray}{[182, 247]} & \cellcolor{sb_red!10} 249 & \cellcolor{sb_red!10}~\textcolor{gray}{[232, 268]} \\
Assault & \cellcolor{white} 1125 & \cellcolor{white}~\textcolor{gray}{[1094, 1160]} & \cellcolor{sb_red!40} 366 & \cellcolor{sb_red!40}~\textcolor{gray}{[359, 374]} & \cellcolor{sb_red!25} 911 & \cellcolor{sb_red!25}~\textcolor{gray}{[906, 917]} & \cellcolor{sb_red!10} 1057 & \cellcolor{sb_red!10}~\textcolor{gray}{[880, 1234]} \\
Asterix & \cellcolor{white} 2216 & \cellcolor{white}~\textcolor{gray}{[2081, 2346]} & \cellcolor{sb_red!10} 810 & \cellcolor{sb_red!10}~\textcolor{gray}{[585, 1035]} & \cellcolor{sb_red!10} 1940 & \cellcolor{sb_red!10}~\textcolor{gray}{[1865, 2015]} & \cellcolor{sb_green!10} 2407 & \cellcolor{sb_green!10}~\textcolor{gray}{[1860, 2955]} \\
Asteroids & \cellcolor{white} 602 & \cellcolor{white}~\textcolor{gray}{[493, 689]} & \cellcolor{white} 609 & \cellcolor{white}~\textcolor{gray}{[595, 623]} & \cellcolor{white} 776 & \cellcolor{white}~\textcolor{gray}{[716, 837]} & \cellcolor{white} 765 & \cellcolor{white}~\textcolor{gray}{[591, 939]} \\
\bottomrule
\end{tabularx}
\end{figure}

\begin{figure}[ht]

\centering
\scriptsize
\setlength\tabcolsep{0pt}
\begin{tabularx}{\textwidth}{l@{\hspace{6pt}} r@{}X@{\hspace{2pt}} r@{}X@{\hspace{2pt}} r@{}X@{\hspace{2pt}} r@{}X} %
\toprule
Task & \multicolumn{2}{c}{MR.Q} & \multicolumn{2}{c}{Revert} & \multicolumn{2}{c}{Non-linear model} & \multicolumn{2}{c}{MSE reward loss} \\
\midrule
Alien & \cellcolor{white} 2471 & \cellcolor{white}~\textcolor{gray}{[1848, 3155]} & \cellcolor{sb_red!25} 66 & \cellcolor{sb_red!25}~\textcolor{gray}{[32, 100]} & \cellcolor{sb_red!10} 2167 & \cellcolor{sb_red!10}~\textcolor{gray}{[1426, 3169]} & \cellcolor{sb_red!25} 734 & \cellcolor{sb_red!25}~\textcolor{gray}{[617, 856]} \\
Amidar & \cellcolor{white} 443 & \cellcolor{white}~\textcolor{gray}{[376, 499]} & \cellcolor{sb_red!25} 39 & \cellcolor{sb_red!25}~\textcolor{gray}{[19, 60]} & \cellcolor{sb_green!10} 466 & \cellcolor{sb_green!10}~\textcolor{gray}{[364, 570]} & \cellcolor{sb_red!10} 157 & \cellcolor{sb_red!10}~\textcolor{gray}{[140, 177]} \\
Assault & \cellcolor{white} 1125 & \cellcolor{white}~\textcolor{gray}{[1094, 1160]} & \cellcolor{sb_red!40} 366 & \cellcolor{sb_red!40}~\textcolor{gray}{[359, 374]} & \cellcolor{sb_red!10} 1033 & \cellcolor{sb_red!10}~\textcolor{gray}{[998, 1068]} & \cellcolor{sb_red!25} 923 & \cellcolor{sb_red!25}~\textcolor{gray}{[873, 987]} \\
Asterix & \cellcolor{white} 2216 & \cellcolor{white}~\textcolor{gray}{[2081, 2346]} & \cellcolor{sb_red!25} 492 & \cellcolor{sb_red!25}~\textcolor{gray}{[425, 560]} & \cellcolor{sb_red!10} 1987 & \cellcolor{sb_red!10}~\textcolor{gray}{[1560, 2414]} & \cellcolor{sb_green!10} 2503 & \cellcolor{sb_green!10}~\textcolor{gray}{[2050, 3040]} \\
Asteroids & \cellcolor{white} 602 & \cellcolor{white}~\textcolor{gray}{[493, 689]} & \cellcolor{white} 249 & \cellcolor{white}~\textcolor{gray}{[239, 259]} & \cellcolor{white} 563 & \cellcolor{white}~\textcolor{gray}{[424, 740]} & \cellcolor{white} 765 & \cellcolor{white}~\textcolor{gray}{[624, 952]} \\
\bottomrule
\end{tabularx}
\end{figure}

\clearpage

\begin{figure}
\centering
\scriptsize
\setlength\tabcolsep{0pt}
\begin{tabularx}{\textwidth}{l@{\hspace{6pt}} r@{}X@{\hspace{2pt}} r@{}X@{\hspace{2pt}} r@{}X@{\hspace{2pt}} r@{}X} %
\toprule
Task & \multicolumn{2}{c}{MR.Q} & \multicolumn{2}{c}{No reward scaling} & \multicolumn{2}{c}{No min} & \multicolumn{2}{c}{No LAP} \\
\midrule
Alien & \cellcolor{white} 2471 & \cellcolor{white}~\textcolor{gray}{[1848, 3155]} & \cellcolor{sb_red!10} 2074 & \cellcolor{sb_red!10}~\textcolor{gray}{[1402, 2821]} & \cellcolor{sb_red!10} 2375 & \cellcolor{sb_red!10}~\textcolor{gray}{[1921, 3052]} & \cellcolor{sb_red!10} 2265 & \cellcolor{sb_red!10}~\textcolor{gray}{[1784, 2768]} \\
Amidar & \cellcolor{white} 443 & \cellcolor{white}~\textcolor{gray}{[376, 499]} & \cellcolor{sb_red!10} 402 & \cellcolor{sb_red!10}~\textcolor{gray}{[345, 454]} & \cellcolor{sb_green!10} 567 & \cellcolor{sb_green!10}~\textcolor{gray}{[465, 686]} & \cellcolor{sb_red!10} 361 & \cellcolor{sb_red!10}~\textcolor{gray}{[291, 428]} \\
Assault & \cellcolor{white} 1125 & \cellcolor{white}~\textcolor{gray}{[1094, 1160]} & \cellcolor{sb_green!25} 1235 & \cellcolor{sb_green!25}~\textcolor{gray}{[1140, 1330]} & \cellcolor{sb_green!10} 1154 & \cellcolor{sb_green!10}~\textcolor{gray}{[967, 1254]} & \cellcolor{sb_red!10} 1095 & \cellcolor{sb_red!10}~\textcolor{gray}{[1027, 1170]} \\
Asterix & \cellcolor{white} 2216 & \cellcolor{white}~\textcolor{gray}{[2081, 2346]} & \cellcolor{sb_green!10} 2476 & \cellcolor{sb_green!10}~\textcolor{gray}{[2051, 2901]} & \cellcolor{sb_green!10} 2881 & \cellcolor{sb_green!10}~\textcolor{gray}{[2600, 3045]} & \cellcolor{sb_green!10} 2503 & \cellcolor{sb_green!10}~\textcolor{gray}{[2250, 2800]} \\
Asteroids & \cellcolor{white} 602 & \cellcolor{white}~\textcolor{gray}{[493, 689]} & \cellcolor{white} 614 & \cellcolor{white}~\textcolor{gray}{[520, 711]} & \cellcolor{white} 769 & \cellcolor{white}~\textcolor{gray}{[713, 878]} & \cellcolor{white} 509 & \cellcolor{white}~\textcolor{gray}{[447, 568]} \\
\bottomrule
\end{tabularx}
\end{figure}

\clearpage

\begin{figure}
\centering
\scriptsize
\setlength\tabcolsep{0pt}
\begin{tabularx}{\textwidth}{l@{\hspace{6pt}} r@{}X@{\hspace{2pt}} r@{}X@{\hspace{2pt}} r@{}X@{\hspace{2pt}} r@{}X} %
\toprule
Task & \multicolumn{2}{c}{MR.Q} & \multicolumn{2}{c}{No MR} & \multicolumn{2}{c}{1-step return} & \multicolumn{2}{c}{No unroll} \\
\midrule
Alien & \cellcolor{white} 2471 & \cellcolor{white}~\textcolor{gray}{[1848, 3155]} & \cellcolor{sb_green!10} 2886 & \cellcolor{sb_green!10}~\textcolor{gray}{[2458, 3315]} & \cellcolor{sb_red!10} 1342 & \cellcolor{sb_red!10}~\textcolor{gray}{[1226, 1454]} & \cellcolor{sb_green!10} 3056 & \cellcolor{sb_green!10}~\textcolor{gray}{[2300, 3614]} \\
Amidar & \cellcolor{white} 443 & \cellcolor{white}~\textcolor{gray}{[376, 499]} & \cellcolor{sb_red!10} 312 & \cellcolor{sb_red!10}~\textcolor{gray}{[204, 421]} & \cellcolor{sb_red!10} 240 & \cellcolor{sb_red!10}~\textcolor{gray}{[213, 287]} & \cellcolor{sb_red!10} 385 & \cellcolor{sb_red!10}~\textcolor{gray}{[335, 471]} \\
Assault & \cellcolor{white} 1125 & \cellcolor{white}~\textcolor{gray}{[1094, 1160]} & \cellcolor{sb_red!10} 1105 & \cellcolor{sb_red!10}~\textcolor{gray}{[1075, 1142]} & \cellcolor{sb_red!10} 1045 & \cellcolor{sb_red!10}~\textcolor{gray}{[904, 1134]} & \cellcolor{sb_red!10} 1079 & \cellcolor{sb_red!10}~\textcolor{gray}{[1003, 1140]} \\
Asterix & \cellcolor{white} 2216 & \cellcolor{white}~\textcolor{gray}{[2081, 2346]} & \cellcolor{white} 2298 & \cellcolor{white}~\textcolor{gray}{[1920, 2731]} & \cellcolor{sb_green!10} 2708 & \cellcolor{sb_green!10}~\textcolor{gray}{[2110, 3710]} & \cellcolor{sb_green!10} 2388 & \cellcolor{sb_green!10}~\textcolor{gray}{[2190, 2740]} \\
Asteroids & \cellcolor{white} 602 & \cellcolor{white}~\textcolor{gray}{[493, 689]} & \cellcolor{white} 485 & \cellcolor{white}~\textcolor{gray}{[387, 584]} & \cellcolor{white} 723 & \cellcolor{white}~\textcolor{gray}{[689, 783]} & \cellcolor{white} 581 & \cellcolor{white}~\textcolor{gray}{[406, 766]} \\
\bottomrule
\end{tabularx}
\end{figure}

\end{document}